\def\isarxiv{1} 

\ifdefined\isarxiv
\documentclass[11pt]{article}

\else
\documentclass[10pt,twocolumn,letterpaper]{article}

\usepackage[review,algorithms]{wacv}      
\fi

\ifdefined\isarxiv
\usepackage{amsmath}
\usepackage{amsthm}
\usepackage{amssymb}
\usepackage{algorithm}
\usepackage{subfig}
\usepackage{algpseudocode}
\usepackage{graphicx}
\usepackage{grffile}
\usepackage{wrapfig,epsfig}
\usepackage{url}
\usepackage{xcolor}
\usepackage{epstopdf}
\usepackage{bbm}
\usepackage{dsfont}
\else
\usepackage{graphicx}
\usepackage{amsmath}
\usepackage{amssymb}
\usepackage{booktabs}
\usepackage{amsthm}
\usepackage{xcolor}
\usepackage{algorithm}
\usepackage{algpseudocode}
\usepackage[pagebackref,breaklinks,colorlinks]{hyperref}
\usepackage[capitalize]{cleveref}
\fi

\allowdisplaybreaks

\ifdefined\isarxiv

\usepackage{tikz}
\usepackage{hyperref}  
\hypersetup{colorlinks=true,citecolor=blue,linkcolor=blue} 
\usetikzlibrary{arrows}
\usepackage[margin=1in]{geometry}

\else

\fi
\graphicspath{{./figs/}}

\theoremstyle{plain}
\newtheorem{theorem}{Theorem}[section]
\newtheorem{lemma}[theorem]{Lemma}
\newtheorem{definition}[theorem]{Definition}

\newtheorem{assumption}[theorem]{Assumption}

\newtheorem{fact}[theorem]{Fact}
\newtheorem{remark}[theorem]{Remark}

\newtheorem{condition}[theorem]{Condition}

\newcommand{\wh}{\widehat}
\newcommand{\wt}{\widetilde}

\newcommand{\R}{\mathbb{R}}

\renewcommand{\d}{\mathrm{d}}

\renewcommand{\tilde}{\wt}
\renewcommand{\hat}{\wh}

\DeclareMathOperator*{\E}{{\mathbb{E}}}

\DeclareMathOperator{\poly}{poly}

\DeclareMathOperator{\tr}{tr}

\makeatletter
\newcommand*{\RN}[1]{\expandafter\@slowromancap\romannumeral #1@}
\makeatother

\usepackage{lineno}

\ifdefined\isarxiv

\else

\crefname{section}{Sec.}{Secs.}
\Crefname{section}{Section}{Sections}
\Crefname{table}{Table}{Tables}
\crefname{table}{Tab.}{Tabs.}

\fi

\begin{document}

\ifdefined\isarxiv

\date{}

\title{Unraveling the Smoothness Properties of Diffusion Models: A Gaussian Mixture Perspective}

\author{
Yingyu Liang\thanks{\texttt{
yingyul@hku.hk}. The University of Hong Kong. \texttt{
yliang@cs.wisc.edu}. University of Wisconsin-Madison.} 
\and
Zhenmei Shi\thanks{\texttt{
zhmeishi@cs.wisc.edu}. University of Wisconsin-Madison.}
\and 
Zhao Song\thanks{\texttt{ magic.linuxkde@gmail.com}. The Simons Institute for the Theory of Computing at the University of California, Berkeley.}
\and 
Yufa Zhou\thanks{\texttt{ yufazhou@seas.upenn.edu}. University of Pennsylvania.}
}

\else

\title{Unraveling the Smoothness Properties of Diffusion Models: A Gaussian Mixture Perspective} 
\maketitle 

\fi

\ifdefined\isarxiv
\begin{titlepage}
  \maketitle
  \begin{abstract}
Diffusion models have made rapid progress in generating high-quality samples across various domains. However, a theoretical understanding of the Lipschitz continuity and second momentum properties of the diffusion process is still lacking. In this paper, we bridge this gap by providing a detailed examination of these smoothness properties for the case where the target data distribution is a mixture of Gaussians, which serves as a universal approximator for smooth densities such as image data. We prove that if the target distribution is a $k$-mixture of Gaussians, the density of the entire diffusion process will also be a $k$-mixture of Gaussians. We then derive tight upper bounds on the Lipschitz constant and second momentum that are independent of the number of mixture components $k$. Finally, we apply our analysis to various diffusion solvers, both SDE and ODE based, to establish concrete error guarantees in terms of the total variation distance and KL divergence between the target and learned distributions. Our results provide deeper theoretical insights into the dynamics of the diffusion process under common data distributions.

  \end{abstract}
  \thispagestyle{empty}
\end{titlepage}

{\hypersetup{linkcolor=black}
\tableofcontents
}
\newpage

\else

\begin{abstract}
Diffusion models have made rapid progress in generating high-quality samples across various domains. However, a theoretical understanding of the Lipschitz continuity and second momentum properties of the diffusion process is still lacking. In this paper, we bridge this gap by providing a detailed examination of these smoothness properties for the case where the target data distribution is a mixture of Gaussians, which serves as a universal approximator for smooth densities such as image data. We prove that if the target distribution is a $k$-mixture of Gaussians, the density of the entire diffusion process will also be a $k$-mixture of Gaussians. We then derive tight upper bounds on the Lipschitz constant and second momentum that are independent of the number of mixture components $k$. Finally, we apply our analysis to various diffusion solvers, both SDE and ODE based, to establish concrete error guarantees in terms of the total variation distance and KL divergence between the target and learned distributions. Our results provide deeper theoretical insights into the dynamics of the diffusion process under common data distributions.

\end{abstract}

\fi

\section{Introduction}\label{sec:intro}

Diffusion models, a prominent generative modeling framework, have rapid progress and garnered significant attention in recent years, due to their potential powerful ability to generate high-quality samples across various domains and diverse applications. 
Score-based generative diffusion models~\cite{hja20,ssk+20} can generate high-quality image samples comparable to GANs, which require adversarial optimization. 
Based on the U-Net~\cite{rfb15}, stable diffusion~\cite{rbl+22}, a conditional multi-modality generation model, can successfully generate business-used images. Based on the Transformer (DiT)~\cite{px23}, OpenAI released a video diffusion model, SORA~\cite{sora}, with a surprising performance.

However, despite these technological strides, a critical gap persists in the theoretical understanding of diffusion processes, especially concerning the Lipschitz continuity and second momentum properties of these models. 
Many existing studies (\cite{d22,ccl+22,llt22,cll23,ccl+24} and many more) make simplifying assumptions about these key smoothness properties but lack rigorous formulation or comprehensive analysis. 
This paper aims to bridge this theoretical gap by providing a detailed examination of the Lipschitz and second momentum characteristics. 
To make the data distribution analyzable, we consider the mixture of Gaussian data distribution, which is a universal approximation (Fact~\ref{fac:approximator}) for any smooth density, such as complex image and video data distributions.

\begin{figure*}[!ht]
    \centering
    \includegraphics[width=0.24\linewidth]{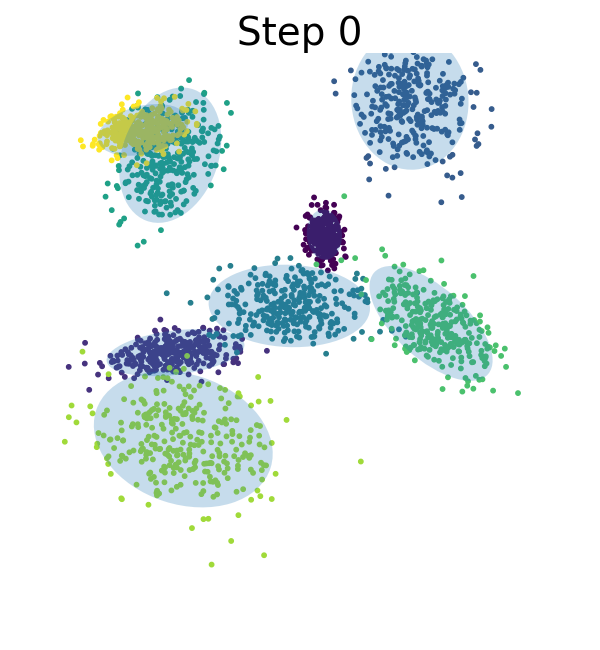}
    \includegraphics[width=0.24\linewidth]{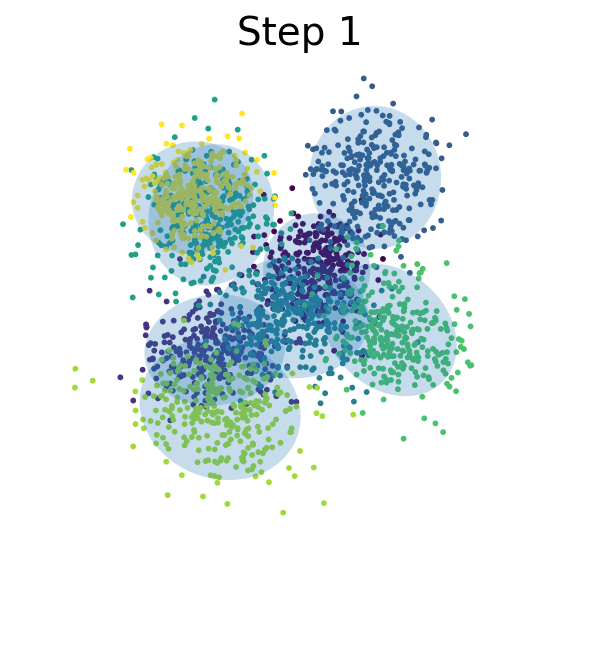}
    \includegraphics[width=0.24\linewidth]{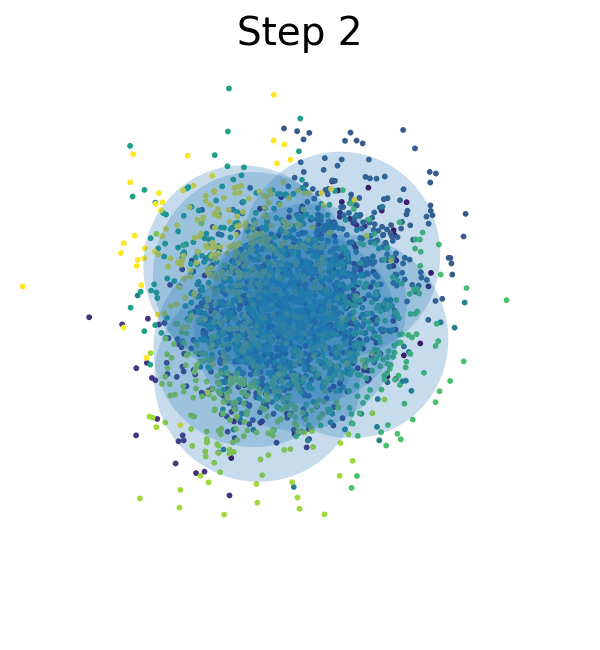}
    \includegraphics[width=0.24\linewidth]{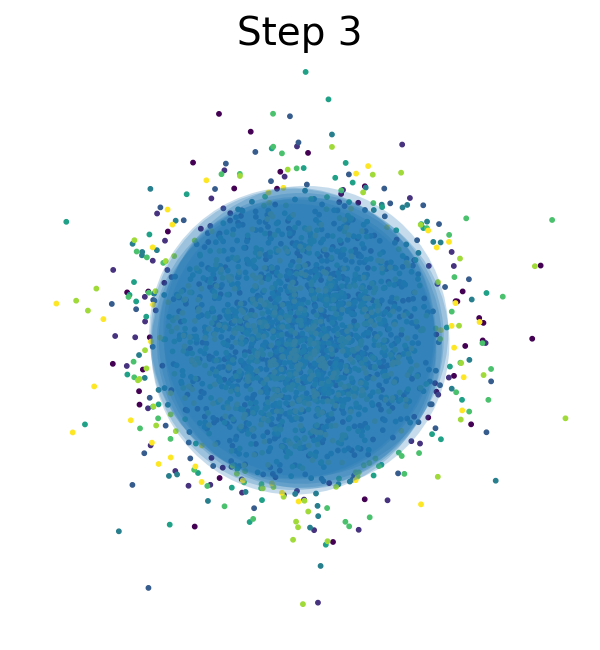}
    \caption{An illustration of discrete diffusion process for $8$ mixture of Gaussian as shown in Eq.~\eqref{eq:foward_SDE_dicretized}. 
    The left figure represents the target 2-dimensional data distribution $p_0$. 
    The right figure represents the standard normal distribution $p_T = \mathcal{N}(0, I_{2\times 2})$, where $T=3$.} 
    \label{fig:intro}
\end{figure*}

\begin{fact}\label{fac:approximator}
A Gaussian mixture model is a universal approximator of densities, in the sense that any smooth density can be approximated with any speciﬁc nonzero amount of error by a Gaussian mixture model with enough components~\cite{s15}. 
\end{fact}
Furthermore, we prove that if the target data distribution is a $k$-mixture of Gaussian, the density of the whole diffusion process will be a $k$-mixture of Gaussian (Lemma~\ref{lem:pdf_of_k_gaussian_plus_single_gaussian:informal}). 
Thus, our focus is studying a mixture of Gaussian target data distribution in the diffusion process, where we can gain a concrete Lipschitz constant and second momentum (Lemma~\ref{lem:lip_const_k_gaussian:informal} and Lemma~\ref{lem:second_moment:informal}). 
Moreover, we explore the implications of these properties through the lens of various solvers, both Stochastic Differential Equation (SDE) and Ordinary Differential Equation (ODE) based, providing a deeper insight into the dynamic behavior of diffusion processes and concrete guarantees in Table~\ref{tab:application}.

\paragraph{Our contribution:}
\begin{itemize}
    \item  As the Gaussian mixture model is a universal approximator of densities (Fact~\ref{fac:approximator}), we assume the target/image data distribution as a $k$-mixture of Gaussian. Then, we show that the density of the whole diffusion process is a $k$-mixture of Gaussian (Lemma~\ref{lem:pdf_of_k_gaussian_plus_single_gaussian:informal}). 
    \item We analyze the Lipschitz and second momentum of $k$-mixture of Gaussian data distribution and provide a tight upper bound, which is independent of $k$, (Lemma~\ref{lem:lip_const_k_gaussian:informal} and Lemma~\ref{lem:second_moment:informal}),
    even when $k$ goes to infinity (see more discussion in Remark~\ref{rem:lip}).
    \item  After applying our analysis to $\mathsf{DDPM}$, which is the SDE version of reverse process, we prove the dynamic of the diffusion process satisfies some concrete total variation bound (Theorem~\ref{thm:TV_ccl22:informal}) and concrete KL divergence bound (Theorem~\ref{thm:kl_cll23:informal} and Theorem~\ref{thm:kl_smooth_data_cll23:informal}) under choosing some total discretization steps. See a summary in Table~\ref{tab:application}.
    \item After applying our analysis to $\mathsf{DPOM}$ and $\mathsf{DPUM}$, which is the ODE version of the reverse process, we prove the dynamic of the diffusion process satisfies some concrete total variation bound (Theorem~\ref{thm:DPOM_ccl24:informal} and Theorem~\ref{thm:DPUM_ccl24:informal}) under choosing some total discretization steps. 
\end{itemize}

\paragraph{Other studies of the mixture of Gaussian under diffusion. }
Recently, there has been a rich line of work that studies a mixture of Gaussian data distribution under the diffusion process, which shares a similar popular setting. 
However, none focus on the Lipschitz smoothness and second momentum property as ours. 
We provide a short summary here for convenience. 
\cite{wcl+24} analyses the effect of diffusion guidance and provides theoretical insights on the instillation of task-specific information into the score function and conditional generation, e.g., text-to-image, by studying the mixture of Gaussian data distribution as a case study.  
\cite{glb+24}~propose a new SDE-solver for diffusion models under a mixture of Gaussian data. 
\cite{sck23,gkl24}~learn mixtures of Gaussian data using the $\mathsf{DDPM}$ objective and gives bound for sample complexity by assuming all covariance matrices are identity but does not use Lipschitz, while our work does not need identity assumptions.  
\cite{cks24}~mainly focuses on solving $k$-mixture of Gaussian in diffusion model by leveraging the property of the mixture of Gaussian and polynomial approximation methods and gives bound for sample complexity by assuming the covariance matrices have bounded condition number and that the mean vectors and covariance matrices lie in a bounded ball.

\begin{table*}[!ht]
    \centering
    \begin{tabular}{cccc}
       Type & Error guarantee & Steps for $\wt{O}(\epsilon_0^2)$ error & Reference \\ 
       \hline 
      $\mathsf{DDPM}$~\cite{ccl+22} & $\mathrm{TV}( p_0, \hat{q})^2$ & $\wt{\Theta}( d/\epsilon_0^2)$&  Theorem~\ref{thm:TV_ccl22:informal} \\ 
      $\mathsf{DDPM}$~\cite{cll23} & $\mathrm{KL}( p_0, \hat{q})$ & $\wt{\Theta} ( d /\epsilon_0^2 )$&  Theorem~\ref{thm:kl_cll23:informal} \\ 
      $\mathsf{DDPM}$~\cite{cll23} & $\mathrm{KL}( p_0, \hat{q})$ & $\wt{\Theta} ( d^2 /\epsilon_0^2 )$ & 
  Theorem~\ref{thm:kl_smooth_data_cll23:informal}\\ 
     $\mathsf{DPOM}$~\cite{ccl+24} & $\mathrm{TV}( p_0, \hat{q})^2$ & $\wt{\Theta}(d/\epsilon_0^2)$ &  Theorem~\ref{thm:DPOM_ccl24:informal} \\ 
     $\mathsf{DPUM}$~\cite{ccl+24} & $\mathrm{TV}( p_0, \hat{q})^2$ & $\wt{\Theta}(\sqrt{d}/\epsilon_0)$ &  Theorem~\ref{thm:DPUM_ccl24:informal} 
    \end{tabular}
    \caption{A summary of our applications using our Lipschitz and second momentum bound, when we assume $\sigma_{\min(p_t)}$ as a constant (defined in Condition~\ref{con:all}). The third column means the number of discretization points 
    required to guarantee a small total variation/KL divergence distance between the target data distribution $p_0$ and the learned distribution $\hat q$. 
    }
    \label{tab:application}
\end{table*} 
\section{Related Work}\label{sec:related}
\noindent{\bf Diffusion models and score-based generative models.}
\cite{hja20} introduced the concept of denoising diffusion probabilistic models ($\mathsf{DDPM}$), which learn to reverse a gradual noising process to generate samples and have been applied to many area \cite{wsd+23,wcz+23,wxz+24_omni}. 
Then, \cite{ssk+20} used Stochastic Differential Equations (SDE) to build the diffusion model and explored the equivalence between score-based generative models (SGMs) and diffusion models, generalizing the diffusion model to continuous time.
There is a line of works~\cite{swmg15} studying the connection between diffusion models and non-equilibrium thermodynamics, providing a theoretical foundation for these models. 
Another line of work has focused on improving the efficiency and quality of diffusion models~\cite{nrw21}. 
Particularly, \cite{se19} leverages score matching to train diffusion models more efficiently; 
\cite{nd21} improved architectures and training techniques for diffusion models.
Diffusion models have been applied to various domains beyond image generation, achieving impressive results, e.g., text-to-image synthesis~\cite{czz+20}, audio synthesis~\cite{kph+20}, image super-resolution~\cite{shc+22}, and so on. 

\noindent{\bf Mixture of Gaussian.}
Mixtures of Gaussian are among the most fundamental and widely used statistical models~\cite{d99, fos08,mv10,bs10,cdss13,soaj14,dk14,dkk+19,adls17,ls17,abhl+18,dk20,bk20,dhkk20,scl+22,bdj+22,xsw+23,bs23,smf+24} 
and studied in neural networks learning~\cite{rgkz21,swl21,lsss24,swl24}. Recently, they have also been widely studied in diffusion models~\cite{sck23,wcl+24,glb+24,gkl24,cks24} (see detailed discussion in Section~\ref{sec:intro}). 

\noindent{\bf Lipschitz and second momentum in score estimation.}
There is a line of work using bounded Lipschitz and second momentum as their key smoothness assumptions for the whole diffusion process without giving any concrete formulation~\cite{ccl+22,kfl22,llt22,lkb+23,cll23,chzw23,ccl+24,cdd23,bdd23,zhf+24,kll+24}, while our work gives a ``close-form'' formulation of Lipschitz and second momentum. 
There is another rich line of work studying how to train the diffusion models to have a better theoretical guarantee~\cite{se19,se20,sk21,sgse20,sdme21,sdcs23,lkb+23,cdd23,rhx+23,chzw23,sck23,yfz+23,llss24_softmax,gkl24,ccl+24,glb+24,cks24,hrx24,hwsl24} and many more.

\noindent{\bf Roadmap.}
In Section~\ref{sec:lip_second_moment}, we provide the notation we use, several definitions, and lemmas related to $k$ mixtures of Gaussian.
In Section~\ref{sec:preli}, we provide preliminary knowledge on score-based generative models (SGMs) and diffusion models. 
Section~\ref{sec:tools_previous} provides the tools that we use from previous papers.
In Section~\ref{sec:main_application}, we present our main results.
In Section~\ref{sec:conclusion}, we conclude the paper.
\section{Lipschitz and Second Momentum of Mixture of Gaussian}\label{sec:lip_second_moment}

In this section, we discuss the notations used, several key definitions for $k$ mixtures of Gaussian distributions, and lemmas concerning the Lipschitz continuity and second momentum of these mixtures. We begin by presenting the notations that are used throughout the paper.
{\bf Notations.}
For two vectors $x \in \R^n$ and $y \in \R^n$, we use $\langle x, y \rangle$ to denote the inner product between $x,y$, i.e., $\langle x, y \rangle = \sum_{i=1}^n x_i y_i$.
We use $\Pr[]$ to denote the probability, we use $\E[]$ to denote the expectation.
We use $e_i$ to denote a vector where only $i$-th coordinate is $1$, and other entries are $0$.
For each $a, b \in \R^n$, we use $a \circ b \in \R^n$ to denote the vector where $i$-th entry is $(a\circ b)_i = a_i b_i$ for all $i \in [n]$, and this is the Hardamard product.
We use ${\bf 1}_n$ to denote a length-$n$ vector where all the entries are ones.
We use $x_{i,j}$ to denote the $j$-th coordinate of $x_i \in \R^n$.
We use $\|x\|_p$ to denote the $\ell_p$ norm of a vector $x \in \R^n$, i.e. $\|x\|_1 := \sum_{i=1}^n |x_i|$, $\|x\|_2 := (\sum_{i=1}^n x_i^2)^{1/2}$, and $\|x\|_{\infty} := \max_{i \in [n]} |x_i|$.
For $k > n$, for any matrix $A \in \R^{k \times n}$, we use $\|A\|$ to denote the spectral norm of $A$, i.e. $\|A\|:=\sup_{x\in \R^n} \|Ax\|_2 / \|x\|_2$.
We use $\sigma_{\min}(A), \sigma_{\max}(A)$ to denote the minimum/maximum singular value of matrix $A$.
For a square matrix $A$, we use $\tr[A]$ to denote the trace of $A$, i.e., $\tr[A] = \sum_{i=1}^n A_{i,i}$.
We use $\det(A)$ to denote the determinant of matrix $A$.
We use $f * g$ to denote the convolution of 2 functions $f,g$.
In addition to $O(\cdot)$ notation, for two functions $f, g$, we use the shorthand $f \lesssim g$ (resp. $\gtrsim$) to indicate that $f \leq C g$ (resp. $\geq$) for an absolute constant $C$.

\noindent{ \bf $k$ mixtures of Gaussian.}
Now, we are ready to introduce $k$ mixtures of Gaussian. Formally, we can have the following definition of the $\mathsf{pdf}$ for $k$ mixtures of Gaussian.
\begin{definition}[$k$ mixtures of Gaussian $\mathsf{pdf}$ ]\label{def:p_t_k_gaussian_informal}
Let $x \in \R^d$, $i \in [k]$, $t \geq 0 \in \R$. For a fixed timestamp $t$, (1) let $\alpha_i(t) \in (0,1)$ be the weight for the $i$-th Gaussian component at time $t$ and $\sum_{i=1}^k \alpha_i(t) = 1$; (2) let $\mu_i(t) \in \R^{d}$ and $\Sigma_i(t) \in \R^{d\times d}$ be the mean vector and covariance matrix for the $i$-th Gaussian component at time $t$.
Then, we define
\begin{align*}
    p_t(x) := & ~ \sum_{i=1}^k \frac{\alpha_i(t)}{(2\pi)^{d/2}\det(\Sigma_i(t))^{1/2}}  \cdot \exp(-\frac{1}{2} (x - \mu_i(t))^\top \Sigma_i(t)^{-1} (x - \mu_i(t))).
\end{align*}
\end{definition}

Then, the following lemma shows that the linear combination between $k$ mixtures of Gaussian and a single standard Gaussian is still a $k$ mixtures of Gaussian. The proof is in Appendix~\ref{sec:all_together}. 
\begin{lemma}[Informal version of Lemma~\ref{lem:pdf_of_k_gaussian_plus_single_gaussian:formal}]\label{lem:pdf_of_k_gaussian_plus_single_gaussian:informal}
Let $a, b \in \R$. Let ${\cal D}$ be a $k$-mixture of Gaussian distribution, and 
let $p$ be its $\mathsf{pdf}$ defined in Definition~\ref{def:p_t_k_gaussian_informal}, i.e.,
\begin{align*}
    p(x) := & ~ \sum_{i=1}^k \frac{\alpha_i}{(2\pi)^{d/2}\det(\Sigma_i)^{1/2}}  \cdot   \exp(-\frac{1}{2} (x - \mu_i)^\top \Sigma_i^{-1} (x - \mu_i))
\end{align*}
Let $x \in \R^d$ sample from ${\cal D}$. Let $z\in \R^d$ and $z \sim {\cal N}(0,I)$, which is independent from $x$. Then, we have a new random variable 
$y = a x + b z$ which is also a $k$-mixture of Gaussian distribution $\wt {\cal D}$, whose $\mathsf{pdf}$ is 
\begin{align*}
    \wt p(x) := & ~ \sum_{i=1}^k \frac{\alpha_i}{(2\pi)^{d/2}\det(\wt \Sigma_i)^{1/2}}   \cdot \exp(-\frac{1}{2} (x - \wt \mu_i)^\top \wt \Sigma_i^{-1} (x - \wt \mu_i)),
\end{align*}
where $ \wt \mu_i = a\mu_i, \wt \Sigma_i = a^2 \Sigma_i + b^2 I$.  
\end{lemma}

Note that the Gaussian mixture model is a universal approximator of densities (Fact~\ref{fac:approximator}). Then, it is natural to assume that the target/image ($p_0$) data distribution as the $k$ mixtures of Gaussian. 
Lemma~\ref{lem:pdf_of_k_gaussian_plus_single_gaussian:informal} tell us that if the target data distribution is $k$ mixtures of Gaussian, then by Eq.~\eqref{eq:foward_SDE_dicretized}, the $\mathsf{pdf}$ of the whole diffusion process is $k$ mixtures of Gaussian, i.e., the $p_t$ for any $t \in [0,T]$. See details in Section~\ref{sec:preli}. 

\noindent {\bf Main properties.}
Thus, we only need to analyze the property, i.e., Lipschitz constant and second momentum, of $k$ mixtures of Gaussian, to have a clear guarantee for the whole diffusion process. In the following two lemmas, we will give concrete bounds of the Lipschitz constant and second momentum of $k$ mixtures of Gaussian. Both proofs can be found in Appendix~\ref{sec:all_together}.

\begin{lemma}[Lipschitz, informal version of Lemma~\ref{lem:lip_const_k_gaussian:formal}]\label{lem:lip_const_k_gaussian:informal}
If the following conditions hold:
    Let $\beta \le \|x-a_t\mu_i\|_2 \leq R$, where $R \geq 1$ and $\beta \in (0, 0.1)$ for each $i \in [k]$.
    Let $p_t(x)$ be defined as Eq.~\eqref{eq:p_t} and $p_t(x) \geq \gamma$, where $\gamma \in (0, 0.1)$.
    Let $\sigma_{\min(p_t)} := \min_{i \in [k]} \{ \sigma_{\min}( a_t^2 \Sigma_i + b_t^2 I ) \}$, $\sigma_{\max(p_t)} := \max_{i \in [k]} \{ \sigma_{\max}( a_t^2 \Sigma_i + b_t^2 I)\}$.
    Let $\det_{\min(p_t)} := \min_{i \in [k]} \{ \det( a_t^2 \Sigma_i + b_t^2 I)\}$.

The Lipschitz constant for the score function $\frac{\d \log(p_t(x))}{\d x}$ is given by:
\begin{align*}
    L = & ~ \frac{1}{\sigma_{\min(p_t)}}+  \frac{2R^2}{\gamma^2\sigma_{\min(p_t)}^2} \cdot \exp(-\frac{\beta^2}{2\sigma_{\max(p_t)}})  \cdot (\frac{1}{(2\pi)^{d}\det_{\min(p_t)}} + \frac{1}{(2\pi)^{d/2}\det_{\min(p_t)}^{1/2}})  .
\end{align*}
\end{lemma}

In Lemma~\ref{lem:lip_const_k_gaussian:informal}, we can clearly see that roughly $L = O(1/\poly(\sigma_{\min(p_t)}))$,  which means the Lipschitz is only conditioned on the smallest singular value of all Gaussian component but independent with $k$.
\begin{remark}\label{rem:lip}
We believe our results in Lemma~\ref{lem:lip_const_k_gaussian:informal} is non-trivial.
First, the Lipschitz bound depends inversely exponential on the dimension $d$.
If $d$ grows lager, the Lipschitz constant $L$ becomes much smaller.
Second, in practice, the $k$ will be super large for complicated data distribution, e.g., millions of Gaussian components for an image distribution. 
Many studies need to overcome the large $k$ hardness in learning Gaussian mixture models~\cite{bdj+22,bs23,rgkz21,swl24}, while our Lipschitz upper-bound is independent with $k$. 
\end{remark}

\begin{lemma}[Second momentum, informal version of Lemma~\ref{lem:second_moment:formal}]\label{lem:second_moment:informal}
Let $x_0 \sim p_0$, where $p_0$ is defined by Eq.~\eqref{eq:p_0}. 
Then, we have
\begin{align*}
    m_2^2 :=  \E_{x_0 \sim p_0}[\| x_0 \|_2^2] =  \sum_{i=1}^k \alpha_i (\|\mu_i\|_2^2  + \tr[\Sigma_i]) \le \max_{i \in [k]} \{ \|\mu_i\|_2^2  + \tr[\Sigma_i] \}.
\end{align*}
\end{lemma}
The proof idea of Lemma~\ref{lem:lip_const_k_gaussian:informal} and Lemma~\ref{lem:second_moment:informal} is that we first consider the case when $k=1$ in Appendix~\ref{sec:continuous} and then we extend to $2$ mixtures of Gaussian setting in Appendix~\ref{sec:general_two_gaussian}. Finally, we can generalize to $k$ setting in Appendix~\ref{sec:k_gaussain} and summarize it in Appendix~\ref{sec:all_together}. 

In Lemma~\ref{lem:second_moment:informal}, we can see that $m_2^2 = O(1)$ roughly, which is independent of $k$ as well. Later, we will apply our Lemma~\ref{lem:lip_const_k_gaussian:informal} and Lemma~\ref{lem:second_moment:informal} in Section~\ref{sec:main_application} to get concrete bound for each diffusion process, e.g., $\mathsf{DDPM}$,  $\mathsf{DPOM}$ and $\mathsf{DPUM}$.

\section{Score Based Model and Diffusion Model}\label{sec:preli}
In Section~\ref{sec:background}, we first briefly introduce $\mathsf{DDPM}$~\cite{hja20}, a stochastic differential equations (SDE) version of the reverse diffusion process, and
score-based generative models (SGMs)~\cite{ssk+20}, which is a generalization of $\mathsf{DDPM}$. 
Then, in Section~\ref{sec:solvers}, we introduce multiple solvers for the reverse process.

\subsection{Background on score based model and diffusion model}\label{sec:background}

First, we denote the input data as $x \in \R^d$ and the target original data distribution as $p_0(x)$.
Assuming that the noisy latent state of score-based generative models (SGMs) at time $t$ is a stochastic process $x_t$ for $0\leq t \leq T$, we have the forward SDE is defined by:
\begin{align}\label{eq:foward_SDE_continuous}
    \d x_t = f(x_t,t) \d t + g(t) \d w_t, ~~ x_0 \sim p_0,  ~~0\leq t \leq T
\end{align}
where $w_t \in \R^{d}$ is the standard Brownian motion, $f(\cdot,t): \R^d \rightarrow \R^d$ which is called drift coefficient, and $g(t): \R \rightarrow \R$ which is called diffusion coefficient. We use $p_t(x)$ to denote the marginal probability density function of $x_t$.
The $\mathsf{pdf}$ at last time step $p_T$ is the prior distribution which often defined as standard Gaussian $p_T(x) = \mathcal{N}(0,I)$.

The continuous forward SDE Eq.~\eqref{eq:foward_SDE_continuous} also has the discrete Markov chain form as
\begin{align}\label{eq:foward_SDE_dicretized}
x_t = a_t x_0 + b_t z,  ~~ x_0 \sim p_0
\end{align}
where $x_0 \sim p_0(x)$ and $z \sim \mathcal{N}(0,I)$, and $a_t,b_t \in \R$ are functions of time. Additionally, we assume as $t \rightarrow T$, $a_t \rightarrow 0$ and $b_t \rightarrow 1$. Thus, $x_T \sim \mathcal{N}(0,I)$. Also, clearly when $t \to 0$, $a_t \to 1$ and $b_t \to 0$ for this is the boundary condition. More specifically, the Eq.~\eqref{eq:foward_SDE_dicretized} can be viewed as the iterative equation in $\mathsf{DDPM}$~\cite{hja20}.

From \cite{a82}, we know $x_t$ also satisfy the reverse SDE:
\begin{align}\label{eq:reverse_SDE}
    \d x_t = \left( f(x_t, t) - g(t)^2 \nabla \log p_t(x_t) \right) \d t + g(t) \, \d \wt{w}_t,
\end{align}
where $\wt{w}_t \in \R^{d}$ is backward Brownian motion \cite{a82}, and $\nabla \log p_t(x_t)$ is the score function. For convenience, we rewrite the reverse SDE Eq.~\eqref{eq:reverse_SDE} in a forward version by switching the time direction, replacing $t$ with $T - t$. Let $\wt{x}_t := x_{T-t}$. The law of $(\wt{x}_t)_{0 \leq t \leq T}$ is identical to the law of $(x_{T-t})_{0 \leq t \leq T}$. We use $q_t$ to denote the density of $\wt{x}_t$:
\begin{align}
    \d \wt{x}_t = & ~ ( -f(\wt{x}_t, T - t) + g(T - t)^2 \nabla \log p_{T-t}(\wt{x}_t) ) \d t + g(T - t) \d w_t. \label{eq:reverse_SDE_T_minus_t}
\end{align}
The process $(\wt{x}_t)_{0 \leq t \leq T}$ converts noise into samples from $p_0$, thereby achieving the objective of generative modeling.

Since we can not obtain the score function $\nabla \log p_t$ directly, we can use a neural network to approximate it, and we denote the estimated score function by $s_t(x)$.
By replacing the score function $\nabla \log p_t$ with our approximated score function $s_t(x)$, we can rewrite Eq.~\eqref{eq:reverse_SDE_T_minus_t} as:
\begin{align}
    \d y_t = & ~ ( -f(y_t, T - t) + g(T - t)^2 s_{T - t}(y_t) ) \d t + g(T - t) \, \d w_t, ~~ y_0 \sim p_T, ~~ 0 \le t \le T. \label{eq:reverse_SDE_approximate}
\end{align}
where $y_t$ is the process we approximate by our SGM $s_t$.

For clarity, we mainly focus on the Ornstein-Uhlenbeck (OU) process, which is a diffusion process with $\exp$ coefficient: 
\begin{definition}\label{def:ou_process}
The forward SDE of OU process ($f(x_t,t)=-x_t$, and $g(t)\equiv \sqrt{2}$) is:
$
    \d x_t = -x_t \d t + \sqrt{2} \d w_t.
$
The corresponding discrete Markov chain form of OU process is given by:
$
    x_t = e^{-t} x_0 + \sqrt{1-e^{-2t}} z,
$
where we can see that $a_t = e^{-t}$, $b_t=\sqrt{1-e^{-2t}}$ in Eq.~\eqref{eq:foward_SDE_dicretized}.
\end{definition}

From Eq.~\eqref{eq:reverse_SDE_approximate} under the OU process, we can have:
\begin{align}\label{eq:ou_y}
    \d y_t = ( x_t + 2 s_{T-t}(y_t) ) \d t + \sqrt{2} \d w_t .
\end{align}

\subsection{Definitions of different solvers}\label{sec:solvers}

In practical applications, it's necessary to adopt a discrete-time approximation for sampling dynamics. We have the following definition. 
\begin{definition}[Time discretization] \label{def:step_size}
Define the $N$ discretization points as $\delta = t_0 \le t_1 \le \cdots \le t_N = T$, where $\delta \geq 0$ is the early stopping parameter, with $\delta=0$ presenting the normal setting. For each discretization step $k$, where $1 \le k \le N$, the step size is denoted by $h_k := t_k - t_{k-1}$.
\end{definition}

Let $\hat{y}_t$ be the discrete approximation of $y_t$ in Eq.~\eqref{eq:ou_y}
starting from $\hat{y}_0 \sim \mathcal{N}(0, I)$. We use $\hat{q}_t$ to denote the density of $\hat{y}_t$. 
Let $N$ be defined in Definition~\ref{def:step_size} and $t'_k=T-t_{N-k}$. 
To solve Eq.~\eqref{eq:ou_y}, we have the following two numerical solvers: 

\begin{definition}\label{def:euler}
We define $\mathsf{EulerMaruyama}$ as the numerical solver satisfying 
\begin{align*}
    \hat{y}_{T-\delta}=\mathsf{EulerMaruyama}(T,s),
\end{align*}
where $\hat{y}_{T-\delta} \in \R^d$ is the output, $T \in \R^+$ is the total time, and $s$ is the score estimates. 
The Euler-Maruyama scheme is, (Equation 5 in \cite{cll23}), for $t \in [t_k', t_{k+1}']$
\begin{align}\label{eq:Euler_Maruyama_scheme}
    \d\hat{y}_t = (  \hat{y}_{t'_k} + 2 s_{T - t'_k}(\hat{y}_{t'_k}) ) \d t + \sqrt{2} \d w_t.
\end{align}
\end{definition}

\begin{definition}\label{def:exp_int}
We define $\mathsf{ExponentialIntegrator}$ as the numerical solver satisfying 
\begin{align*}
 \hat{y}_{T-\delta}=\mathsf{ExponentialIntegrator}(T,s),
\end{align*}
where $\hat{y}_{T-\delta} \in \R^d$ is the output, $T \in \R^+$ is the total time, and $s$ is the score estimates. 
The Exponential Integrator scheme is, (Equation 6 in \cite{cll23}), for $t \in [t_k', t_{k+1}']$
\begin{align}\label{eq:expontential_integrator_scheme}
    \d \hat{y}_t = ( \hat{y}_t + 2 s_{T - t'_k}(\hat{y}_{t'_k}) ) \d t + \sqrt{2} \d w_t.
\end{align}
\end{definition}

The two methods above are used for solving SDE. The difference is that in the first term of RHS of Euler-Maruyama uses $\hat{y}_{t'_k}$, while Exponential Integrator uses $\hat{y}_t$. The Exponential Integrator scheme has a closed-form solution (see detail in Section 1.1 of \cite{cll23}).

Now, we introduce two ordinary differential equations (ODE) solvers $\mathsf{DPOM}$ and $\mathsf{DPUM}$, which ignore the Brownian motion term in Eq.~\eqref{eq:ou_y}. We will provide their concrete steps bound in Section~\ref{sec:main_application}.

\begin{definition}\label{def:algo_DPOM}
We define $\mathsf{DPOM}$ (Diffusion Predictor with Overdamped Modeling) as Algorithm 1 in \cite{ccl+24} 
satisfying 
$
    \hat{y}_{T-\delta}=\mathsf{DPOM}(T,h_{\mathrm{pred}},h_{\mathrm{corr}},s),
$
where $\delta$ is the early stopping parameter, $\hat{y}_{T-\delta} \in \R^d$ is the output, $T \in \R^+$ is the total steps, $h_{\mathrm{pred}}$ is the predictor step size, $h_{\mathrm{corr}}$ is the corrector step size, (see detailed definition in Algorithm 1 of \cite{ccl+24}), and $s$ is the score estimates. 
\end{definition}

\begin{definition}\label{def:algo_DPUM}
We define $\mathsf{DPUM}$ (Diffusion Predictor with Underdamped Modeling) as Algorithm 2 in \cite{ccl+24},
satisfying 
$
    \hat{y}_{T-\delta}=\mathsf{DPUM}(T,h_{\mathrm{pred}},h_{\mathrm{corr}},s),
$
where $\delta$ is the early stopping parameter, $\hat{y}_{T-\delta} \in \R^d$ is the output, $T \in \R^+$ is the total steps, $h_{\mathrm{pred}}$ is the predictor step size, $h_{\mathrm{corr}}$ is the corrector step size, (see detailed definition in Algorithm 2 of \cite{ccl+24}), and $s$ is the score estimates. 
\end{definition}

\section{Main Result for Application}\label{sec:main_application}
In this section, we will provide the main results of applications. 
In Section~\ref{sec:application_assumption}, we provide our key definitions and assumptions used. 
In Section~\ref{sec:main_results:TV}, we provide our results for the total variation bound. 
In Section~\ref{sec:main_results:KL}, we provide our results for the KL divergence bound. 
In Section~\ref{sec:main_results:ODE}, we provide our results for the probability flow ODE method, including DPOM and DPUM.

\subsection{Key definitions and assumptions}\label{sec:application_assumption}

We first assume that the loss of the learned score estimator is upper bounded by $\epsilon_0^2$ in Assumption~\ref{ass:assumption_score_estimation}. 
Then, we can show that we can recover the target/image data distribution under a small total variation or KL divergence gap later.  

\begin{assumption}[Score estimation error, Assumption 1 in \cite{cll23}, page 6, and Assumption 3 in \cite{ccl+22}, page 6]\label{ass:assumption_score_estimation}
The learned score function $s_t(x)$ satisfies for any $1 \leq k \leq N$,
\begin{align*}
   \frac{1}{T} \sum_{k=1}^{N} h_k \E_{p_{t_k}} [\| \nabla \log p_{t_k}(x) - s_{t_k}(x) \|_2^2] \leq \epsilon_0^2. 
\end{align*}
where $h_k$ is the step size defined in Definition~\ref{def:step_size} for step $k$, $N$ is the total steps, and $\sum_{k=1}^N h_k = T$.
\end{assumption}
To avoid ill-distribution $q_T$, we have the below definition follows~\cite{cll23,ccl+24}.
\begin{definition}\label{def:epsilon}
We define $\epsilon$ as total variation distance between $q_{T-\delta}$ and $q_T$. 
\end{definition}

Let the data distribution $p_0(x)$ be a $k$-mixture of Gaussians:
\begin{align}\label{eq:p_0}
    p_0(x) := \sum_{i=1}^k \alpha_i \mathcal{N}(\mu_i,\Sigma_i).
\end{align}
Using the Lemma~\ref{lem:pdf_of_k_gaussian_plus_single_gaussian:informal}, we know the $\mathsf{pdf}$ of $x_t$ at the any time $t$ is given by:
\begin{align}\label{eq:p_t}
    p_t(x) = \sum_{i=1}^k \alpha_i \mathcal{N}(a_t \mu_i, a_t^2 \Sigma_i + b_t^2 I).
\end{align}

Then, we define the following conditions of the $k$ mixtures of Gaussian. 
\begin{condition}\label{con:init}
   All conditions here are related to the beginning time density $p_0$ in Eq.~\eqref{eq:p_0}. 
   \begin{itemize}
    \item Let $\sigma_{\min(p_0)} = \min_{i \in [k]} \{\sigma_{\min}(\Sigma_i)\}$  and $\sigma_{\max(p_0)} = \max_{i \in [k]} \{\sigma_{\max}(\Sigma_i)\}$.
    \item Let $\mu_{\max(p_0)} = \max_{i \in [k]} \{\|\mu_i\|_2^2\}$ and $\det_{\min(p_0)} = \min_{i \in [k]} \{\det(\Sigma_i)\}$.
   \end{itemize}
\end{condition}

In Condition~\ref{con:init}, we denote $\sigma_{\min(p_0)}, \sigma_{\max(p_0)}$ as the minimum/maximum singular value between the covariance matrices of $k$-components at the original data distribution $p_0$. 
$\mu_{\max(p_0)}$ is the maximum $\ell_2^2$-norm between mean vectors of the $k$-components at $p_0$. 
$\det_{\min(p_0)}$ is the minimum determinant between the covariance matrices of $k$-components at $p_0$.

\begin{condition}\label{con:all}
   All conditions here are related to the all time  density $p_t$ in Eq.~\eqref{eq:p_t}, where $t\in [0,T]$. Let $x \in \R^d$ and $a_t,b_t$ is defined by Definition~\ref{def:ou_process}. Assume Assumption~\ref{ass:assumption_score_estimation}.
   \begin{itemize}
    \item Let $\beta \leq \|x-a_t\mu_i\|_2 \leq R$, where $R \geq 1$ and $\beta \in (0, 0.1)$, for each $i \in [k]$.
    \item Let $p_t(x)$ be defined as Eq.~\eqref{eq:p_t} and $p_t(x) \geq \gamma$, where $\gamma \in (0, 0.1)$.
    \item Let $\sigma_{\min(p_t)} := \min_{i \in [k]} \{ \sigma_{\min}( a_t^2 \Sigma_i + b_t^2 I ) \}$, $\sigma_{\max(p_t)} := \max_{i \in [k]} \{ \sigma_{\max}( a_t^2 \Sigma_i + b_t^2 I)\}$.
    \item Let $\det_{\min(p_t)} := \min_{i \in [k]} \{ \det( a_t^2 \Sigma_i + b_t^2 I)\}$.
   \end{itemize}
\end{condition}

In Condition~\ref{con:all}, we denote $\sigma_{\min(p_t)}, \sigma_{\max(p_t)}$ as the minimum/maximum singular value between the covariance matrices of $k$-components at $p_t$. 
$\beta, R$ be the lower/upper bound of the $\ell_2$ distance between the input data $x$ and all $k$ scaled mean vectors $a_t \mu _i$ at timestep $t$.
Let $\gamma$ be the lower bound of $p_t(x)$, the $\mathsf{pdf}$ at time $t$ and $\det_{\min(p_t)}$ be the minimum determinant at $p_t$.

Clearly, we have $\sigma_{\max(p_t)} \ge \sigma_{\max(p_0)}$, so Condition~\ref{con:all} is a stronger condition.

\begin{fact}[Pinsker's inequality \cite{ck11}]\label{fac:pinsker} Let $p,q$ are two probability distributions, then we have 
$
    \mathrm{TV(p,q)}\leq \sqrt{\frac{1}{2}\mathrm{KL}(p\|q)}.
$
\end{fact}
From Fact~\ref{fac:pinsker}, we can see the total variation is a weaker bound than KL divergence.

\subsection{Total variation}\label{sec:main_results:TV}
Now, we are ready to present our result for total variation bound assuming data distribution is $k$-mixture of Gaussian ($p_0$ in Eq.~\eqref{eq:p_0}). Recall $\epsilon_0$ is defined in Assumption~\ref{ass:assumption_score_estimation}, $\epsilon$ is denied in Definition~\ref{def:epsilon} and $h_k$ is defined in Definition~\ref{def:step_size}. See the proof in Appendix~\ref{sec:app:application}.

\begin{theorem}[$\mathsf{DDPM}$, total variation, informal version of Theorem~\ref{thm:TV_ccl22:formal}]\label{thm:TV_ccl22:informal}
Assume Condition~\ref{con:init} and~\ref{con:all} hold. The step size $h_k := T/N$ satisfies $h_k = O(1/L)$ and $L \geq 1$ for $k \in [N]$. Let $\hat{q}$ denote the density of the output of the $\mathsf{EulerMaruyama}$ defined by Definition~\ref{def:euler}. 
Then, we have 
\begin{align*}
    \mathrm{TV}(\hat{q}, p_0) \lesssim  \underbrace{\sqrt{\mathrm{KL}(p_0 \| \mathcal{N}(0,I)) }\exp(-T)}_{ \mathrm{convergence~of~forward~process}} +  \underbrace{(L\sqrt{dh} + L m_2 h)\sqrt{T}}_{\mathrm{discretization~ error}} + \underbrace{\epsilon_0\sqrt{T}}_{\mathrm{score~estimation~error}
}.
\end{align*}
where $L = \frac{1}{\sigma_{\min(p_t)}}+  \frac{2R^2}{\gamma^2\sigma_{\min(p_t)}^2} \cdot (\frac{1}{(2\pi)^{d}\det_{\min(p_t)}} + \frac{1}{(2\pi)^{d/2}\det_{\min(p_t)}^{1/2}})  \cdot \exp(-\frac{\beta^2}{2\sigma_{\max(p_t)}})$, $m_2 = (\sum_{i=1}^k \alpha_i (\|\mu_i\|_2^2  + \tr[\Sigma_i]))^{1/2}$, and $\mathrm{KL}(p_0(x) \| \mathcal{N}(0,I)) \leq \frac{1}{2} (-\log( \mathrm{det}_{\min(p_0)} )  + d \sigma_{\max(p_0)} + \mu_{\max(p_0)} - d)$. 
\end{theorem}

In Theorem~\ref{thm:TV_ccl22:informal}, suppose that $m_2 \leq d$ and choose $T = \Theta( \log(\mathrm{KL}(p_0 \| \mathcal{N}(0,I))/\epsilon))$,  $h_k = \Theta( \frac{\epsilon^2}{L^2 d})$, then we have
$
    \mathrm{TV}(\hat{q}, p_0) \leq \wt{O}(\epsilon + \epsilon_0), ~ \text{for } N = \wt{O}(L^2 d/\epsilon^2).
$
In particular, in order to have $\mathrm{TV}(\hat{q}, p_0) \leq \epsilon$, it suffices to have score error $\epsilon_0 \leq \wt{O}(\epsilon)$, where $\wt{O}(\cdot)$ hides $T$ which is a $\log$ term. 
Thus, Theorem~\ref{thm:TV_ccl22:informal} provides a guarantee for total variation bound between target data distribution $p_0$ and learned output distribution $\hat q$ with concrete $L$ and $m_2$.

\subsection{KL divergence}\label{sec:main_results:KL}

Similarly, we can present our result for the KL divergence bound in the following two theorems, assuming data distribution is $k$-mixture of Gaussian. Recall $\epsilon_0$ is defined in Assumption~\ref{ass:assumption_score_estimation}, $\epsilon$ is denied in Definition~\ref{def:epsilon} and $h_k$ is defined in Definition~\ref{def:step_size}. See the proof in Appendix~\ref{sec:app:application}.

\begin{theorem}[$\mathsf{DDPM}$, KL divergence, informal version of Theorem~\ref{thm:kl_cll23:formal}]\label{thm:kl_cll23:informal}
Assume Condition~\ref{con:all}. We use uniform discretization points. 

(1) Let $\hat{q}$ denote the density of the output of the $\mathsf{ExponentialIntegrator}$ (Definition~\ref{def:exp_int}), we have
\begin{align*}
    \mathrm{KL}(p_0 \| \hat{q}) \lesssim (M_2 + d) \exp(-T) + T\epsilon_0^2 + \frac{d T^2 L^2}{N}.
\end{align*}
In particular, choosing $T = \Theta(\log ( M_2 d/\epsilon_0 ))$ and $N = \Theta ( d T^2 L^2/\epsilon_0^2 )$, then $\mathrm{KL}(p_0 \| \hat{q}) = \tilde{O} (\epsilon_0^2)$.

(2) Let $\hat{q}$ denote the density of the output of the $\mathsf{EulerMaruyama}$ (Definition~\ref{def:euler}), we have
\begin{align*}
    \mathrm{KL}(p_0 \| \hat{q}) \lesssim (M_2 + d) \exp(-T) + T\epsilon_0^2 + \frac{d T^2 L^2}{N} + \frac{T^3 M_2}{N^2}.
\end{align*}
where $M_2 = \sum_{i=1}^k \alpha_i (\|\mu_i\|_2^2  + \tr[\Sigma_i])$ and $L = \frac{1}{\sigma_{\min(p_t)}}+  \frac{2R^2}{\gamma^2\sigma_{\min(p_t)}^2} \cdot (\frac{1}{(2\pi)^{d}\det_{\min(p_t)}} + \frac{1}{(2\pi)^{d/2}\det_{\min(p_t)}^{1/2}})  \cdot \exp(-\frac{\beta^2}{2\sigma_{\max(p_t)}})$.
\end{theorem}

\begin{theorem}[$\mathsf{DDPM}$, KL divergence for smooth data distribution, informal version of Theorem~\ref{thm:kl_smooth_data_cll23:formal}]\label{thm:kl_smooth_data_cll23:informal}
Assume Condition~\ref{con:init} and~\ref{con:all}. We use the exponentially decreasing (then constant) step size $h_k = c \min \{\max \{t_k, 1/L \}, 1\},  c = \frac{T + \log L}{N} \leq \frac{1}{Kd}$. Let $\hat{q}$ denote the density of the output of the $\mathsf{ExponentialIntegrator}$ defined by Definition~\ref{def:exp_int}.
Then, we have 
\begin{align*}
    \mathrm{KL}(p_0 \| \hat{q}) \lesssim (M_2 + d) \exp(-T) + T \epsilon_0^2 + \frac{d^2 (T+\log L )^2}{N},
\end{align*}
where $M_2 = \sum_{i=1}^k \alpha_i (\|\mu_i\|_2^2  + \tr[\Sigma_i])$ and $L = \frac{1}{\sigma_{\min(p_t)}}+  \frac{2R^2}{\gamma^2\sigma_{\min(p_t)}^2} \cdot (\frac{1}{(2\pi)^{d}\det_{\min(p_t)}} + \frac{1}{(2\pi)^{d/2}\det_{\min(p_t)}^{1/2}})  \cdot \exp(-\frac{\beta^2}{2\sigma_{\max(p_t)}})$.
In particular, if we set $T = \Theta(\log ( M_2 d/\epsilon_0) )$ and $N = \Theta ( d^2 (T + \log L)^2/\epsilon_0^2 )$, then $ \mathrm{KL}(p_0 \| \hat{q}) \leq \wt{O}( \epsilon_0^2 )$.
In addition, for Euler-Maruyama scheme defined in Definition~\ref{def:euler}, the same bounds hold with an additional $M_2 \sum_{k=1}^N h_k^3$ term.
\end{theorem}

Theorem~\ref{thm:kl_smooth_data_cll23:informal} and Theorem~\ref{thm:kl_cll23:informal} provide a guarantee for KL divergence bound between target data distribution $p_0$ and learned output distribution $\hat q$ with concrete $L$ and $M_2$.  
Theorem~\ref{thm:kl_smooth_data_cll23:informal} has $\log^2 L$ instead of $L^2$ in the bound for total number of discretization points $N$, but includes an additional $d$ compared to Theorem~\ref{thm:kl_cll23:informal}. 
On the other hand, Theorem~\ref{thm:kl_smooth_data_cll23:informal} requires the data distribution $p_0$ is Lipschitz and second-order differentiable \cite{cll23}, allowing a better bound in terms of $L$, while all other theorems \cite{ccl+22,ccl+24} require conditions about Lipschitz on $p_t$ for any  $0\leq t \leq T$. However, our assumption $p_0$ is $k$-mixture of Gaussians satisfies all conditions they use.

From Fact~\ref{fac:pinsker}, we can compare KL divergence results with total variation results. Notice that $M_2 = m_2^2$ when comparing theorems for total variation and KL divergence. According to Fact~\ref{fac:pinsker}, the square of the TV distance is comparable to the KL divergence, which explains the squared relationship of the second momentum term.

\subsection{Probability flow ODE}\label{sec:main_results:ODE}

Notice that in the previous results we are considering SDE based models. However from \cite{ssk+20}, we know that we can also use ODE to run the reverse process, which has the same marginal distribution with SDE reverse process but is thereby deterministic. In this section, we provide results of $\mathsf{DPOM}$ and $\mathsf{DPUM}$ (Algorithm 1 and 2 in \cite{ccl+24}) algorithms which are based on ODE reverse process.

\begin{theorem}[$\mathsf{DPOM}$, informal version of Theorem~\ref{thm:DPOM_ccl24:formal}]\label{thm:DPOM_ccl24:informal}
Assume Condition~\ref{con:all}. We use the $\mathsf{DPOM}$ algorithm defined in Definition~\ref{def:algo_DPOM}, and let $\hat{q}$ be the output density of it. Then, we have 
\begin{align*}
    \mathrm{TV}(\hat{q}, p_0) \lesssim  (\sqrt{d} + m_2) \exp(-T) + L^2 T d^{1/2} h_{\mathrm{pred}}  + L^{3/2} T d^{1/2} h_{\mathrm{corr}}^{1/2} + L^{1/2} T \epsilon_0 + \epsilon.
\end{align*}
where $m_2 = (\sum_{i=1}^k \alpha_i (\|\mu_i\|_2^2  + \tr[\Sigma_i]))^{1/2}$ and $L = \frac{1}{\sigma_{\min(p_t)}}+  \frac{2R^2}{\gamma^2\sigma_{\min(p_t)}^2} \cdot (\frac{1}{(2\pi)^{d}\det_{\min(p_t)}} + \frac{1}{(2\pi)^{d/2}\det_{\min(p_t)}^{1/2}})  \cdot \exp(-\frac{\beta^2}{2\sigma_{\max(p_t)}})$. 
In particular, if we set $T = \Theta(\log(d  m_2/\epsilon))$, $h_{\mathrm{pred}} = \tilde{\Theta}(\frac{\epsilon}{L^{2}d^{1/2}})$, $h_{\mathrm{corr}} = \tilde{\Theta}(\frac{\epsilon}{L^{3} d})$, and if the score estimation error satisfies $\epsilon_0 \leq \tilde{O}(\frac{\epsilon}{\sqrt{L}})$, then we can obtain TV error $\epsilon$ with a total iteration complexity of $\wt{\Theta}(L^{3}d/\epsilon^2)$ steps.
\end{theorem}

\begin{theorem}[$\mathsf{DPUM}$, informal version of Theorem~\ref{thm:DPUM_ccl24:formal}]\label{thm:DPUM_ccl24:informal}
Assume Condition~\ref{con:all}. We use the $\mathsf{DPUM}$ algorithm defined in Definition~\ref{def:algo_DPUM}, and let $\hat{q}$ be the output density of it. 
Then, we have 
\begin{align*}
    \mathrm{TV}(\hat{q}, p_0) \lesssim  (\sqrt{d} + m_2) \exp(-T) + L^2 T d^{1/2} h_{\mathrm{pred}} +  L^{3/2} T d^{1/2} h_{\mathrm{corr}}^{1/2} + L^{1/2} T \epsilon_0 + \epsilon.
\end{align*}
where $L = \frac{1}{\sigma_{\min(p_t)}}+  \frac{2R^2}{\gamma^2\sigma_{\min(p_t)}^2} \cdot (\frac{1}{(2\pi)^{d}\det_{\min(p_t)}} + \frac{1}{(2\pi)^{d/2}\det_{\min(p_t)}^{1/2}})  \cdot \exp(-\frac{\beta^2}{2\sigma_{\max(p_t)}})$ and  $m_2 = (\sum_{i=1}^k \alpha_i (\|\mu_i\|_2^2  + \tr[\Sigma_i]))^{1/2}$.
In particular, if we set $T = \Theta(\log(d  m_2/\epsilon))$, $h_{\mathrm{pred}} = \tilde{\Theta}(\frac{\epsilon}{L^{2}d^{1/2}})$, $h_{\mathrm{corr}} = \tilde{\Theta}(\frac{\epsilon}{L^{3/2} d^{1/2}})$, and if the score estimation error satisfies $\epsilon_0 \leq \tilde{O}(\frac{\epsilon}{\sqrt{L}})$, then we can obtain TV error $\epsilon$ with a total iteration complexity of $\wt{\Theta}(L^{2}d^{1/2}/\epsilon)$ steps.
\end{theorem}

Theorem~\ref{thm:DPOM_ccl24:informal} and Theorem~\ref{thm:DPUM_ccl24:informal} provide a guarantee for total variation bound between target data distribution $p_0$ and learned output distribution $\hat q$ with concrete $L$ and $M_2$ for ODE reverse process. 
The difference between the $\mathsf{DPOM}$ (Theorem~\ref{thm:DPOM_ccl24:informal}) and $\mathsf{DPUM}$ (Theorem~\ref{thm:DPUM_ccl24:informal}) is the complexity of $h_{\mathrm{corr}}$ and the final iteration complexity term. Using $\mathsf{DPUM}$ algorithm, we can reduce the total iteration complexity from $\tilde{\Theta}(L^{3}d/\epsilon^2)$ to $\tilde{\Theta}(L^{2}\sqrt{d}/\epsilon)$.
\section{Tools From Previous Work}\label{sec:tools_previous}

In this section we present several tools we use from previous work.

\begin{assumption}[Lipschitz score, Assumption 1 in \cite{ccl+22}, page 6]\label{ass:assumption_lipschitz}
For all $t \geq 0$, the score $\nabla \log p_t$ is $L$-Lipschitz.
\end{assumption}

\begin{assumption}[Second momentum bound, Assumption 2 in \cite{ccl+22}, page 6 and  Assumption 2 in \cite{cll23}, page 6]\label{ass:assumption_moment}
We assume that $m_2^2 := M_2 := \E_{p_0}[\|x\|_2^2] < \infty$.
\end{assumption}

\begin{assumption}[Smooth data distributions, Assumption 4 in \cite{cll23}, page 10]\label{ass:assumption_smooth_data}
    The data distribution admits a density $p_0 \in C^2 (\R^d)$ and $\nabla \log p_0$ is $L$-Lipschitz, where $C^2$ means second-order differentiable.
\end{assumption}

\begin{remark}
We can notice $M_2=m_2^2$. 
The theorems from \cite{cll23} use $M_2$ for KL divergence. 
The theorems form \cite{ccl+22} and \cite{ccl+24} use $m_2 = \sqrt{m_2^2}=\sqrt{M_2}$ for total variance, because of Pinsker’s inequality (Fact~\ref{fac:pinsker}).
\end{remark}

We state a tool from previous work \cite{ccl+22},
\begin{theorem}[DDPM, Theorem 2 in \cite{ccl+22}, page 7]\label{thm:theorem2_TV_in_ccl+22}
Suppose that Assumptions~\ref{ass:assumption_lipschitz}, \ref{ass:assumption_moment} and \ref{ass:assumption_score_estimation} hold. Let $\hat{q}_T$ be the output of DDPM algorithm  at times $T$, and suppose that the step size $h := T/N$ satisfies $h \lesssim
 1/L$, where $L \geq 1$. 
Then, it holds that
\begin{align*}
    \mathrm{TV}(\hat{q}_T, p_0) \lesssim \underbrace{\sqrt{\mathrm{KL}(p_0 || \mathcal{N}(0,I)) }\exp(-T)}_{ \mathrm{convergence~of~forward~process}} +  \underbrace{(L\sqrt{dh}  + L m_2 h)\sqrt{T}}_{\mathrm{discretization~ error}} + \underbrace{\epsilon_0\sqrt{T}}_{\mathrm{score~estimation~error}
}.
\end{align*}
\end{theorem}

We state a tool from previous work \cite{cll23}.
\begin{theorem}[Theorem 1 in \cite{cll23}]\label{thm:theorem1_kl_cll23}
Suppose that Assumptions \ref{ass:assumption_lipschitz}, \ref{ass:assumption_moment},
\ref{ass:assumption_score_estimation} hold. If $L \geq 1$, $h_k \leq 1$ for $k = 1, \dots, N$ and $T \geq 1$, using uniform discretization points yields the following:
\begin{itemize}
    \item Using Exponential Integrator scheme \eqref{eq:expontential_integrator_scheme}, we have
    \begin{align*}
        \mathrm{KL}(p_0 \| \hat{q}_T) \lesssim (M_2 + d) \exp(-T) + T\epsilon_0^2 + \frac{d T^2 L^2}{N}.
    \end{align*}
    
    In particular, choosing $T = \Theta(\log ( dM_2/\epsilon_0^2))$ and $N = \Theta ( d T^2 L^2/\epsilon_0^2 )$ makes this $\tilde{O} (\epsilon_0^2)$.
    
    \item Using the Euler-Maruyama scheme \eqref{eq:Euler_Maruyama_scheme}, we have
    \begin{align*}
        \mathrm{KL}(p_0 \| \hat{q}_T) \lesssim (M_2 + d)\exp(-T) + T\epsilon_0^2 + \frac{d T^2 L^2}{N} + \frac{T^3 M_2}{N^2}.
    \end{align*}
\end{itemize}
\end{theorem}

\begin{theorem}[Theorem 5 in \cite{cll23}, page 10]\label{thm:theorem5_kl_cll23}
There is a universal constant $K$ such that the following holds. Under Assumptions \ref{ass:assumption_moment}, \ref{ass:assumption_score_estimation}, and \ref{ass:assumption_smooth_data} hold, by using the exponentially decreasing (then constant) step size $h_k = c \min \{\max \{t_k, \frac{1}{L} \}, 1\},  c = \frac{T+\log L}{N} \leq \frac{1}{Kd}$, the sampling dynamic \eqref{eq:expontential_integrator_scheme} results in a distribution $\hat{q}_T$ such that
\begin{align*}
    \mathrm{KL}(p_0 \| \hat{q}_T) \lesssim (M_2 + d)\exp(-T) + T \epsilon_0^2 + \frac{d^2 (T+\log L)^2}{N}.
\end{align*}
Choosing $T = \Theta(\log ( dM_2/\epsilon_0^2))$ and $N = \Theta ( d^2 (T + \log L)^2/\epsilon_0^2 )$ makes this $\tilde{O} (\epsilon_0^2)$.
In addition, for Euler-Maruyama scheme \eqref{eq:Euler_Maruyama_scheme}, the same bounds hold with an additional $M_2 \sum_{k=1}^N h_k^3$ term.
\end{theorem}

We state a tool from previous work \cite{ccl+24},
\begin{theorem}[DPOM, Theorem 2 in \cite{ccl+24}, page 6]\label{thm:theorem2_DPOM_TV_ccl+24}
Suppose that Assumptions~\ref{ass:assumption_lipschitz}, \ref{ass:assumption_moment} and \ref{ass:assumption_score_estimation} hold. If $\hat{q}_T$ denotes the output of DPOM (see Algorithm 1 in \cite{ccl+24}) with early stopping.
Then, it holds that
\begin{align*}
    \mathrm{TV}(\hat{q}_T, p_0) \lesssim (\sqrt{d} + m_2) \exp(-T) + L^2 T d^{1/2} h_{\mathrm{pred}} + L^{3/2} T d^{1/2} h_{\mathrm{corr}}^{1/2} + L^{1/2} T \epsilon_0 + \epsilon.
\end{align*}
In particular, if we set $T = \Theta(\log(d m_2^2/\epsilon^2))$, $h_{\mathrm{pred}} = \tilde{\Theta}(\frac{\epsilon}{L^{2}d^{1/2}})$, $h_{\mathrm{corr}} = \tilde{\Theta}(\frac{\epsilon}{L^{3} d})$, and if the score estimation error satisfies $\epsilon_0 \leq \tilde{O}(\frac{\epsilon}{\sqrt{L}})$, then we can obtain TV error $\epsilon$ with a total iteration complexity of $\tilde{\Theta}(L^{3}d/\epsilon^2)$ steps.
\end{theorem}

\begin{theorem}[DPUM, Theorem 3 in \cite{ccl+24}, page 7]\label{thm:theorem3_DPUM_TV_ccl+24}
Suppose that Assumptions~\ref{ass:assumption_lipschitz}, \ref{ass:assumption_moment} and \ref{ass:assumption_score_estimation}  hold. If $\hat{q}_T$ denotes the output of DPUM (see Algorithm 2 in \cite{ccl+24}) with early stopping.
Then, it holds that
\begin{align*}
    \mathrm{TV}(\hat{q}_T, p_0) \lesssim (\sqrt{d} + m_2) \exp(-T) + L^2 T d^{1/2} h_{\mathrm{pred}} + L^{3/2} T d^{1/2} h_{\mathrm{corr}}^{1/2} + L^{1/2} T \epsilon_0 + \epsilon.
\end{align*}
In particular, if we set $T = \Theta(\log(d m_2^2/\epsilon^2))$, $h_{\mathrm{pred}} = \tilde{\Theta}(\frac{\epsilon}{L^{2}d^{1/2}})$, $h_{\mathrm{corr}} = \tilde{\Theta}(\frac{\epsilon}{L^{3/2} d^{1/2}})$, and if the score estimation error satisfies $\epsilon_0 \leq \tilde{O}(\frac{\epsilon}{\sqrt{L}})$, then we can obtain TV error $\epsilon$ with a total iteration complexity of $\tilde{\Theta}(L^{2}d^{1/2}/\epsilon)$ steps.
\end{theorem}

\section{Conclusion}\label{sec:conclusion}
We have presented a theoretical analysis of the Lipschitz continuity and second momentum properties of diffusion models, focusing on the case where the target data distribution is a mixture of Gaussian. Our results provide concrete bounds on key properties of the diffusion process and establish error guarantees for various diffusion solvers. These findings contribute to a deeper understanding of diffusion models and pave the way for further theoretical and practical advancements in this field.

\ifdefined\isarxiv
\section*{Acknowledgement}
Research is partially supported by the National Science Foundation (NSF) Grants 2023239-DMS, CCF-2046710, and Air Force Grant FA9550-18-1-0166.
\bibliographystyle{alpha}
\bibliography{ref}
\else
\bibliography{ref}

\newcommand{\etalchar}[1]{$^{#1}$}
\begin{thebibliography}{SDWMG15}

\bibitem[ABDH{\etalchar{+}}18]{abhl+18}
Hassan Ashtiani, Shai Ben-David, Nicholas Harvey, Christopher Liaw, Abbas Mehrabian, and Yaniv Plan.
\newblock Nearly tight sample complexity bounds for learning mixtures of gaussians via sample compression schemes.
\newblock {\em Advances in Neural Information Processing Systems}, 31, 2018.

\bibitem[ADLS17]{adls17}
Jayadev Acharya, Ilias Diakonikolas, Jerry Li, and Ludwig Schmidt.
\newblock Sample-optimal density estimation in nearly-linear time.
\newblock In {\em Proceedings of the Twenty-Eighth Annual ACM-SIAM Symposium on Discrete Algorithms}, pages 1278--1289. SIAM, 2017.

\bibitem[And82]{a82}
Brian~DO Anderson.
\newblock Reverse-time diffusion equation models.
\newblock {\em Stochastic Processes and their Applications}, 12(3):313--326, 1982.

\bibitem[BDD23]{bdd23}
Joe Benton, George Deligiannidis, and Arnaud Doucet.
\newblock Error bounds for flow matching methods.
\newblock {\em arXiv preprint arXiv:2305.16860}, 2023.

\bibitem[BDJ{\etalchar{+}}22]{bdj+22}
Ainesh Bakshi, Ilias Diakonikolas, He~Jia, Daniel~M Kane, Pravesh~K Kothari, and Santosh~S Vempala.
\newblock Robustly learning mixtures of k arbitrary gaussians.
\newblock In {\em Proceedings of the 54th Annual ACM SIGACT Symposium on Theory of Computing}, pages 1234--1247, 2022.

\bibitem[BK20]{bk20}
Ainesh Bakshi and Pravesh Kothari.
\newblock Outlier-robust clustering of non-spherical mixtures.
\newblock {\em arXiv preprint arXiv:2005.02970}, 2020.

\bibitem[BS10]{bs10}
Mikhail Belkin and Kaushik Sinha.
\newblock Polynomial learning of distribution families.
\newblock In {\em 2010 IEEE 51st Annual Symposium on Foundations of Computer Science}, pages 103--112. IEEE, 2010.

\bibitem[BS23]{bs23}
Rares-Darius Buhai and David Steurer.
\newblock Beyond parallel pancakes: Quasi-polynomial time guarantees for non-spherical gaussian mixtures.
\newblock In {\em The Thirty Sixth Annual Conference on Learning Theory}, pages 548--611. PMLR, 2023.

\bibitem[CCL{\etalchar{+}}22]{ccl+22}
Sitan Chen, Sinho Chewi, Jerry Li, Yuanzhi Li, Adil Salim, and Anru~R Zhang.
\newblock Sampling is as easy as learning the score: theory for diffusion models with minimal data assumptions.
\newblock {\em arXiv preprint arXiv:2209.11215}, 2022.

\bibitem[CCL{\etalchar{+}}23]{ccl+24}
Sitan Chen, Sinho Chewi, Holden Lee, Yuanzhi Li, Jianfeng Lu, and Adil Salim.
\newblock The probability flow ode is provably fast.
\newblock {\em Advances in Neural Information Processing Systems}, 36, 2023.

\bibitem[CDD23]{cdd23}
Sitan Chen, Giannis Daras, and Alex Dimakis.
\newblock Restoration-degradation beyond linear diffusions: A non-asymptotic analysis for ddim-type samplers.
\newblock In {\em International Conference on Machine Learning}, pages 4462--4484. PMLR, 2023.

\bibitem[CDSS13]{cdss13}
Siu-On Chan, Ilias Diakonikolas, Xiaorui Sun, and Rocco~A Servedio.
\newblock Learning mixtures of structured distributions over discrete domains.
\newblock In {\em Proceedings of the twenty-fourth annual ACM-SIAM symposium on Discrete algorithms}, pages 1380--1394. SIAM, 2013.

\bibitem[CHZW23]{chzw23}
Minshuo Chen, Kaixuan Huang, Tuo Zhao, and Mengdi Wang.
\newblock Score approximation, estimation and distribution recovery of diffusion models on low-dimensional data.
\newblock In {\em International Conference on Machine Learning}, pages 4672--4712. PMLR, 2023.

\bibitem[CK11]{ck11}
Imre Csisz{\'a}r and J{\'a}nos K{\"o}rner.
\newblock {\em Information theory: coding theorems for discrete memoryless systems}.
\newblock Cambridge University Press, 2011.

\bibitem[CKS24]{cks24}
Sitan Chen, Vasilis Kontonis, and Kulin Shah.
\newblock Learning general gaussian mixtures with efficient score matching.
\newblock {\em arXiv preprint arXiv:2404.18893}, 2024.

\bibitem[CLL23]{cll23}
Hongrui Chen, Holden Lee, and Jianfeng Lu.
\newblock Improved analysis of score-based generative modeling: User-friendly bounds under minimal smoothness assumptions.
\newblock In {\em International Conference on Machine Learning}, pages 4735--4763. PMLR, 2023.

\bibitem[CZZ{\etalchar{+}}20]{czz+20}
Nanxin Chen, Yu~Zhang, Heiga Zen, Ron~J Weiss, Mohammad Norouzi, and William Chan.
\newblock Wavegrad: Estimating gradients for waveform generation.
\newblock {\em arXiv preprint arXiv:2009.00713}, 2020.

\bibitem[Das99]{d99}
Sanjoy Dasgupta.
\newblock Learning mixtures of gaussians.
\newblock In {\em 40th Annual Symposium on Foundations of Computer Science (Cat. No. 99CB37039)}, pages 634--644. IEEE, 1999.

\bibitem[DB22]{d22}
Valentin De~Bortoli.
\newblock Convergence of denoising diffusion models under the manifold hypothesis.
\newblock {\em arXiv preprint arXiv:2208.05314}, 2022.

\bibitem[DHKK20]{dhkk20}
Ilias Diakonikolas, Samuel~B Hopkins, Daniel Kane, and Sushrut Karmalkar.
\newblock Robustly learning any clusterable mixture of gaussians.
\newblock {\em arXiv preprint arXiv:2005.06417}, 2020.

\bibitem[DK14]{dk14}
Constantinos Daskalakis and Gautam Kamath.
\newblock Faster and sample near-optimal algorithms for proper learning mixtures of gaussians.
\newblock In {\em Conference on Learning Theory}, pages 1183--1213. PMLR, 2014.

\bibitem[DK20]{dk20}
Ilias Diakonikolas and Daniel~M Kane.
\newblock Small covers for near-zero sets of polynomials and learning latent variable models.
\newblock In {\em 2020 IEEE 61st Annual Symposium on Foundations of Computer Science (FOCS)}, pages 184--195. IEEE, 2020.

\bibitem[DKK{\etalchar{+}}19]{dkk+19}
Ilias Diakonikolas, Gautam Kamath, Daniel Kane, Jerry Li, Ankur Moitra, and Alistair Stewart.
\newblock Robust estimators in high-dimensions without the computational intractability.
\newblock {\em SIAM Journal on Computing}, 48(2):742--864, 2019.

\bibitem[FOS08]{fos08}
Jon Feldman, Ryan O'Donnell, and Rocco~A Servedio.
\newblock Learning mixtures of product distributions over discrete domains.
\newblock {\em SIAM Journal on Computing}, 37(5):1536--1564, 2008.

\bibitem[GKL24]{gkl24}
Khashayar Gatmiry, Jonathan Kelner, and Holden Lee.
\newblock Learning mixtures of gaussians using diffusion models.
\newblock {\em arXiv preprint arXiv:2404.18869}, 2024.

\bibitem[GLB{\etalchar{+}}24]{glb+24}
Hanzhong Guo, Cheng Lu, Fan Bao, Tianyu Pang, Shuicheng Yan, Chao Du, and Chongxuan Li.
\newblock Gaussian mixture solvers for diffusion models.
\newblock {\em Advances in Neural Information Processing Systems}, 36, 2024.

\bibitem[HJA20]{hja20}
Jonathan Ho, Ajay Jain, and Pieter Abbeel.
\newblock Denoising diffusion probabilistic models.
\newblock {\em Advances in neural information processing systems}, 33:6840--6851, 2020.

\bibitem[HO07]{ho07}
John~R Hershey and Peder~A Olsen.
\newblock Approximating the kullback leibler divergence between gaussian mixture models.
\newblock In {\em 2007 IEEE International Conference on Acoustics, Speech and Signal Processing-ICASSP'07}, volume~4, pages IV--317. IEEE, 2007.

\bibitem[HRX24]{hrx24}
Yinbin Han, Meisam Razaviyayn, and Renyuan Xu.
\newblock Neural network-based score estimation in diffusion models: Optimization and generalization.
\newblock In {\em The Twelfth International Conference on Learning Representations}, 2024.

\bibitem[HWSL24]{hwsl24}
Jerry Yao-Chieh Hu, Weimin Wu, Zhao Song, and Han Liu.
\newblock On statistical rates and provably efficient criteria of latent diffusion transformers (dits).
\newblock {\em arXiv preprint arXiv:2407.01079}, 2024.

\bibitem[KFL22]{kfl22}
Dohyun Kwon, Ying Fan, and Kangwook Lee.
\newblock Score-based generative modeling secretly minimizes the wasserstein distance.
\newblock {\em Advances in Neural Information Processing Systems}, 35:20205--20217, 2022.

\bibitem[KLL{\etalchar{+}}24]{kll+24}
Dongjun Kim, Chieh-Hsin Lai, Wei-Hsiang Liao, Naoki Murata, Yuhta Takida, Toshimitsu Uesaka, Yutong He, Yuki Mitsufuji, and Stefano Ermon.
\newblock Consistency trajectory models: Learning probability flow {ODE} trajectory of diffusion.
\newblock In {\em The Twelfth International Conference on Learning Representations}, 2024.

\bibitem[KPH{\etalchar{+}}20]{kph+20}
Zhifeng Kong, Wei Ping, Jiaji Huang, Kexin Zhao, and Bryan Catanzaro.
\newblock Diffwave: A versatile diffusion model for audio synthesis.
\newblock {\em arXiv preprint arXiv:2009.09761}, 2020.

\bibitem[LKB{\etalchar{+}}23]{lkb+23}
Jae~Hyun Lim, Nikola~B Kovachki, Ricardo Baptista, Christopher Beckham, Kamyar Azizzadenesheli, Jean Kossaifi, Vikram Voleti, Jiaming Song, Karsten Kreis, Jan Kautz, et~al.
\newblock Score-based diffusion models in function space, 2023.
\newblock {\em URL https://arxiv. org/abs/2302.07400}, 2023.

\bibitem[LLSS24]{llss24_softmax}
Chenyang Li, Yingyu Liang, Zhenmei Shi, and Zhao Song.
\newblock Exploring the frontiers of softmax: Provable optimization, applications in diffusion model, and beyond.
\newblock {\em arXiv preprint arXiv:2405.03251}, 2024.

\bibitem[LLT22]{llt22}
Holden Lee, Jianfeng Lu, and Yixin Tan.
\newblock Convergence for score-based generative modeling with polynomial complexity.
\newblock {\em Advances in Neural Information Processing Systems}, 35:22870--22882, 2022.

\bibitem[LS17]{ls17}
Jerry Li and Ludwig Schmidt.
\newblock Robust and proper learning for mixtures of gaussians via systems of polynomial inequalities.
\newblock In {\em Conference on Learning Theory}, pages 1302--1382. PMLR, 2017.

\bibitem[LSSS24]{lsss24}
Yingyu Liang, Zhizhou Sha, Zhenmei Shi, and Zhao Song.
\newblock Differential privacy mechanisms in neural tangent kernel regression.
\newblock {\em arXiv preprint arXiv:2407.13621}, 2024.

\bibitem[MV10]{mv10}
Ankur Moitra and Gregory Valiant.
\newblock Settling the polynomial learnability of mixtures of gaussians.
\newblock In {\em 2010 IEEE 51st Annual Symposium on Foundations of Computer Science}, pages 93--102. IEEE, 2010.

\bibitem[ND21]{nd21}
Alexander~Quinn Nichol and Prafulla Dhariwal.
\newblock Improved denoising diffusion probabilistic models.
\newblock In {\em International conference on machine learning}, pages 8162--8171. PMLR, 2021.

\bibitem[NRW21]{nrw21}
Eliya Nachmani, Robin~San Roman, and Lior Wolf.
\newblock Non gaussian denoising diffusion models.
\newblock {\em arXiv preprint arXiv:2106.07582}, 2021.

\bibitem[Ope24]{sora}
OpenAI.
\newblock Video generation models as world simulators, 2024.
\newblock \url{https://openai.com/research/video-generation-models-as-world-simulators}.

\bibitem[PKK22]{pkk22}
Jo{\~a}o~M Pereira, Joe Kileel, and Tamara~G Kolda.
\newblock Tensor moments of gaussian mixture models: Theory and applications.
\newblock {\em arXiv preprint arXiv:2202.06930}, 2022.

\bibitem[PX23]{px23}
William Peebles and Saining Xie.
\newblock Scalable diffusion models with transformers.
\newblock In {\em Proceedings of the IEEE/CVF International Conference on Computer Vision}, pages 4195--4205, 2023.

\bibitem[RBL{\etalchar{+}}22]{rbl+22}
Robin Rombach, Andreas Blattmann, Dominik Lorenz, Patrick Esser, and Bj{\"o}rn Ommer.
\newblock High-resolution image synthesis with latent diffusion models.
\newblock In {\em Proceedings of the IEEE/CVF conference on computer vision and pattern recognition}, pages 10684--10695, 2022.

\bibitem[RFB15]{rfb15}
Olaf Ronneberger, Philipp Fischer, and Thomas Brox.
\newblock U-net: Convolutional networks for biomedical image segmentation.
\newblock In {\em Medical image computing and computer-assisted intervention--MICCAI 2015: 18th international conference, Munich, Germany, October 5-9, 2015, proceedings, part III 18}, pages 234--241. Springer, 2015.

\bibitem[RGKZ21]{rgkz21}
Maria Refinetti, Sebastian Goldt, Florent Krzakala, and Lenka Zdeborov{\'a}.
\newblock Classifying high-dimensional gaussian mixtures: Where kernel methods fail and neural networks succeed.
\newblock In {\em International Conference on Machine Learning}, pages 8936--8947. PMLR, 2021.

\bibitem[RHX{\etalchar{+}}23]{rhx+23}
Alex Reneau, Jerry Yao-Chieh Hu, Chenwei Xu, Weijian Li, Ammar Gilani, and Han Liu.
\newblock Feature programming for multivariate time series prediction.
\newblock In {\em Fortieth International Conference on Machine Learning (ICML)}, 2023.

\bibitem[SCK23]{sck23}
Kulin Shah, Sitan Chen, and Adam Klivans.
\newblock Learning mixtures of gaussians using the ddpm objective.
\newblock {\em Advances in Neural Information Processing Systems}, 36:19636--19649, 2023.

\bibitem[SCL{\etalchar{+}}22]{scl+22}
Zhenmei Shi, Jiefeng Chen, Kunyang Li, Jayaram Raghuram, Xi~Wu, Yingyu Liang, and Somesh Jha.
\newblock The trade-off between universality and label efficiency of representations from contrastive learning.
\newblock In {\em The Eleventh International Conference on Learning Representations}, 2022.

\bibitem[Sco15]{s15}
David~W Scott.
\newblock {\em Multivariate density estimation: theory, practice, and visualization}.
\newblock John Wiley \& Sons, 2015.

\bibitem[SDCS23]{sdcs23}
Yang Song, Prafulla Dhariwal, Mark Chen, and Ilya Sutskever.
\newblock Consistency models.
\newblock In {\em Proceedings of the 40th International Conference on Machine Learning}, pages 32211--32252, 2023.

\bibitem[SDME21]{sdme21}
Yang Song, Conor Durkan, Iain Murray, and Stefano Ermon.
\newblock Maximum likelihood training of score-based diffusion models.
\newblock {\em Advances in neural information processing systems}, 34:1415--1428, 2021.

\bibitem[SDWMG15]{swmg15}
Jascha Sohl-Dickstein, Eric Weiss, Niru Maheswaranathan, and Surya Ganguli.
\newblock Deep unsupervised learning using nonequilibrium thermodynamics.
\newblock In {\em International conference on machine learning}, pages 2256--2265. PMLR, 2015.

\bibitem[SE19]{se19}
Yang Song and Stefano Ermon.
\newblock Generative modeling by estimating gradients of the data distribution.
\newblock {\em Advances in neural information processing systems}, 32, 2019.

\bibitem[SE20]{se20}
Yang Song and Stefano Ermon.
\newblock Improved techniques for training score-based generative models.
\newblock {\em Advances in neural information processing systems}, 33:12438--12448, 2020.

\bibitem[SGSE20]{sgse20}
Yang Song, Sahaj Garg, Jiaxin Shi, and Stefano Ermon.
\newblock Sliced score matching: A scalable approach to density and score estimation.
\newblock In {\em Uncertainty in Artificial Intelligence}, pages 574--584. PMLR, 2020.

\bibitem[SHC{\etalchar{+}}22]{shc+22}
Chitwan Saharia, Jonathan Ho, William Chan, Tim Salimans, David~J Fleet, and Mohammad Norouzi.
\newblock Image super-resolution via iterative refinement.
\newblock {\em IEEE transactions on pattern analysis and machine intelligence}, 45(4):4713--4726, 2022.

\bibitem[SK21]{sk21}
Yang Song and Diederik~P Kingma.
\newblock How to train your energy-based models.
\newblock {\em arXiv preprint arXiv:2101.03288}, 2021.

\bibitem[SMF{\etalchar{+}}24]{smf+24}
Zhenmei Shi, Yifei Ming, Ying Fan, Frederic Sala, and Yingyu Liang.
\newblock Domain generalization via nuclear norm regularization.
\newblock In {\em Conference on Parsimony and Learning}, pages 179--201. PMLR, 2024.

\bibitem[SOAJ14]{soaj14}
Ananda~Theertha Suresh, Alon Orlitsky, Jayadev Acharya, and Ashkan Jafarpour.
\newblock Near-optimal-sample estimators for spherical gaussian mixtures.
\newblock {\em Advances in Neural Information Processing Systems}, 27, 2014.

\bibitem[SSDK{\etalchar{+}}20]{ssk+20}
Yang Song, Jascha Sohl-Dickstein, Diederik~P Kingma, Abhishek Kumar, Stefano Ermon, and Ben Poole.
\newblock Score-based generative modeling through stochastic differential equations.
\newblock {\em arXiv preprint arXiv:2011.13456}, 2020.

\bibitem[SWL21]{swl21}
Zhenmei Shi, Junyi Wei, and Yingyu Liang.
\newblock A theoretical analysis on feature learning in neural networks: Emergence from inputs and advantage over fixed features.
\newblock In {\em International Conference on Learning Representations}, 2021.

\bibitem[SWL24]{swl24}
Zhenmei Shi, Junyi Wei, and Yingyu Liang.
\newblock Provable guarantees for neural networks via gradient feature learning.
\newblock {\em Advances in Neural Information Processing Systems}, 36, 2024.

\bibitem[Vin04]{v04}
Susana Vinga.
\newblock Convolution integrals of normal distribution functions.
\newblock {\em Supplementary material to Vinga and Almeida (2004)“R{\'e}nyi continuous entropy of DNA sequences}, 2004.

\bibitem[WCL{\etalchar{+}}24]{wcl+24}
Yuchen Wu, Minshuo Chen, Zihao Li, Mengdi Wang, and Yuting Wei.
\newblock Theoretical insights for diffusion guidance: A case study for gaussian mixture models.
\newblock {\em arXiv preprint arXiv:2403.01639}, 2024.

\bibitem[WCZ{\etalchar{+}}23]{wcz+23}
Yilin Wang, Zeyuan Chen, Liangjun Zhong, Zheng Ding, Zhizhou Sha, and Zhuowen Tu.
\newblock Dolfin: Diffusion layout transformers without autoencoder.
\newblock {\em arXiv preprint arXiv:2310.16305}, 2023.

\bibitem[WSD{\etalchar{+}}23]{wsd+23}
Zirui Wang, Zhizhou Sha, Zheng Ding, Yilin Wang, and Zhuowen Tu.
\newblock Tokencompose: Grounding diffusion with token-level supervision.
\newblock {\em arXiv preprint arXiv:2312.03626}, 2023.

\bibitem[WXZ{\etalchar{+}}24]{wxz+24_omni}
Yilin Wang, Haiyang Xu, Xiang Zhang, Zeyuan Chen, Zhizhou Sha, Zirui Wang, and Zhuowen Tu.
\newblock Omnicontrolnet: Dual-stage integration for conditional image generation.
\newblock In {\em Proceedings of the IEEE/CVF Conference on Computer Vision and Pattern Recognition}, pages 7436--7448, 2024.

\bibitem[XSW{\etalchar{+}}23]{xsw+23}
Zhuoyan Xu, Zhenmei Shi, Junyi Wei, Fangzhou Mu, Yin Li, and Yingyu Liang.
\newblock Towards few-shot adaptation of foundation models via multitask finetuning.
\newblock In {\em The Twelfth International Conference on Learning Representations}, 2023.

\bibitem[YFZ{\etalchar{+}}23]{yfz+23}
Zhantao Yang, Ruili Feng, Han Zhang, Yujun Shen, Kai Zhu, Lianghua Huang, Yifei Zhang, Yu~Liu, Deli Zhao, Jingren Zhou, et~al.
\newblock Lipschitz singularities in diffusion models.
\newblock In {\em The Twelfth International Conference on Learning Representations}, 2023.

\bibitem[ZHF{\etalchar{+}}24]{zhf+24}
Jianbin Zheng, Minghui Hu, Zhongyi Fan, Chaoyue Wang, Changxing Ding, Dacheng Tao, and Tat-Jen Cham.
\newblock Trajectory consistency distillation.
\newblock {\em arXiv preprint arXiv:2402.19159}, 2024.

\end{thebibliography}
\bibliographystyle{plain}
\fi

\newpage
\onecolumn
\appendix
\begin{center}
	\textbf{\LARGE Appendix }
\end{center}

\ifdefined\isarxiv

\else

{\hypersetup{linkcolor=black}
\tableofcontents

\newpage
}
\fi

{\bf Roadmap.}

\begin{itemize}
    \item Section~\ref{sec:limitations} discusses the limitations of the paper.
    \item Section~\ref{sec:impacts} discusses the societal impacts of the paper.
    \item Section~\ref{sec:preliminary} provides the preliminary for the paper.
    \item Section~\ref{sec:continuous} provides the case when we consider a continuous time score function, specifically a single Gaussian.
    \item Section~\ref{sec:general_two_gaussian} provides the case when we consider the score function to be 2 mixtures of Gaussians.
    \item Section~\ref{sec:k_gaussain} provides the case when we consider the score function to be $k$ mixture of Gaussians.
    \item Section~\ref{sec:all_together} provides lemmas that we use for a more concrete calculation for theorems in Section~\ref{sec:tools_previous}.
    \item Section~\ref{sec:app:application} provides our main results when we consider the data distribution is $k$ mixture of Gaussians.
\end{itemize}

\section{Limitations}\label{sec:limitations}

This work has not directly addressed the practical applications of our results. Additionally, we did not provide a sample complexity bound for our settings. Future research could explore how these findings might be implemented in real-world scenarios and work on improving these limitations.

\section{Societal Impacts}\label{sec:impacts}
We explore and provide a deeper understanding of the diffusion models and also explicitly give the Lipschitz constant for $k$-mixture of Gaussians, which may inspire a better algorithm design. 

On the other hand, our paper is purely theoretical in nature, so we foresee no immediate negative ethical impact.

\section{Preliminary}\label{sec:preliminary}
This section provides some preliminary knowledge and is organized as below:
\begin{itemize}
    \item Section~\ref{sec:preliminary:facts} provides the facts we use.
    \item Section~\ref{sec:preliminary:properties_exp} provides the property of exp function we use.
    \item Section~\ref{sec:preliminary:property_lip} provides the Lipschitz multiplication property we use.
\end{itemize}

\subsection{Facts}\label{sec:preliminary:facts}

We provide several basic facts from calculus and linear algebra that are used in the proofs.

\begin{fact}[Calculus]\label{fac:calculus}
For $x \in \R$, $y \in \R$, $t \in \R$, $u \in \R^n$, $v \in \R^n$, it is well-known that
\begin{itemize}
    \item $\frac{\d x}{\d t} = \frac{\d x}{\d y}\frac{\d y}{\d t}$ (chain rule)
    \item  $\frac{\d xy}{\d t} = \frac{\d x}{\d t}y + \frac{\d y}{\d t}x$ (product rule)
    \item $\frac{\d x^n}{\d x} = nx^{n-1}$ (power rule)
    \item $\frac{\d \langle u, v\rangle}{\d u} = v$ (derivative of the inner product)
    \item $\frac{\d \exp (x)}{\d x}= \exp (x)$ (derivative of exponential function)
    \item $\frac{\d \log x}{\d x} = \frac{1}{x}$ (derivative of logarithm function)
    \item $\frac{\d }{\d u} \|u\|_2^2 = 2u$ (derivative of $\ell_2$ norm)
    \item $\frac{\d y}{\d x} = 0$ if $y$ is independent from $x$. (derivative of independent variables)
\end{itemize}
\end{fact}

\begin{fact}[Norm Bounds]\label{fac:norm_bounds}
For $a \in \R$, $b \in \R$, $u \in \R^n$, $v \in \R^n$, $A \in \R^{k \times n}$, $W \in \R^{n \times n}$ is symmetric and p.s.d., we have
\begin{itemize}
    \item $\| a u\|_2 = |a| \cdot \| u\|_2$ (absolute homogeneity)
    \item $\|u+v\|_2\leq \|u\|_2 + \|v\|_2$ (triangle inequality)
    \item $|u^\top v| \leq \|u\|_2 \cdot \|v\|_2$ (Cauchy–Schwarz inequality)
    \item $\|u^\top\|_2 = \|u\|_2$
    \item $\|Au\|_2 \leq \|A\| \cdot \|u\|_2$
    \item $\| a A\| = |a| \cdot \| A\|$
    \item $\|A\|=\sigma_{\max}(A)$
    \item $\|A^{-1}\|=\frac{1}{\sigma_{\min}(A)}$
    \item $u^\top W u \geq \|u\|_2^2 \cdot \sigma_{\min}(W)$. 
    \item $\sigma_{\min}(W^{-1}) = \frac{1}{\sigma_{\max}(W)}$.
\end{itemize}
\end{fact}

\begin{fact}[Matrix Calculus]\label{fac:matrix_calculus}
Let $W \in \R^{n \times n}$ denote a symmetric matrix. Let $x \in \R^n$ and $s \in \R^n$. Suppose that $s$ is independent of $x$. Then, we know
\begin{itemize}
    \item $\frac{\d}{\d x} (x-s)^\top W (x-s) = 2 W (x-s)$
\end{itemize}
\end{fact}

\subsection{Properties of \texorpdfstring{$\exp$}{} functions}\label{sec:preliminary:properties_exp}

During the course of proving the Lipschitz continuity for mixtures of Gaussians, we found that we need to use the following bound for the $\exp$ function.
\begin{fact}
For $|a-b|\leq 0.1$, where $a \in \R$, $b \in \R$, we have
\begin{align*}
 |\exp(a)-\exp(b)| \leq |\exp(a)| \cdot 2|a-b|   
\end{align*}
\end{fact}
\begin{proof}
We have
\begin{align*}
    |\exp(a)-\exp(b)| = & ~ |\exp(a) \cdot (1 -\exp(b-a))|\\
    = & ~ |\exp(a)| \cdot |(1 -\exp(b-a))|\\
    \leq & ~ |\exp(a)| \cdot 2|a-b|
\end{align*}
where the first step follows from simple algebra, the second step follows from $|a \cdot b| = |a| \cdot |b|$, and the last step follows from $|\exp(x) - 1| \leq 2x$ for all $x \in (0, 0.1)$.
\end{proof}

\begin{fact}
For $\|u-v\|_{\infty} \leq 0.1$, where $u,v \in \R^n$, we have 
\begin{align*}
 \|\exp(u)-\exp(v)\|_2 \leq \|\exp(u)\|_2 \cdot 2\|u-v\|_{\infty}
\end{align*}
\end{fact}
\begin{proof}
We have
\begin{align*}
    \|\exp(u)-\exp(v)\|_2 = & ~ \|\exp(u) \circ ({\bf 1}_n -\exp(v-u))\|_2\\
    \leq & ~ \|\exp(u)\|_2 \cdot \|{\bf 1}_n -\exp(v-u)\|_{\infty}\\
    \leq & ~ \|\exp(u)\|_2 \cdot 2\|u-v\|_{\infty}
\end{align*}
where the first step follows from notation of Hardamard product, the second step follows from $\|u \circ v\|_2 \leq \|u\|_{\infty} \cdot \|v\|_2$, and the last step follows from $|\exp(x) - 1| \leq 2x$ for all $x \in (0, 0.1)$.
\end{proof}

\begin{fact} [Mean value theorem for vector function] \label{fac:mvt}
For vector $x, y \in C \subset \R^n$, vector function $f(x): C \to \R$, $g(x): C \to \R^m$, let $f,g$ be differentiable on open convex domain $C$, we have
\begin{itemize}
    \item Part 1: $f(y) - f(x) = \nabla f(x+t(y-x))^\top (y-x)$
    \item Part 2: $\| g(y) - g(x) \|_2 \leq \| g'(x+t(y-x))\| \cdot \| y - x \|_2$ for some $t \in (0,1)$, where $g'(a)$ denotes a matrix which its $(i,j)$-th term is $\frac{\d g(a)_j}{\d a_i}$. 
    \item Part 3: If $\| g'(a) \| \leq M$ for all $a \in C$, then $\| g(y) - g(x) \|_2 \leq M \| y - x \|_2$ for all $x, y \in C$.
\end{itemize}
\end{fact}
\begin{proof}
{\bf Proof of Part 1}

{\bf Part 1} can be verified by applying Mean Value Theorem of 1-variable function on $\gamma(c) = f(x + c(y-x))$.
\begin{align*}
    f(y) - f(x) = \gamma(1) - \gamma(0) = \gamma'(t) (1-0) = \nabla f(x+t(y-x))^\top (y-x)
\end{align*}
where $t \in (0,1)$.

{\bf Proof of Part 2}

Let $G(c) := (g(y) - g(x))^\top g(c)$, we have
\begin{align*}
    \| g(y) - g(x) \|_2^2 = & ~ G(y) - G(x) \\
    = & ~ \nabla G(x+t(y-x)) ^\top (y-x) \\
    = & ~ (\underbrace{g'(x+t(y-x))}_{n \times m} \cdot \underbrace{(g(y) - g(x))}_{m \times 1})^\top \cdot \underbrace{(y-x)}_{n \times 1} \\
    \leq & ~ \| g'(x+t(y-x)) \| \cdot \| g(y) - g(x) \|_2 \cdot \| y-x \|_2
\end{align*}
the initial step is by basic calculation, the second step is from {\bf Part 1}, the third step uses chain rule, the 4th step is due to Cauchy-Schwartz inequality. Removing $\| g(y) - g(x) \|_2$ on both sides gives the result.

{\bf Proof of Part 3}

{\bf Part 3} directly follows from {\bf Part 2}.
\end{proof}

\begin{fact}\label{fac:upper_bound_exp_prime}
Let $g'(a)$ denotes a matrix whose $(i,j)$-th term is $\frac{\d g(a)_j}{\d a_i}$.
For $u \in \R^n$, $v \in \R^n$, $\|u\|_2, \|v\|_2 \leq R$, where $R \geq 0$, $t \in (0,1)$, we have
\begin{align*}
    \|\exp'(u + t(v-u))\| \leq \exp(R)
\end{align*}
\end{fact}

\begin{proof}
We can show
\begin{align*}
    \|\exp'(u + t(v-u))\| = & ~ \|\mathrm{diag}(\exp(u + t(v-u)))\|\\
    \leq & ~ \sigma_{\max}(\mathrm{diag}(\exp(u + t(v-u))))\\
    = & ~ \max_{i \in [n]} \exp(u_i + t(v_i-u_i))\\
    \leq & ~ \max_{i \in [n]} 
 \max\{\exp(v_i), \exp(u_i)\}\\ 
    \leq & ~ \exp(R)
\end{align*}
where the first step follows from $\frac{\d \exp(x)}{\d x} = \mathrm{diag} (\exp (x))$, the second step follows from Fact~\ref{fac:norm_bounds}, the third step follows from spectral norm of a diagonal matrix is the absolute value of its largest entry, the fourth step follows from $t \in (0,1)$, and the last step follows from $\|\exp(v)\|_{\infty} \leq \exp(\|v\|_{\infty}) \leq \exp(\|v\|_2)$.
\end{proof}

\begin{fact}
For $u \in \R^n$, $v \in \R^n$, $\|u\|_2, \|v\|_2 \leq R$, where $R \geq 0$, we have
\begin{align*}
    \|\exp(u)-\exp(v)\|_2 \leq \exp(R)\|u-v\|_2
\end{align*}
\end{fact}

\begin{proof}
We can show, for $t \in (0,1)$, 
\begin{align*}
    \|\exp(u)-\exp(v)\|_2 \leq & ~ \|\exp'(u + t(v-u))\| \cdot \|u-v\|_2\\
    \leq & ~ \exp(R)\|u-v\|_2
\end{align*}
where the first step follows from Fact~\ref{fac:mvt}, the second step follows from Fact~\ref{fac:upper_bound_exp_prime}.
\end{proof}

\begin{fact}\label{fac:upper_bound_exp}
For $a \in \R$, $b \in \R$, $a, b \leq R$, where $R \geq 0$, we have
\begin{align*}
    |\exp(a)-\exp(b)| \leq \exp(R)|a-b|
\end{align*}
\end{fact}

\begin{proof}
We can show, for $t \in (0,1)$, 
\begin{align*}
    |\exp(a)-\exp(b)| = & ~ |\exp'(a + t(b-a))| \cdot |a-b|\\
    = & ~ |\exp(a + t(b-a))| \cdot |a-b|\\
    \leq & ~ \max\{\exp(a),\exp(b)\} \cdot |a-b|\\
    \leq & ~ \exp(R) |u-v|
\end{align*}
where the first step follows from Mean Value Theorem, the second step follows from Fact~\ref{fac:calculus}, the third step follows from $t\in (0,1)$, and the last step follows from $a,b\leq R$.

\end{proof}

\subsection{Lipschitz multiplication property}\label{sec:preliminary:property_lip}
Our overall proofs of Lipschitz constant for $k$-mixture of Gaussians follow the idea from Fact below.

\begin{fact}
If the following conditions hold
\begin{itemize}
    \item $\| f_i(x) -  f_i(y) \|_2 \leq L \cdot \| x - y \|_2$
    \item $R:= \max_{i \in [n], x} |f_i(x)|$
\end{itemize}
Then, we have
\begin{align*}
    |\prod_{i=1}^k f_i(x) - \prod_{i=1}^k f_i(y) | \leq k \cdot R^{k-1} \cdot L \cdot \| x- y \|_2
\end{align*}
\end{fact}
\begin{proof}
We can show
\begin{align*}
    & |\prod_{i=1}^k f_i(x) - \prod_{i=1}^k f_i(y)| \\
    = & ~ |f_k(x) \prod_{i=1}^{k-1} f_i(x) - f_k(y) \prod_{i=1}^{k-1} f_i(y)|\\
    \leq & ~ |f_k(x) \prod_{i=1}^{k-1} f_i(x) - f_k(y) \prod_{i=1}^{k-1} f_i(x)| + |f_k(y) \prod_{i=1}^{k-1} f_i(x) - f_k(y) \prod_{i=1}^{k-1} f_i(y)|\\
    = & ~ |(f_k(x) - f_k(y) ) \prod_{i=1}^{k-1} f_i(x) | + |f_k(y) (\prod_{i=1}^{k-1} f_i(x) - \prod_{i=1}^{k-1} f_i(y))|\\
    \leq & ~ L \cdot \| x-y \|_2 \cdot R^{k-1} + R \cdot |\prod_{i=1}^{k-1} f_i(x) - \prod_{i=1}^{k-1} f_i(y)|\\
    \leq & ~ L \cdot \| x-y \|_2 \cdot R^{k-1} + R \cdot (|L \cdot \| x-y \|_2 \cdot R^{k-2} + R \cdot |\prod_{i=1}^{k-2} f_i(x) - \prod_{i=1}^{k-2} f_i(y)|)\\
    = & ~ 2 \cdot L \cdot \| x-y \|_2 \cdot R^{k-1} +  R^2 \cdot |\prod_{i=1}^{k-2} f_i(x) - \prod_{i=1}^{k-2} f_i(y)|\\
    \leq & ~ k \cdot R^{k-1} \cdot L \cdot \| x- y \|_2
\end{align*}
where the first step follows from simple algebra, the second step follows from Fact~\ref{fac:norm_bounds}, the third step follows from rearranging terms, the fourth step follows from the assumptions of the lemma, the fifth step follows from the same logic of above, the sixth step follows from simple algebra, and the last step follows from the recursive process.
\end{proof}

\section{Single Gaussian Case}\label{sec:continuous}

In this section, we consider the continuous case of $p_t(x)$, which is the probability density function ($\mathsf{pdf}$) of the input data $x$, and also a function of time $t$. More specifically, we consider the cases when $p_t(x)$ is: (1) a single Gaussian where either the mean is a function of time (Section~\ref{sec:single_gaussian:mean_function_of_time}) or the covariance is a function of time (Section~\ref{sec:single_gaussian:cov_function_of_time}), (2) a single Gaussian where both the mean and the covariance are a function of time (Section~\ref{sec:single_gaussian:mean_cov_function_of_time}). And then, we compute the upper bound and Lipschitz constant for the score function i.e. the gradient of log $\mathsf{pdf}$ $\frac{\d \log p_t(x)}{\d x}$.

\subsection{Case when the mean of \texorpdfstring{$p_t(x)$}{} is a function of time}\label{sec:single_gaussian:mean_function_of_time}

We start our calculation by a simple case.
Consider $p_t $ such that
\begin{align*}
    p_t(x) = \Pr_{x' \sim \mathcal{N}( t {\bf 1}_d, I_{d }) }[ x' = x]
\end{align*}
Let $\mathsf{pdf}$ is $\R^d \rightarrow \R$ denote $p_t$. We have
$\log (\mathsf{pdf} ()) $ is $\R^d \rightarrow \R$. Then, we can get gradient $\nabla \log (\mathsf{pdf} ())$ is a function of $t$ because of $\mathcal{N}( t {\bf 1}_d, I_{d })$. Inject $x$ and $y$ into the gradient function, then we are done.

Below we define the $\mathsf{pdf}$ for the continues case when the mean is a function of time.

\begin{definition}\label{def:p_t_continuous}
If the following conditions hold
\begin{itemize}
    \item Let $x \in \R^d$.
    \item Let $t \in \R$, and $t \geq 0$.
\end{itemize}

We define
\begin{align*}
    p_t(x) := \frac{1}{(2\pi)^{d/2}} \exp(-\frac{1}{2} \|x - t \mathbf{1}_d\|_2^2)
\end{align*}

Further, we have
\begin{align*}
    \log p_t(x) = - \frac{d}{2}\log (2\pi) -\frac{1}{2} \|x - t \mathbf{1}_d\|_2^2
\end{align*}
\end{definition}

Below we calculate the score function of $\mathsf{pdf}$ for the continuous case when the mean is a function of time.
\begin{lemma}\label{lem:gradient_log_p_t_continuous}
If the following conditions hold
\begin{itemize}
    \item Let $x \in \R^d$.
    \item Let $t \in \R$, and $t \geq 0$.
\end{itemize} 
Then, 
\begin{align*}
    \frac{\d \log p_t(x)}{\d x} = t\mathbf{1}_d - x
\end{align*}
\end{lemma}

\begin{proof}
We can show 
\begin{align*}
    \frac{\d \log p_t(x)}{\d x} = & ~ \frac{\d}{\d x} (- \frac{d}{2}\log (2\pi) -\frac{1}{2} \|x - t \mathbf{1}_d\|_2^2)\\
    = & ~ -\frac{1}{2} \cdot \frac{\d}{\d x} \|x - t \mathbf{1}_d\|_2^2\\
    = & ~ - \frac{1}{2} \cdot 2 (x - t \mathbf{1}_d)\\
    = & ~ t\mathbf{1}_d - x
\end{align*}
where the first step follows from Definition~\ref{def:p_t_continuous}, the second step follows from variables are independent, the third step follows from Fact~\ref{fac:calculus}, and the last step follows from simple algebra.
\end{proof}

Below we calculate the upper bound for the score function of $\mathsf{pdf}$ for continuous case when the mean is a function of time.
\begin{lemma}[Linear growth]
If the following conditions hold
\begin{itemize}
    \item Let $x \in \R^d$.
    \item Let $t \in \R$, and $t \geq 0$.
\end{itemize} 
Then, 
\begin{align*}
    \|\frac{\d \log p_t(x)}{\d x}\|_2 \leq t + \| x \|_2
\end{align*}
\end{lemma}

\begin{proof}
We can show
\begin{align*}
    \|\frac{\d \log p_t(x)}{\d x}\|_2 = & ~ \|  t\mathbf{1}_d - x \|_2\\
    \leq & ~ \|t\mathbf{1}_d\|_2 + \| -x \|_2\\
    \leq & ~ t + \| x \|_2
\end{align*}
where the first step follows from Lemma~\ref{lem:gradient_log_p_t_continuous}, the second step follows from Fact~\ref{fac:norm_bounds}, and the last step follows from simple algebra.
\end{proof}

Below we calculate the Lipschitz constant for the score function of $\mathsf{pdf}$ for continuous case when the mean is a function of time.
\begin{lemma}[Lipschitz]
If the following conditions hold
\begin{itemize}
    \item Let $x, \wt{x} \in \R^d$.
    \item Let $t \in \R$, and $t \geq 0$.
\end{itemize} 
Then, 
\begin{align*}
    \|\frac{\d \log p_t(x)}{\d x} - \frac{\d \log p_t(\wt{x})}{\d \wt{x}}\|_2 = \| \wt{x} - x \|_2
\end{align*}
\end{lemma}

\begin{proof}
We can show
\begin{align*}
    \|\frac{\d \log p_t(x)}{\d x} - \frac{\d \log p_t(\wt{x})}{\d \wt{x}}\|_2 = & ~ \|  t\mathbf{1}_d - x - (t\mathbf{1}_d - \wt{x}) \|_2\\
    = & ~ \| \wt{x} - x \|_2
\end{align*}
where the first step follows from Lemma~\ref{lem:gradient_log_p_t_continuous}, and the last step follows from simple algebra.
\end{proof}

\subsection{Case when the covariance of \texorpdfstring{$p_t(x)$}{} is a function of time}\label{sec:single_gaussian:cov_function_of_time}

Below we define the $\mathsf{pdf}$ for continuous case when the covariance is a function of time.
\begin{definition}\label{def:p_t_continuous_t_sigma}
If the following conditions hold
\begin{itemize}
    \item Let $x \in \R^d$.
    \item Let $t \in \R$, and $t \geq 0$.
\end{itemize}

We define
\begin{align*}
    p_t(x) := \frac{1}{t^{1/2}(2\pi)^{d/2}} \exp(-\frac{1}{2t} \|x - \mathbf{1}_d\|_2^2)
\end{align*}

Further, we have
\begin{align*}
    \log p_t(x) =  - \frac{1}{2} \log t - \frac{d}{2}\log (2\pi)-\frac{1}{2t} \|x  - \mathbf{1}_d\|_2^2
\end{align*}
\end{definition}

Below we calculate the score function of $\mathsf{pdf}$ for continuous case when the covariance is a function of time.
\begin{lemma}\label{lem:gradient_log_p_t_continuous_t_sigma}
If the following conditions hold
\begin{itemize}
    \item Let $x \in \R^d$.
    \item Let $t \in \R$, and $t \geq 0$.
\end{itemize} 
Then, 
\begin{align*}
    \frac{\d \log p_t(x)}{\d x} = \frac{1}{t} (\mathbf{1}_d - x)
\end{align*}
\end{lemma}

\begin{proof}
We can show 
\begin{align*}
    \frac{\d \log p_t(x)}{\d x} = & ~ \frac{\d}{\d x} (  - \frac{1}{2} \log t - \frac{d}{2}\log (2\pi)-\frac{1}{2t} \|x  - \mathbf{1}_d\|_2^2)\\
    = & ~ -\frac{1}{2t} \cdot \frac{\d}{\d x} \|x - \mathbf{1}_d\|_2^2\\
    = & ~ - \frac{1}{2t} \cdot 2 (x - \mathbf{1}_d)\\
    = & ~ \frac{1}{t} (\mathbf{1}_d - x)
\end{align*}
where the first step follows from Definition~\ref{def:p_t_continuous_t_sigma}, the second step follows from variables are independent, the third step follows from Fact~\ref{fac:calculus}, and the last step follows from simple algebra.
\end{proof}

Below we calculate the upper bound of the score function of $\mathsf{pdf}$ for the continuous case when the covariance is a function of time.

\begin{lemma}[Linear growth]
If the following conditions hold
\begin{itemize}
    \item Let $x \in \R^d$.
    \item Let $t \in \R$, and $t \geq 0$.
\end{itemize} 
Then, 
\begin{align*}
    \|\frac{\d \log p_t(x)}{\d x}\|_2 \leq \frac{1}{t} (1+\|x\|_2)
\end{align*}
\end{lemma}

\begin{proof}
We can show
\begin{align*}
    \|\frac{\d \log p_t(x)}{\d x}\|_2 = & ~ \|  \frac{1}{t} (\mathbf{1}_d - x) \|_2\\
    = & ~ |\frac{1}{t}| \cdot \|\mathbf{1}_d - x\|_2\\
    = & ~ \frac{1}{t} \|\mathbf{1}_d - x\|_2\\
    \leq & ~ \frac{1}{t} (\|\mathbf{1}_d\|_2 + \| -x \|_2)\\
    = & ~ \frac{1}{t} (1+\|x\|_2)
\end{align*}
where the first step follows from Lemma~\ref{lem:gradient_log_p_t_continuous_t_sigma}, the second step follows from Fact~\ref{fac:norm_bounds}, the third step follows from $t\geq 0$, the fourth step follows from Fact~\ref{fac:norm_bounds}, and the last step follows from simple algebra.
\end{proof}

Below we calculate the Lipschitz constant of the score function of $\mathsf{pdf}$ for continuous case when the covariance is a function of time.
\begin{lemma}[Lipschitz]
If the following conditions hold
\begin{itemize}
    \item Let $x, \wt{x} \in \R^d$.
    \item Let $t \in \R$, and $t \geq 0$.
\end{itemize} 
Then, 
\begin{align*}
    \|\frac{\d \log p_t(x)}{\d x} - \frac{\d \log p_t(\wt{x})}{\d \wt{x}}\|_2 = \frac{1}{t} \|x-\wt{x}\|_2
\end{align*}
\end{lemma}

\begin{proof}
We can show
\begin{align*}
    \|\frac{\d \log p_t(x)}{\d x} - \frac{\d \log p_t(\wt{x})}{\d \wt{x}}\|_2 = & ~ \|  \frac{1}{t} (\mathbf{1}_d - x) - \frac{1}{t} (\mathbf{1}_d - \wt{x}) \|_2\\
    = & ~ \|\frac{1}{t} (\wt{x} - x)\|_2\\
    = & ~ \frac{1}{t} \|x-\wt{x}\|_2
\end{align*}
where the first step follows from Lemma~\ref{lem:gradient_log_p_t_continuous_t_sigma}, the second step follows from simple algebra, the third step follows from Fact~\ref{fac:norm_bounds}.
\end{proof}

\subsection{A general version for single Gaussian}\label{sec:single_gaussian:mean_cov_function_of_time}

Now we combine the previous results by calculate a slightly more complex case.
Consider $p_t $ such that
\begin{align*}
    p_t(x) = \Pr_{x' \sim \mathcal{N}( \mu(t), \Sigma(t)) }[ x' = x]
\end{align*}
where $\mu(t) \in \R^d$, $\Sigma(t) \in \R^{d \times d}$ and they are derivative to $t$ and $\Sigma(t)$ is a symmetric p.s.d. matrix whose the smallest singular value is always larger than a fixed $\sigma_{\min} > 0$.

\begin{definition}\label{def:p_t_continuous_general}
If the following conditions hold
\begin{itemize}
    \item Let $x \in \R^d$.
    \item Let $t \in \R$, and $t \geq 0$.
\end{itemize}

We define
\begin{align*}
    p_t(x) := \frac{1}{(2\pi)^{d/2}\det(\Sigma(t))^{1/2}}  \exp(-\frac{1}{2} (x - \mu(t))^\top \Sigma(t)^{-1} (x - \mu(t))).
\end{align*}

Further, we have
\begin{align*}
    \log p_t(x) = - \frac{d}{2}\log (2\pi) - \frac{1}{2}\log \det (\Sigma(t)) - \frac{1}{2} (x - \mu(t))^\top \Sigma(t)^{-1} (x - \mu(t))
\end{align*}
\end{definition}

Below we calculate the score function of $\mathsf{pdf}$ for continuous case when both the mean and covariance are a function of time.

\begin{lemma}\label{lem:gradient_log_p_t_continuous_general}
If the following conditions hold
\begin{itemize}
    \item Let $x \in \R^d$.
    \item Let $t \in \R$, and $t \geq 0$.
\end{itemize} 
Then, 
\begin{align*}
    \frac{\d \log p_t(x)}{\d x} = - \Sigma(t)^{-1} (x - \mu(t))
\end{align*}
\end{lemma}

\begin{proof}
We can show
\begin{align*}
    \frac{\d \log p_t(x)}{\d x} = & ~ \frac{\d}{\d x} ( - \frac{d}{2}\log (2\pi) - \frac{1}{2}\log \det (\Sigma(t)) - \frac{1}{2} (x - \mu(t))^\top \Sigma(t)^{-1} (x - \mu(t)))\\
    = & ~ - \frac{1}{2} \cdot \frac{\d}{\d x} (x - \mu(t))^\top \Sigma(t)^{-1} (x - \mu(t))\\
    = & ~ - \frac{1}{2} \cdot 2 \Sigma(t)^{-1} (x - \mu(t))\\
    = & ~ - \Sigma(t)^{-1} (x - \mu(t))
\end{align*}
where the first step follows from Definition~\ref{def:p_t_continuous_general}, the second step follows from Fact~\ref{fac:calculus}, the third step follows from Fact~\ref{fac:matrix_calculus}, and the last step follows from simple algebra.
\end{proof}

Below we calculate the upper bound of the score function of $\mathsf{pdf}$ for continuous case when both the mean and covariance is a function of time.

\begin{lemma}[Linear growth]
If the following conditions hold
\begin{itemize}
    \item Let $x \in \R^d$.
    \item Let $t \in \R$, and $t \geq 0$.
\end{itemize} 
Then, 
\begin{align*}
    \|\frac{\d \log p_t(x)}{\d x}\|_2 \leq   \frac{1}{\sigma_{\min}(\Sigma(t))}  \cdot (\|\mu(t)\|_2+\|x\|_2)
\end{align*}
\end{lemma}

\begin{proof}
We can show
\begin{align*}
    \|\frac{\d \log p_t(x)}{\d x}\|_2 = & ~ \| - \Sigma(t)^{-1} (x - \mu(t)) \|_2\\
    \leq & ~ \| -\Sigma(t)^{-1} \| \cdot \|x-\mu(t)\|_2\\
    = & ~ \| \Sigma(t)^{-1} \|  \cdot \|x-\mu(t)\|_2\\
    = & ~ \frac{1}{\sigma_{\min}(\Sigma(t))}  \cdot \|x-\mu(t)\|_2\\
    \leq & ~ \frac{1}{\sigma_{\min}(\Sigma(t))}  \cdot (\|x\|_2 + \|-\mu(t)\|_2)\\
    = & ~ \frac{1}{\sigma_{\min}(\Sigma(t))}  \cdot (\|\mu(t)\|_2+\|x\|_2)
\end{align*}
where the first step follows from Lemma~\ref{lem:gradient_log_p_t_continuous_general}, the second step follows from Fact~\ref{fac:norm_bounds}, the third step follows from Fact~\ref{fac:norm_bounds}, the fourth step follows from Fact~\ref{fac:norm_bounds}, the fifth step follows from Fact~\ref{fac:norm_bounds}, and the last step follows from Fact~\ref{fac:norm_bounds}.
\end{proof}

Below we calculate the Lipschitz constant of the score function of $\mathsf{pdf}$ for continuous case when both the mean and covariance are a function of time.

\begin{lemma}[Lipschitz]
If the following conditions hold
\begin{itemize}
    \item Let $x, \wt{x} \in \R^d$.
    \item Let $t \in \R$, and $t \geq 0$.
\end{itemize} 
Then, 
\begin{align*}
    \|\frac{\d \log p_t(x)}{\d x} - \frac{\d \log p_t(\wt{x})}{\d \wt{x}}\|_2 \leq & ~ \frac{1}{\sigma_{\min}(\Sigma(t))}  \cdot \|x-\wt{x}\|_2
\end{align*}
\end{lemma}

\begin{proof}
We can show
\begin{align*}
    \|\frac{\d \log p_t(x)}{\d x} - \frac{\d \log p_t(\wt{x})}{\d \wt{x}}\|_2 = & ~ \|- \Sigma(t)^{-1} (x - \mu(t)) - (- \Sigma(t)^{-1} (\wt{x} - \mu(t))) \|_2\\
    = & ~ \|- \Sigma(t)^{-1} (x - \wt{x})  \|_2\\
    \leq & ~ \| -\Sigma(t)^{-1} \| \cdot \|x-\wt{x}\|_2\\
    = & ~ \frac{1}{\sigma_{\min}(\Sigma(t))}  \cdot \|x-\wt{x}\|_2
\end{align*}
where the first step follows from Lemma~\ref{lem:gradient_log_p_t_continuous_general}, the second step follows from simple algebra, the third step follows from Fact~\ref{fac:norm_bounds}, and the last step follows from Fact~\ref{fac:norm_bounds}.
\end{proof}

\section{A General Version for Two Gaussian}\label{sec:general_two_gaussian}
In this section, we compute the linear growth and Lipschitz constant for a mixture of 2 Gaussian where both the mean and covariance are a function of time.
The organization of this section is as follows:
\begin{itemize}
    \item Section~\ref{sec:general_two_gaussian:definitions} defines the probability density function ($\mathsf{pdf}$) $p_t(x)$ that we use, which is a mixture of 2 Gaussian.
    \item Section~\ref{sec:general_two_gaussian:lemmas_for_score} provides lemmas that are used for calculation of the score function i.e. gradient of the log $\mathsf{pdf}$ $\frac{\d \log p_t(x)}{\d x}$.
    \item  Section~\ref{sec:general_two_gaussian:score} provides the expression of the score function.
    \item Section~\ref{sec:general_two_gaussian:lemmas_for_upper_bound_score} provides lemmas that are used for calculation of the upper bound of the score function.
    \item Section~\ref{sec:general_two_gaussian:upper_bound_score} provides the expression of the upper bound of the score function.
    \item Section~\ref{sec:general_two_gaussian:bound_for_lip_score} provides lemmas of upper bound for some base functions that are used for calculation of the Lipschitz constant of the score function.
    \item Section~\ref{sec:general_two_gaussian:lip_for_lip_score} provides lemmas of Lipschitz constant for some base functions that are used for calculation of the Lipschitz constant of the score function.
    \item Section~\ref{sec:general_two_gaussian:f_x_for_lip_score} provides lemmas of Lipschitz constant for $f(x)$ that are used for calculation of the Lipschitz constant of the score function.
    \item Section~\ref{sec:general_two_gaussian:g_x_for_lip_score} provides lemmas of Lipschitz constant for $g(x)$ that are used for calculation of the Lipschitz constant of the score function.
    \item Section~\ref{sec:general_two_gaussian:lip_score} provides the expression of the Lipschitz constant of the score function.
\end{itemize}

First, we define the following.
Let $\alpha(t) \in (0,1)$ and also is a function of time $t$. Consider $p_t $ such that
\begin{align*}
    p_t(x) = \Pr_{x' \sim \alpha(t)\mathcal{N}( \mu_1(t), \Sigma_1(t)) + (1-\alpha(t))) \mathcal{N}( \mu_2(t), \Sigma_2(t)) }[ x' = x]
\end{align*}
where $\mu_1(t), \mu_2(t) \in \R^d$, $\Sigma_1(t), \Sigma_2(t) \in \R^{d \times d}$ and they are derivative to $t$ and $\Sigma_1(t), \Sigma_2(t)$ is a symmetric p.s.d. matrix whose the smallest singular value is always larger than a fixed value $\sigma_{\min} > 0$.

For further simplicity of calculation, we denote $\alpha(t)$ to be $\alpha$.

\subsection{Definitions}\label{sec:general_two_gaussian:definitions}

Below we define function $N_1$ and $N_2$.
\begin{definition}\label{def:N_1_N_2}
If the following conditions hold
\begin{itemize}
    \item Let $x \in \R^d$.
    \item Let $t \in \R$, and $t \geq 0$.
\end{itemize}

We define 
\begin{align*}
    N_1(x) := \frac{1}{(2\pi)^{d/2}\det(\Sigma_1(t))^{1/2}}  \exp(-\frac{1}{2} (x - \mu_1(t))^\top \Sigma_1(t)^{-1} (x - \mu_1(t)))
\end{align*}
and
\begin{align*}
    N_2(x) := \frac{1}{(2\pi)^{d/2}\det(\Sigma_2(t))^{1/2}}  \exp(-\frac{1}{2} (x - \mu_2(t))^\top \Sigma_2(t)^{-1} (x - \mu_2(t)))
\end{align*}
It's clearly to see that $N_i \leq \frac{1}{(2\pi)^{d/2}\det(\Sigma_i(t))^{1/2}}$ since $N_i(x)$ takes maximum when $x=\mu_i$. 
\end{definition}

Below we define the $\mathsf{pdf}$ for 2 mixtures of Gaussians.

\begin{definition}\label{def:p_t_continuous_general_2_gaussian}
If the following conditions hold
\begin{itemize}
    \item Let $x \in \R^d$.
    \item Let $t \in \R$, and $t \geq 0$.
    \item Let $\alpha \in \R$ and $\alpha \in (0,1)$.
    \item Let $N_1(x), N_2(x)$ be defined as Definition~\ref{def:N_1_N_2}.
\end{itemize}

We define
\begin{align*}
    p_t(x) := &\frac{\alpha}{(2\pi)^{d/2}\det(\Sigma_1(t))^{1/2}}  \exp(-\frac{1}{2} (x - \mu_1(t))^\top \Sigma_1(t)^{-1} (x - \mu_1(t))) + \\
    &\frac{1-\alpha}{(2\pi)^{d/2}\det(\Sigma_2(t))^{1/2}}  \exp(-\frac{1}{2} (x - \mu_2(t))^\top \Sigma_2(t)^{-1} (x - \mu_2(t))).
\end{align*}

This can be further rewritten as follows:
\begin{align*}
    p_t(x) = \alpha N_1(x) + (1-\alpha) N_2(x)
\end{align*}

Further, we have 
\begin{align*}
    \log p_t(x) = \log (\alpha N_1(x) + (1-\alpha) N_2(x))
\end{align*}
\end{definition}

\subsection{Lemmas for calculation of the score function}\label{sec:general_two_gaussian:lemmas_for_score}

This subsection describes lemmas that are used for further calculation of the score function.

This lemma calculates the gradient of function $N_i$.
\begin{lemma}\label{lem:gradient_N_1_N_2}
If the following conditions hold
\begin{itemize}
    \item Let $x \in \R^d$.
    \item Let $t \in \R$, and $t \geq 0$.
    \item Let $\alpha \in \R$ and $\alpha \in (0,1)$.
    \item Let $N_1(x), N_2(x)$ be defined as Definition~\ref{def:N_1_N_2}.
\end{itemize}

Then, for $i \in \{1,2\}$, we have
\begin{align*}
    \frac{\d N_i(x)}{\d x} = N_i(x) (-\Sigma_i(t)^{-1} (x - \mu_i(t)))
\end{align*}
\end{lemma}

\begin{proof}
We can show
\begin{align*}
    \frac{\d N_i(x)}{\d x} = & ~ \frac{\d}{\d x} (\frac{1}{(2\pi)^{d/2}\det(\Sigma_i(t))^{1/2}}  \exp(-\frac{1}{2} (x - \mu_i(t))^\top \Sigma_i(t)^{-1} (x - \mu_i(t))))\\
    = & ~ N_i(x) \cdot \frac{\d}{\d x} (-\frac{1}{2} (x - \mu_i(t))^\top \Sigma_i(t)^{-1} (x - \mu_i(t)))\\
    = & ~ N_i(x) (-\frac{1}{2} \cdot 2 \Sigma_i(t)^{-1} (x - \mu_i(t)))\\
    = & ~ N_i(x) (-\Sigma_i(t)^{-1} (x - \mu_i(t)))
\end{align*}
where the first step follows from Definition~\ref{def:N_1_N_2}, the second step follows from Fact~\ref{fac:calculus}, the third step follows from Fact~\ref{fac:matrix_calculus}, and the last step follows from simple algebra.
\end{proof}

This lemma calculates the gradient of function $p_t(x)$.
\begin{lemma}\label{lem:gradient_p_t_continuous_general_2_gaussian}
If the following conditions hold
\begin{itemize}
    \item Let $x \in \R^d$.
    \item Let $t \in \R$, and $t \geq 0$.
    \item Let $\alpha \in \R$ and $\alpha \in (0,1)$.
    \item Let $p_t(x)$ be defined as Definition~\ref{def:p_t_continuous_general_2_gaussian}.
    \item Let $N_1(x), N_2(x)$ be defined as Definition~\ref{def:N_1_N_2}.
\end{itemize}

Then, 
\begin{align*}
    \frac{\d p_t(x)}{\d x} = \alpha N_1(x) (-\Sigma_1(t)^{-1} (x - \mu_1(t))) + (1-\alpha)  N_2(x) (-\Sigma_2(t)^{-1} (x - \mu_2(t)))
\end{align*}
\end{lemma}

\begin{proof}
We can show
\begin{align*}
    \frac{\d p_t(x)}{\d x} = & ~ \frac{\d}{\d x} (\alpha N_1(x) + (1-\alpha) N_2(x))\\
    = & ~ \alpha \frac{\d}{\d x}  N_1(x) + (1-\alpha) \frac{\d}{\d x} N_2(x)\\
    = & ~ \alpha N_1(x) (-\Sigma_1(t)^{-1} (x - \mu_1(t))) + (1-\alpha)  N_2(x) (-\Sigma_2(t)^{-1} (x - \mu_2(t)))
\end{align*}
where the first step follows from Definition~\ref{def:p_t_continuous_general_2_gaussian}, the second step follows from Fact~\ref{fac:calculus}, and the last step follows from Lemma~\ref{lem:gradient_N_1_N_2}.
\end{proof}

\subsection{Calculation of the score function}\label{sec:general_two_gaussian:score}

Below we define $f(x)$ and $g(x)$ that simplify further calculation.
\begin{definition}\label{def:f_x_g_x}
For further simplicity, we define the following functions:

If the following conditions hold
\begin{itemize}
    \item Let $x \in \R^d$.
    \item Let $t \in \R$, and $t \geq 0$.
    \item Let $\alpha \in \R$ and $\alpha \in (0,1)$.
    \item Let $N_1(x), N_2(x)$ be defined as Definition~\ref{def:N_1_N_2}.
\end{itemize}

We define 
\begin{align*}
    f(x) := \frac{\alpha N_1(x)}{\alpha N_1(x) + (1-\alpha) N_2(x)}
\end{align*}
and 
\begin{align*}
    g(x) := \frac{(1-\alpha) N_2(x)}{\alpha N_1(x) + (1-\alpha) N_2(x)}
\end{align*}

And it's clearly to see that $0\leq f(x)\leq 1$, $0\leq g(x)\leq 1$ and $f(x) + g(x) = 1$.
\end{definition}

This lemma calculates the score function.
\begin{lemma}\label{lem:gradient_log_p_t_continuous_general_2_gaussian}
If the following conditions hold
\begin{itemize}
    \item Let $x \in \R^d$.
    \item Let $t \in \R$, and $t \geq 0$.
    \item Let $\alpha \in \R$ and $\alpha \in (0,1)$.
    \item Let $p_t(x)$ be defined as Definition~\ref{def:p_t_continuous_general_2_gaussian}.
    \item Let $N_1(x), N_2(x)$ be defined as Definition~\ref{def:N_1_N_2}.
    \item Let $f(x), g(x)$ be defined as Definition~\ref{def:f_x_g_x}.
\end{itemize}

Then, 
\begin{align*}
    \frac{\d \log p_t(x)}{\d x} = \frac{\alpha N_1(x) (-\Sigma_1(t)^{-1} (x - \mu_1(t)))}{\alpha N_1(x) + (1-\alpha) N_2(x)} + \frac{(1-\alpha)  N_2(x) (-\Sigma_2(t)^{-1} (x - \mu_2(t)))}{\alpha N_1(x) + (1-\alpha) N_2(x)}
\end{align*}
\end{lemma}

\begin{proof}
We can show
\begin{align*}
    \frac{\d \log p_t(x)}{\d x} = & ~ \frac{\d \log p_t(x)}{\d p_t(x)} \frac{\d p_t(x)}{\d x}\\
    = & ~ \frac{1}{p_t(x)}\frac{\d p_t(x)}{\d x}\\
    = & ~ \frac{1}{p_t(x)} (\alpha N_1(x) (-\Sigma_1(t)^{-1} (x - \mu_1(t))) + (1-\alpha)  N_2(x) (-\Sigma_2(t)^{-1} (x - \mu_2(t))))\\
    = & ~ \frac{\alpha N_1(x) (-\Sigma_1(t)^{-1} (x - \mu_1(t)))}{\alpha N_1(x) + (1-\alpha) N_2(x)} + \frac{(1-\alpha)  N_2(x) (-\Sigma_2(t)^{-1} (x - \mu_2(t)))}{\alpha N_1(x) + (1-\alpha) N_2(x)}\\
    = & ~ f(x) (-\Sigma_1(t)^{-1} (x - \mu_1(t)) + g(x)(-\Sigma_2(t)^{-1} (x - \mu_2(t))
\end{align*}
where the first step follows from Fact~\ref{fac:calculus}, the second step follows from Fact~\ref{fac:calculus}, the third step follows from Lemma~\ref{lem:gradient_p_t_continuous_general_2_gaussian}, the fourth step follows from Definition~\ref{def:p_t_continuous_general_2_gaussian} and the last step follows from Definition~\ref{def:f_x_g_x}.

\end{proof}

\subsection{Lemmas for the calculation of the upper bound of the score function}\label{sec:general_two_gaussian:lemmas_for_upper_bound_score}
This section provides lemmas that are used in calculation of upper bound of the score function.

This lemma calculates the upper bound of function $\|- \Sigma_i(t)^{-1} (x - \mu_i(t)) \|_2$.
\begin{lemma}\label{lem:bound_x_minus_mu}
If the following conditions hold
\begin{itemize}
    \item Let $x \in \R^d$.
    \item Let $t \in \R$, and $t \geq 0$.
\end{itemize} 

Then, for each $i \in \{1,2\}$, we have
\begin{align*}
    \|- \Sigma_i(t)^{-1} (x - \mu_i(t)) \|_2  \leq   \frac{1}{\sigma_{\min}(\Sigma_i(t))}  \cdot (\|x\|_2 + \|\mu_i(t)\|_2)
\end{align*}
\end{lemma}

\begin{proof}
We can show
\begin{align*}
    \| - \Sigma_i(t)^{-1} (x - \mu_i(t)) \|_2 \leq & ~ \| -\Sigma_i(t)^{-1} \| \cdot \|x-\mu_i(t)\|_2\\
    = & ~ \| \Sigma_i(t)^{-1} \| \cdot \|x-\mu_i(t)\|_2\\
    = & ~ \frac{1}{\sigma_{\min}(\Sigma_i(t))}  \cdot \|x-\mu(t)\|_2\\
    \leq & ~ \frac{1}{\sigma_{\min}(\Sigma_i(t))}  \cdot (\|x\|_2 + \|-\mu_i(t)\|_2)\\
    = & ~ \frac{1}{\sigma_{\min}(\Sigma_i(t))}  \cdot (\|x\|_2 + \|\mu_i(t)\|_2)
\end{align*}
where the first step follows from Fact~\ref{fac:norm_bounds}, the second step follows from Fact~\ref{fac:norm_bounds}, the third step follows from Fact~\ref{fac:norm_bounds}, the fourth step follows from Fact~\ref{fac:norm_bounds}, and the last step follows from simple algebra.
\end{proof}

\subsection{Upper bound of the score function}\label{sec:general_two_gaussian:upper_bound_score}

This lemma calculates the upper bound of the score function.
\begin{lemma}[Linear growth]
If the following conditions hold
\begin{itemize}
    \item Let $x \in \R^d$.
    \item Let $t \in \R$, and $t \geq 0$.
    \item Let $\alpha \in \R$ and $\alpha \in (0,1)$.
    \item Let $p_t(x)$ be defined as Definition~\ref{def:p_t_continuous_general_2_gaussian}.
    \item Let $f(x), g(x)$ be defined as Definition~\ref{def:f_x_g_x}.
    \item Let $\sigma_{\min} := \min \{ \sigma_{\min}( \Sigma_1(t) ) , \sigma_{\min} ( \Sigma_2(t) ) \}$. 
    \item Let $\mu_{\max}: = \max\{ \| \mu_1(t) \|_2, \| \mu_2(t) \|_2 , 1 \}$. 
\end{itemize} 

Then, 
\begin{align*}
    \|\frac{\d \log p_t(x)}{\d x}\|_2 \leq & ~ \sigma_{\min}^{-1} \cdot \mu_{\max} \cdot (1+ \| x \|_2)
\end{align*}
\end{lemma}

\begin{proof}
We can show
\begin{align*}
    \|\frac{\d \log p_t(x)}{\d x}\|_2 = & ~ \|f(x) (-\Sigma_1(t)^{-1} (x - \mu_1(t)) + g(x)(-\Sigma_2(t)^{-1} (x - \mu_2(t))\|_2\\
    \leq & ~ \|f(x) (-\Sigma_1(t)^{-1} (x - \mu_1(t))\|_2 + \|g(x)(-\Sigma_2(t)^{-1} (x - \mu_2(t))_2\|_2 \\
    \leq & ~ \max_{i \in [2]} \| -\Sigma_i(t)^{-1} (x - \mu_i(t)) \|_2  \\
    \leq & ~ \max_{i \in [2]} (\frac{1}{\sigma_{\min}(\Sigma_i(t))}  \cdot (\|x\|_2 + \|\mu_i(t)\|_2))\\
    \leq & ~ \sigma_{\min}^{-1} ( \mu_{\max} + \| x \|_2) \\
    \leq & ~ \sigma_{\min}^{-1} \cdot \mu_{\max} \cdot  (1 + \| x \|_2 )
\end{align*}
where the first step follows from Lemma~\ref{lem:gradient_log_p_t_continuous_general_2_gaussian}, the second step follows from Fact~\ref{fac:norm_bounds}, the third step follows from $f(x) + g(x) = 1$ and $f(x), g(x) \geq 0$, the fourth step follows from Lemma~\ref{lem:bound_x_minus_mu}, the fifth step follows from definition of $\mu_{\max}$ and $\sigma_{\min}$, and the last step follows from $\mu_{\max} \geq 1$. 
\end{proof}

\subsection{Lemmas for Lipschitz calculation: upper bound of base functions}\label{sec:general_two_gaussian:bound_for_lip_score}
This section provides the lemmas of bounds of base functions that are used in calculation of Lipschitz of the score function.

This lemma calculate the upper bound of the function $\|- \Sigma_i(t)^{-1} (x - \mu_i(t))\|_2$.
\begin{lemma}\label{lem:upper_bound_x_minus_mu}
If the following conditions hold
\begin{itemize}
    \item Let $x \in \R^d$.
    \item Let $t \in \R$, and $t \geq 0$.
    \item Let $\|x-\mu_i(t)\|_2 \leq R$,  where $R \geq 1$, for each $i \in \{1,2\}$.
\end{itemize}

Then, for each $i \in \{1,2\}$, we have
\begin{align*}
    \|- \Sigma_i(t)^{-1} (x - \mu_i(t))\|_2 \leq \frac{R}{\sigma_{\min}(\Sigma_i(t))}
\end{align*}
\end{lemma}

\begin{proof}
We can show
\begin{align*}
    \|- \Sigma_i(t)^{-1} (x - \mu_i(t))\|_2 \leq & ~ \|- \Sigma_i(t)^{-1} \| \cdot \| x - \mu_i(t)\|_2\\
    = & ~ \frac{1}{\sigma_{\min}(\Sigma_i(t))}  \cdot \| x - \mu_i(t)\|_2\\
    \leq & ~ \frac{R}{\sigma_{\min}(\Sigma_i(t))}
\end{align*}
where the first step follows from Fact~\ref{fac:norm_bounds}, the second step follows from Fact~\ref{fac:norm_bounds}, and the last step follows from $\|x-\mu_i(t)\|_2 \leq R$.
\end{proof}

This lemma calculate the lower bound of the function $(x - \mu_i(t))^\top \Sigma_i(t)^{-1} (x - \mu_i(t))$.
\begin{lemma}\label{lem:lower_bound_x_minus_mu}
If the following conditions hold
\begin{itemize}
    \item Let $x \in \R^d$.
    \item Let $t \in \R$, and $t \geq 0$.
    \item Let $\|x-\mu_i(t)\|_2 \leq R$, where $R \geq 1$, for each $i \in \{1,2\}$.
    \item Let $\|x-\mu_i(t)\|_2 \geq \beta$, where $\beta \in (0, 0.1)$, for each $i \in \{1,2\}$.
\end{itemize}

Then,
\begin{align*}
    (x - \mu_i(t))^\top \Sigma_i(t)^{-1} (x - \mu_i(t)) \geq \frac{\beta^2}{\sigma_{\max}(\Sigma_i(t))}
\end{align*}
\end{lemma}

\begin{proof}
We can show
\begin{align*}
    \mathrm{LHS} \geq & ~  \|x - \mu_i(t)\|_2^2 \cdot \sigma_{\min}(\Sigma_i(t)^{-1})\\
    = & ~ \|x - \mu_i(t)\|_2^2 \cdot \frac{1}{\sigma_{\max}(\Sigma_i(t))}\\
    \geq & ~ \frac{\beta^2}{\sigma_{\max}(\Sigma_i(t))}
\end{align*}
where the first step follows from Fact~\ref{fac:norm_bounds}, the second step follows from Fact~\ref{fac:norm_bounds}, and the last step follows from $\|x-\mu_i(t)\|_2 \geq \beta$.
\end{proof}

This lemma calculate the upper bound of the function $\exp(-\frac{1}{2} (x - \mu_i(t))^\top \Sigma_i(t)^{-1} (x - \mu_i(t)))$.
\begin{lemma}\label{lem:upper_bound_exp}
If the following conditions hold
\begin{itemize}
    \item Let $x \in \R^d$.
    \item Let $t \in \R$, and $t \geq 0$.
    \item Let $\|x-\mu_i(t)\|_2 \leq R$, where $R \geq 1$, for each $i \in \{1,2\}$.
    \item Let $\|x-\mu_i(t)\|_2 \geq \beta$, where $\beta \in (0, 0.1)$, for each $i \in \{1,2\}$.
\end{itemize}

Then,
\begin{align*}
    \exp(-\frac{1}{2} (x - \mu_i(t))^\top \Sigma_i(t)^{-1} (x - \mu_i(t))) \leq \exp(-\frac{\beta^2}{2\sigma_{\max}(\Sigma_i(t))})
\end{align*}
\end{lemma}

\begin{proof}
We can show
\begin{align*}
    \mathrm{LHS} = & ~ \exp(-\frac{1}{2} (x - \mu_i(t))^\top \Sigma_i(t)^{-1} (x - \mu_i(t)))\\
    \leq & ~ \exp(-\frac{\beta^2}{2\sigma_{\max}(\Sigma_i(t))})
\end{align*}
where the first step follows from Fact~\ref{fac:norm_bounds}, the second step follows from Lemma~\ref{lem:lower_bound_x_minus_mu}.
\end{proof}

\subsection{Lemmas for Lipschitz calculation: Lipschitz constant of base functions}\label{sec:general_two_gaussian:lip_for_lip_score}
This section provides the lemmas of Lipschitz constant of base functions that are used in calculation of Lipschitz of the score function.

This lemma calculates Lipschitz constant of function $\|- \Sigma_i(t)^{-1} (x - \mu_i(t)) - (- \Sigma_i(t)^{-1} (\wt{x} - \mu_i(t))) \|_2$.
\begin{lemma}\label{lem:lip_x_minus_mu}
If the following conditions hold
\begin{itemize}
    \item Let $x,\wt{x} \in \R^d$.
    \item Let $t \in \R$, and $t \geq 0$.
\end{itemize}

Then, for $i \in \{1,2\}$, we have
\begin{align*}
    \|- \Sigma_i(t)^{-1} (x - \mu_i(t)) - (- \Sigma_i(t)^{-1} (\wt{x} - \mu_i(t))) \|_2 \leq \frac{1}{\sigma_{\min}(\Sigma_i(t))}  \cdot \|x-\wt{x}\|_2
\end{align*}
\end{lemma}

\begin{proof}
    We can show
\begin{align*}
    \mathrm{LHS} = & ~ \|- \Sigma_i(t)^{-1} (x - \wt{x})  \|_2\\
    \leq & ~ \| -\Sigma_i(t)^{-1} \| \cdot \|x-\wt{x}\|_2\\
    = & ~ \frac{1}{\sigma_{\min}(\Sigma_i(t))}  \cdot \|x-\wt{x}\|_2
\end{align*}
where the first step follows from simple algebra, the second step follows from Fact~\ref{fac:norm_bounds}, and the last step follows from Fact~\ref{fac:norm_bounds}.
\end{proof}

This lemma calculates Lipschitz constant of function $|-\frac{1}{2} (x - \mu_i(t))^\top \Sigma_i(t)^{-1} (x - \mu_i(t)) - (-\frac{1}{2} (\wt{x} - \mu_i(t))^\top \Sigma_i(t)^{-1} (\wt{x} - \mu_i(t)))|$.

\begin{lemma}\label{lem:lip_sigma_x_minus_mu}
If the following conditions hold
\begin{itemize}
    \item Let $x,\wt{x} \in \R^d$.
    \item Let $t \in \R$, and $t \geq 0$.
    \item Let $\|x-\mu_i(t)\|_2 \leq R$, $\|\wt{x} - \mu_i(t)\|_2 \leq R$, where $R \geq 1$, for each $i \in \{1,2\}$.
\end{itemize}

Then, for each $i \in \{1,2\}$, we have 
\begin{align*}
    |-\frac{1}{2} (x - \mu_i(t))^\top \Sigma_i(t)^{-1} (x - \mu_i(t)) - (-\frac{1}{2} (\wt{x} - \mu_i(t))^\top \Sigma_i(t)^{-1} (\wt{x} - \mu_i(t)))| \leq \frac{R}{\sigma_{\min}(\Sigma_i(t))}  \cdot \|x-\wt{x}\|_2
\end{align*}
\end{lemma}

\begin{proof}
We can show
\begin{align*}
    \mathrm{LHS} \leq & ~ |-\frac{1}{2} (x - \mu_i(t))^\top \Sigma_i(t)^{-1} (x - \mu_i(t)) - (-\frac{1}{2} (x - \mu_i(t))^\top \Sigma_i(t)^{-1} (\wt{x} - \mu_i(t)))| \\
    & ~ + |-\frac{1}{2} (x - \mu_i(t))^\top \Sigma_i(t)^{-1} (\wt{x} - \mu_i(t)) - (-\frac{1}{2} (\wt{x} - \mu_i(t))^\top \Sigma_i(t)^{-1} (\wt{x} - \mu_i(t)))|\\
    \leq & ~ |-\frac{1}{2} (x - \mu_i(t))^\top \Sigma_i(t)^{-1} (x - \wt{x})| + | -\frac{1}{2} (x - \wt{x})^\top \Sigma_i(t)^{-1} (\wt{x} - \mu_i(t))|\\
    \leq & ~ \frac{1}{2} \cdot \|\Sigma_i(t)^{-1}(x - \mu_i(t))\|_2 \cdot \|x - \wt{x}\|_2 + \frac{1}{2} \cdot \|\Sigma_i(t)^{-1} (x - \wt{x}) \|_2 \cdot \| \wt{x} - \mu_i(t)\|_2\\
    \leq & ~ \frac{1}{2} \cdot \frac{1}{\sigma_{\min}(\Sigma_i(t))} \cdot R \cdot \|x - \wt{x}\|_2 + \frac{1}{2} \cdot \|\Sigma_i(t)^{-1} (x - \wt{x}) \|_2 \cdot R\\
    \leq & ~ \frac{1}{2} \cdot \frac{1}{\sigma_{\min}(\Sigma_i(t))} \cdot R \cdot \|x - \wt{x}\|_2 + \frac{1}{2} \cdot \frac{1}{\sigma_{\min}(\Sigma_i(t))}  \cdot \|x-\wt{x}\|_2 \cdot R\\
    = & ~ \frac{R}{\sigma_{\min}(\Sigma_i(t))}  \cdot \|x-\wt{x}\|_2
\end{align*}
where the first step follows from Fact~\ref{fac:norm_bounds}, the second step follows from simple algebra, the third step follows from Fact~\ref{fac:norm_bounds}, the fourth step follows from $\|x-\mu_i(t)\|_2 \leq R$, $\|\wt{x} - \mu_i(t)\|_2 \leq R$, the fifth step follows from Lemma~\ref{lem:lip_x_minus_mu}, and the last step follows from simple algebra.
\end{proof}

This lemma calculates Lipschitz constant of function $|N_i(x) - N_i(\wt{x})|$.
\begin{lemma}\label{lem:lip_N}
If the following conditions hold
\begin{itemize}
    \item Let $x,\wt{x} \in \R^d$.
    \item Let $t \in \R$, and $t \geq 0$.
    \item Let $N_1(x), N_2(x)$ be defined as Definition~\ref{def:N_1_N_2}.
    \item Let $\|x-\mu_i(t)\|_2 \leq R$, where $R \geq 1$, for each $i \in \{1,2\}$.
    \item Let $\|x-\mu_i(t)\|_2 \geq \beta$, where $\beta \in (0, 0.1)$, for each $i \in \{1,2\}$.
\end{itemize}

Then, for each $i \in \{1,2\}$, we have
\begin{align*}
    |N_i(x) - N_i(\wt{x})| \leq \frac{1}{(2\pi)^{d/2}\det(\Sigma_i(t))^{1/2}} \cdot  \exp(-\frac{\beta^2}{2\sigma_{\max}(\Sigma_i(t))}) \cdot \frac{R}{\sigma_{\min}(\Sigma_i(t))} \cdot \|x-\wt{x}\|_2
\end{align*}
\end{lemma}

\begin{proof}
We can show
\begin{align*}
    |N_i(x) - N_i(\wt{x})| = & ~ |\frac{1}{(2\pi)^{d/2}\det(\Sigma_i(t))^{1/2}}  \exp(-\frac{1}{2} (x - \mu_i(t))^\top \Sigma_i(t)^{-1} (x - \mu_i(t))) \\
    & ~ - \frac{1}{(2\pi)^{d/2}\det(\Sigma_i(t))^{1/2}}  \exp(-\frac{1}{2} (\wt{x} - \mu_i(t))^\top \Sigma_i(t)^{-1} (\wt{x} - \mu_i(t)))|\\
    = & ~ \frac{1}{(2\pi)^{d/2}\det(\Sigma_i(t))^{1/2}} \cdot |\exp(-\frac{1}{2} (x - \mu_i(t))^\top \Sigma_i(t)^{-1} (x - \mu_i(t)))\\
    & ~ - \exp(-\frac{1}{2} (\wt{x} -\mu_i(t))^\top \Sigma_i(t)^{-1} (\wt{x} - \mu_i(t)))|\\
    \leq & ~ \frac{1}{(2\pi)^{d/2}\det(\Sigma_i(t))^{1/2}} \cdot \exp(-\frac{\beta^2}{2\sigma_{\max}(\Sigma_i(t))})\\
    & ~ \cdot |-\frac{1}{2} (x - \mu_i(t))^\top \Sigma_i(t)^{-1} (x - \mu_i(t)) - (-\frac{1}{2} (\wt{x} - \mu_i(t))^\top \Sigma_i(t)^{-1} (\wt{x} - \mu_i(t)))|\\
    \leq & ~ \frac{1}{(2\pi)^{d/2}\det(\Sigma_i(t))^{1/2}} \cdot \exp(-\frac{\beta^2}{2\sigma_{\max}(\Sigma_i(t))}) \cdot \frac{R}{\sigma_{\min}(\Sigma_i(t))} \cdot \|x-\wt{x}\|_2
\end{align*}
where the first step follows from Definition~\ref{def:N_1_N_2}, the second step follows from simple algebra, the third step follows from Fact~\ref{fac:upper_bound_exp}, and the last step follows from Lemma~\ref{lem:upper_bound_exp}.
\end{proof}

This lemma calculates Lipschitz constant of function $|\alpha N_1(x) + (1-\alpha)N_2(x) - (\alpha N_1(\wt{x}) + (1-\alpha)N_2(\wt{x}))|$.
\begin{lemma}\label{lem:lip_alpha_N}
If the following conditions hold
\begin{itemize}
    \item Let $x,\wt{x} \in \R^d$.
    \item Let $t \in \R$, and $t \geq 0$.
    \item Let $N_1(x), N_2(x)$ be defined as Definition~\ref{def:N_1_N_2}.
    \item Let $\alpha \in \R$ and $\alpha \in (0,1)$.
    \item Let $\|x-\mu_i(t)\|_2 \leq R$, where $R \geq 1$, for each $i \in \{1,2\}$.
    \item Let $\|x-\mu_i(t)\|_2 \geq \beta$, where $\beta \in (0, 0.1)$, for each $i \in \{1,2\}$.
    \item Let $\sigma_{\min} := \min \{ \sigma_{\min}( \Sigma_1(t) ) , \sigma_{\min} ( \Sigma_2(t) ) \}$. 
    \item Let $\sigma_{\max} := \max \{ \sigma_{\max}( \Sigma_1(t) ) , \sigma_{\max} ( \Sigma_2(t) ) \}$.
    \item Let $\det_{\min} := \min \{ \det( \Sigma_1(t) ) , \det ( \Sigma_2(t) ) \}$.
\end{itemize}

Then, we have
\begin{align*}
    |\alpha N_1(x) + (1-\alpha)N_2(x) - (\alpha N_1(\wt{x}) + (1-\alpha)N_2(\wt{x}))| \leq \frac{1}{(2\pi)^{d/2}\det_{\min}^{1/2}} \cdot \exp(-\frac{\beta^2}{2\sigma_{\max}}) \cdot \frac{R}{\sigma_{\min}} \cdot \|x-\wt{x}\|_2
\end{align*}
\end{lemma}

\begin{proof}
We can show
\begin{align*}
    \mathrm{LHS} = & ~ |\alpha N_1(x) - \alpha N_1(\wt{x}) + (1-\alpha)N_2(x) - (1-\alpha)N_2(\wt{x})|\\
    \leq & ~ \alpha|N_1(x) - N_1(\wt{x})| + (1-\alpha)|N_2(x) - N_2(\wt{x})|\\
    \leq & ~ \frac{\alpha}{(2\pi)^{d/2}\det(\Sigma_1(t))^{1/2}} \cdot \exp(-\frac{\beta^2}{2\sigma_{\max}(\Sigma_1(t))}) \cdot \frac{R}{\sigma_{\min}(\Sigma_1(t))} \cdot \|x-\wt{x}\|_2 \\
    & ~ + \frac{1-\alpha}{(2\pi)^{d/2}\det(\Sigma_2(t))^{1/2}} \cdot \exp(-\frac{\beta^2}{2\sigma_{\max}(\Sigma_2(t))}) \cdot \frac{R}{\sigma_{\min}(\Sigma_2(t))} \cdot \|x-\wt{x}\|_2\\
    \leq & ~ \frac{1}{(2\pi)^{d/2}}\max_{i \in [2]} \frac{1}{\det(\Sigma_i(t))^{1/2}} \cdot \exp(-\frac{\beta^2}{2\sigma_{\max}(\Sigma_i(t))}) \cdot \frac{R}{\sigma_{\min}(\Sigma_i(t))} \cdot \|x-\wt{x}\|_2\\
    \leq & ~ \frac{1}{(2\pi)^{d/2}\det_{\min}^{1/2}} \cdot \exp(-\frac{\beta^2}{2\sigma_{\max}}) \cdot \frac{R}{\sigma_{\min}} \cdot \|x-\wt{x}\|_2
\end{align*}
where the first step follows from simple algebra, the second step follows from Fact~\ref{fac:norm_bounds}, the third step follows from Lemma~\ref{lem:lip_N}, the fourth step follows from $\alpha \in (0,1)$, and the last step follows from the definition of $\det_{\min},\sigma_{\max},\sigma_{\min}$.
\end{proof}

This lemma calculates Lipschitz constant of function $|(\alpha N_1(x) + (1-\alpha)N_2(x))^{-1}- (\alpha N_1(\wt{x}) + (1-\alpha)N_2(\wt{x}))^{-1}|$.
\begin{lemma}\label{lem:lip_alpha_N_to_neg_1}
If the following conditions hold
\begin{itemize}
    \item Let $x,\wt{x} \in \R^d$.
    \item Let $t \in \R$, and $t \geq 0$.
    \item Let $N_1(x), N_2(x)$ be defined as Definition~\ref{def:N_1_N_2}.
    \item Let $\alpha \in \R$ and $\alpha \in (0,1)$.
    \item Let $\|x-\mu_i(t)\|_2 \leq R$, where $R \geq 1$, for each $i \in \{1,2\}$.
    \item Let $\|x-\mu_i(t)\|_2 \geq \beta$, where $\beta \in (0, 0.1)$, for each $i \in \{1,2\}$.
    \item Let $\alpha N_1(x) + (1-\alpha)N_2(x) \geq \gamma$, where $\gamma \in (0, 0.1)$.
    \item Let $\sigma_{\min} := \min \{ \sigma_{\min}( \Sigma_1(t) ) , \sigma_{\min} ( \Sigma_2(t) ) \}$. 
    \item Let $\sigma_{\max} := \max \{ \sigma_{\max}( \Sigma_1(t) ) , \sigma_{\max} ( \Sigma_2(t) ) \}$.
    \item Let $\det_{\min} := \min \{ \det( \Sigma_1(t) ) , \det ( \Sigma_2(t) ) \}$.
\end{itemize}

Then,
\begin{align*}
    & ~ |(\alpha N_1(x) + (1-\alpha)N_2(x))^{-1}- (\alpha N_1(\wt{x}) + (1-\alpha)N_2(\wt{x}))^{-1}|\\
    \leq & ~ \gamma^{-2} \cdot \frac{1}{(2\pi)^{d/2}\det_{\min}^{1/2}} \cdot \exp(-\frac{\beta^2}{2\sigma_{\max}}) \cdot \frac{R}{\sigma_{\min}} \cdot \|x-\wt{x}\|_2
\end{align*}
\end{lemma}

\begin{proof}
We can show
\begin{align*}
    \mathrm{LHS} \leq & ~ (\alpha N_1(x) + (1-\alpha)N_2(x))^{-1} \cdot (\alpha N_1(\wt{x}) + (1-\alpha)N_2(\wt{x}))^{-1} \\
    & ~ \cdot |\alpha N_1(x) + (1-\alpha)N_2(x) - (\alpha N_1(\wt{x}) + (1-\alpha)N_2(\wt{x}))|\\
    \leq & ~ \gamma^{-2} \cdot |\alpha N_1(x) + (1-\alpha)N_2(x) - (\alpha N_1(\wt{x}) + (1-\alpha)N_2(\wt{x}))|\\
    \leq & ~ \gamma^{-2} \cdot \frac{1}{(2\pi)^{d/2}\det_{\min}^{1/2}} \cdot \exp(-\frac{\beta^2}{2\sigma_{\max}}) \cdot \frac{R}{\sigma_{\min}} \cdot \|x-\wt{x}\|_2
\end{align*}
where the first step follows from simple algebra, the second step follows from $\alpha N_1(x) + (1-\alpha)N_2(x) \geq \gamma$, and the last step follows from Lemma~\ref{lem:lip_alpha_N}.
\end{proof}

\subsection{Lemmas for Lipschitz calculation: \texorpdfstring{$f(x)$}{}} \label{sec:general_two_gaussian:f_x_for_lip_score}

This lemma calculates Lipschitz constant of function $|f(x) - f(\wt{x})|$.
\begin{lemma}\label{lem:lip_f}
If the following conditions hold
\begin{itemize}
    \item Let $x,\wt{x} \in \R^d$.
    \item Let $t \in \R$, and $t \geq 0$.
    \item Let $N_1(x), N_2(x)$ be defined as Definition~\ref{def:N_1_N_2}.
    \item Let $f(x)$ be defined as Definition~\ref{def:f_x_g_x}.
    \item Let $\alpha \in \R$ and $\alpha \in (0,1)$.
    \item Let $\|x-\mu_i(t)\|_2 \leq R$, where $R \geq 1$, for each $i \in \{1,2\}$.
    \item Let $\|x-\mu_i(t)\|_2 \geq \beta$, where $\beta \in (0, 0.1)$, for each $i \in \{1,2\}$.
    \item Let $\alpha N_1(x) + (1-\alpha)N_2(x) \geq \gamma$, where $\gamma \in (0, 0.1)$.
    \item Let $\sigma_{\min} := \min \{ \sigma_{\min}( \Sigma_1(t) ) , \sigma_{\min} ( \Sigma_2(t) ) \}$. 
    \item Let $\sigma_{\max} := \max \{ \sigma_{\max}( \Sigma_1(t) ) , \sigma_{\max} ( \Sigma_2(t) ) \}$.
    \item Let $\det_{\min} := \min \{ \det( \Sigma_1(t) ) , \det ( \Sigma_2(t) ) \}$.
\end{itemize}

Then,
\begin{align*}
    |f(x) - f(\wt{x})| \leq 2 \alpha \cdot \gamma^{-2} \cdot ( \frac{1}{(2\pi)^{d}\det_{\min}} + \frac{1}{(2\pi)^{d/2}\det_{\min}^{1/2}} ) \cdot \exp(-\frac{\beta^2}{2\sigma_{\max}}) \cdot \frac{R}{\sigma_{\min}} \cdot \|x-\wt{x}\|_2
\end{align*}
\end{lemma}

\begin{proof}
We can show
\begin{align*}
    |f(x) - f(\wt{x})| = & ~ |\frac{\alpha N_1 (x)}{\alpha N_1(x) + (1-\alpha)N_2(x)} - \frac{\alpha N_1 (\wt{x})}{\alpha N_1(\wt{x}) + (1-\alpha)N_2(\wt{x})}|\\
    \leq & ~ |\frac{\alpha N_1 (x)}{\alpha N_1(x) + (1-\alpha)N_2(x)} - \frac{\alpha N_1 (x)}{\alpha N_1(\wt{x}) + (1-\alpha)N_2(\wt{x})}|\\
    & ~ + |\frac{\alpha N_1 (x)}{\alpha N_1(\wt{x}) + (1-\alpha)N_2(\wt{x})} - \frac{\alpha N_1 (\wt{x})}{\alpha N_1(\wt{x}) + (1-\alpha)N_2(\wt{x})}|\\
    = & ~ \alpha \cdot | N_1 (x)| \cdot |(\alpha N_1(x) + (1-\alpha)N_2(x))^{-1}- (\alpha N_1(\wt{x}) + (1-\alpha)N_2(\wt{x}))^{-1}|\\
    & ~ + \alpha \cdot |N_1(x)-N_1(\wt{x})|\cdot|(\alpha N_1(\wt{x}) + (1-\alpha)N_2(\wt{x}))^{-1}|
\end{align*}
where the first step follows from Definition~\ref{def:f_x_g_x}, the second step follows from Fact~\ref{fac:norm_bounds}, and the last step follows from simple algebra.

For the first term in the above, we have
\begin{align}\label{eq:lip_f_1st_term}
    & ~ \alpha \cdot | N_1 (x)| \cdot |(\alpha N_1(x) + (1-\alpha)N_2(x))^{-1}- (\alpha N_1(\wt{x}) + (1-\alpha)N_2(\wt{x}))^{-1}| \notag \\
    \leq & ~ \alpha \cdot \frac{1}{(2\pi)^{d/2}\det(\Sigma_1(t))^{1/2}} \cdot |(\alpha N_1(x) + (1-\alpha)N_2(x))^{-1}- (\alpha N_1(\wt{x}) + (1-\alpha)N_2(\wt{x}))^{-1}| \notag\\
    \leq & ~ \alpha \cdot \frac{1}{(2\pi)^{d/2}\det(\Sigma_1(t))^{1/2}} \cdot \gamma^{-2} \cdot \frac{1}{(2\pi)^{d/2}\det_{\min}^{1/2}} \cdot \exp(-\frac{\beta^2}{2\sigma_{\max}}) \cdot \frac{R}{\sigma_{\min}} \cdot \|x-\wt{x}\|_2 \notag\\
    \leq & ~ \alpha \cdot \gamma^{-2} \cdot \frac{1}{(2\pi)^{d}\det_{\min}} \cdot \exp(-\frac{\beta^2}{2\sigma_{\max}}) \cdot \frac{R}{\sigma_{\min}} \cdot \|x-\wt{x}\|_2
\end{align}
where the first step follows from $N_1(x) \leq \frac{1}{(2\pi)^{d/2}\det(\Sigma_1(t))^{1/2}}$, the second step follows from Lemma~\ref{lem:lip_alpha_N_to_neg_1} and the last step follows from definition of $\det_{\min}$.

For the second term in the above, we have
\begin{align}\label{eq:lip_f_2nd_term}
    & ~ \alpha \cdot |N_1(x)-N_1(\wt{x})|\cdot|(\alpha N_1(\wt{x}) + (1-\alpha)N_2(\wt{x}))^{-1}| \notag\\
    \leq & ~  \alpha\cdot \gamma^{-1} \cdot|N_1(x)-N_1(\wt{x})| \notag\\
    \leq & ~ \alpha\cdot \gamma^{-1}\cdot \frac{1}{(2\pi)^{d/2}\det(\Sigma_1(t))^{1/2}} \cdot \exp(-\frac{\beta^2}{2\sigma_{\max}(\Sigma_1(t))}) \cdot \frac{R}{\sigma_{\min}(\Sigma_1(t))} \cdot \|x-\wt{x}\|_2
\end{align}
where the first step follows from $\alpha N_1(x) + (1-\alpha)N_2(x) \geq \gamma$, the second step follows from Lemma~\ref{lem:lip_alpha_N}.

Combining Eq.~\eqref{eq:lip_f_1st_term} and Eq.~\eqref{eq:lip_f_2nd_term} together, we have
\begin{align*}
    |f(x) - f(\wt{x})| \leq & ~
    \alpha \cdot \gamma^{-2} \cdot \frac{1}{(2\pi)^{d}\det_{\min}} \cdot \exp(-\frac{\beta^2}{2\sigma_{\max}}) \cdot \frac{R}{\sigma_{\min}} \cdot \|x-\wt{x}\|_2\\
    & ~ + \alpha\cdot \gamma^{-1}\cdot \frac{1}{(2\pi)^{d/2}\det(\Sigma_1(t))^{1/2}} \cdot \exp(-\frac{\beta^2}{2\sigma_{\max}(\Sigma_1(t))}) \cdot \frac{R}{\sigma_{\min}(\Sigma_1(t))} \cdot \|x-\wt{x}\|_2\\
    \leq & ~
    \alpha \cdot \gamma^{-2} \cdot \frac{1}{(2\pi)^{d}\det_{\min}} \cdot \exp(-\frac{\beta^2}{2\sigma_{\max}}) \cdot \frac{R}{\sigma_{\min}} \cdot \|x-\wt{x}\|_2\\
    & ~ + \alpha\cdot \gamma^{-1}\cdot \frac{1}{(2\pi)^{d/2}\det_{\min}^{1/2}} \cdot \exp(-\frac{\beta^2}{2\sigma_{\max}}) \cdot \frac{R}{\sigma_{\min}} \cdot \|x-\wt{x}\|_2\\
    \leq & ~ 2 \alpha \cdot \gamma^{-2} \cdot ( \frac{1}{(2\pi)^{d}\det_{\min}} + \frac{1}{(2\pi)^{d/2}\det_{\min}^{1/2}} ) \cdot \exp(-\frac{\beta^2}{2\sigma_{\max}}) \cdot \frac{R}{\sigma_{\min}} \cdot \|x-\wt{x}\|_2
\end{align*}
where the first step follows from the bound of the first term and the second term, the second step follows from the definition of $\det_{\min},\sigma_{\max},\sigma_{\min}$, and the third step follows from $\gamma < 0.1$.
\end{proof}

This lemma calculates Lipschitz constant of function $\|f(x)(- \Sigma_1(t)^{-1} (x - \mu_1(t)) )- f(\wt{x})(- \Sigma_1(t)^{-1} (\wt{x} - \mu_1(t)) )\|_2$.
\begin{lemma}\label{lem:lip_f_x_minus_mu}
If the following conditions hold
\begin{itemize}
    \item Let $x,\wt{x} \in \R^d$.
    \item Let $t \in \R$, and $t \geq 0$.
    \item Let $N_1(x), N_2(x)$ be defined as Definition~\ref{def:N_1_N_2}.
    \item Let $f(x)$ be defined as Definition~\ref{def:f_x_g_x}.
    \item Let $\alpha \in \R$ and $\alpha \in (0,1)$.
    \item Let $\|x-\mu_i(t)\|_2 \leq R$, where $R \geq 1$, for each $i \in \{1,2\}$.
    \item Let $\|x-\mu_i(t)\|_2 \geq \beta$, where $\beta \in (0, 0.1)$, for each $i \in \{1,2\}$.
    \item Let $\alpha N_1(x) + (1-\alpha)N_2(x) \geq \gamma$, where $\gamma \in (0, 0.1)$.
    \item Let $\sigma_{\min} := \min \{ \sigma_{\min}( \Sigma_1(t) ) , \sigma_{\min} ( \Sigma_2(t) ) \}$. 
    \item Let $\sigma_{\max} := \max \{ \sigma_{\max}( \Sigma_1(t) ) , \sigma_{\max} ( \Sigma_2(t) ) \}$.
    \item Let $\det_{\min} := \min \{ \det( \Sigma_1(t) ) , \det ( \Sigma_2(t) ) \}$.
\end{itemize} 

Then, we have
\begin{align*}
    & ~ \|f(x)(- \Sigma_1(t)^{-1} (x - \mu_1(t)) )- f(\wt{x})(- \Sigma_1(t)^{-1} (\wt{x} - \mu_1(t)) )\|_2\\
    \leq & ~ (\frac{1}{\sigma_{\min}} + 2 \alpha \cdot \gamma^{-2} \cdot ( \frac{1}{(2\pi)^{d}\det_{\min}} + \frac{1}{(2\pi)^{d/2}\det_{\min}^{1/2}} ) \cdot \exp(-\frac{\beta^2}{2\sigma_{\max}}) \cdot (\frac{R}{\sigma_{\min}})^2) \cdot \|x-\wt{x}\|_2
\end{align*}
\end{lemma}

\begin{proof}
We can show
\begin{align*}
    \mathrm{LHS} \leq & ~ \|f(x)(- \Sigma_1(t)^{-1} (x - \mu_1(t)) )- f(x)(- \Sigma_1(t)^{-1} (\wt{x} - \mu_1(t)) )\|_2\\  
    & ~ + \|f(x)(- \Sigma_1(t)^{-1} (\wt{x} - \mu_1(t)) )- f(\wt{x})(- \Sigma_1(t)^{-1} (\wt{x} - \mu_1(t)) )\|_2  \\
    \leq & ~ |f(x)| \cdot \|(- \Sigma_1(t)^{-1} (x - \mu_1(t)) )- (- \Sigma_1(t)^{-1} (\wt{x} - \mu_1(t)) )\|_2\\
    & ~ + |f(x)-f(\wt{x})| \cdot \|- \Sigma_1(t)^{-1} (\wt{x} - \mu_1(t))\|_2
\end{align*}
where the first step follows from Fact~\ref{fac:norm_bounds}, the second step follows from Fact~\ref{fac:norm_bounds}.

For the first term in the above, we have
\begin{align}\label{eq:lip_f_x_minus_mu_1st_term}
    & ~ |f(x)| \cdot \|(- \Sigma_1(t)^{-1} (x - \mu_1(t)) )- (- \Sigma_1(t)^{-1} (\wt{x} - \mu_1(t)) )\|_2 \notag\\
    \leq & ~ \|(- \Sigma_1(t)^{-1} (x - \mu_1(t)) )- (- \Sigma_1(t)^{-1} (\wt{x} - \mu_1(t)) )\|_2 \notag\\
    \leq & ~ \frac{1}{\sigma_{\min}(\Sigma_1(t))}  \cdot \|x-\wt{x}\|_2
\end{align}
where the first step follows from $f(x)\leq1$, the second step follows from Lemma~\ref{lem:lip_x_minus_mu}.

For the second term in the above, we have
\begin{align}\label{eq:lip_f_x_minus_mu_2nd_term}
    & ~ |f(x)-f(\wt{x})| \cdot \|- \Sigma_1(t)^{-1} (\wt{x} - \mu_1(t))\|_2 \notag\\
    \leq & ~ \frac{R}{\sigma_{\min}(\Sigma_1(t))} \cdot  |f(x)-f(\wt{x})| \notag\\
    \leq & ~ \frac{R}{\sigma_{\min}(\Sigma_1(t))} \cdot  2 \alpha \cdot \gamma^{-2} \cdot ( \frac{1}{(2\pi)^{d}\det_{\min}} + \frac{1}{(2\pi)^{d/2}\det_{\min}^{1/2}} ) \cdot \exp(-\frac{\beta^2}{2\sigma_{\max}}) \cdot \frac{R}{\sigma_{\min}} \cdot \|x-\wt{x}\|_2
\end{align}
where the first step follows from Lemma~\ref{lem:upper_bound_x_minus_mu}, the second step follows from Lemma~\ref{lem:lip_f}.

Combining Eq.~\eqref{eq:lip_f_x_minus_mu_1st_term} and Eq.~\eqref{eq:lip_f_x_minus_mu_2nd_term} together, we have
\begin{align*}
    & ~ \|f(x)(- \Sigma_1(t)^{-1} (x - \mu_1(t)) )- f(\wt{x})(- \Sigma_1(t)^{-1} (\wt{x} - \mu_1(t)) )\|_2\\
    \leq & ~ \frac{1}{\sigma_{\min}(\Sigma_1(t))}  \cdot \|x-\wt{x}\|_2 \\
    & ~ + \frac{R}{\sigma_{\min}(\Sigma_1(t))} \cdot  2 \alpha \cdot \gamma^{-2} \cdot ( \frac{1}{(2\pi)^{d}\det_{\min}} + \frac{1}{(2\pi)^{d/2}\det_{\min}^{1/2}} )\cdot \exp(-\frac{\beta^2}{2\sigma_{\max}}) \cdot \frac{R}{\sigma_{\min}} \cdot \|x-\wt{x}\|_2\\
    \leq & ~ \frac{1}{\sigma_{\min}}  \cdot \|x-\wt{x}\|_2 \\
    & ~ + \frac{R}{\sigma_{\min}} \cdot  2 \alpha \cdot \gamma^{-2} \cdot ( \frac{1}{(2\pi)^{d}\det_{\min}} + \frac{1}{(2\pi)^{d/2}\det_{\min}^{1/2}} ) \cdot \exp(-\frac{\beta^2}{2\sigma_{\max}}) \cdot \frac{R}{\sigma_{\min}} \cdot \|x-\wt{x}\|_2\\
    = & ~ (\frac{1}{\sigma_{\min}} + 2 \alpha \cdot \gamma^{-2} \cdot ( \frac{1}{(2\pi)^{d}\det_{\min}} + \frac{1}{(2\pi)^{d/2}\det_{\min}^{1/2}} ) \cdot \exp(-\frac{\beta^2}{2\sigma_{\max}}) \cdot (\frac{R}{\sigma_{\min}})^2) \cdot \|x-\wt{x}\|_2
\end{align*}
where the first step follows from the bound of the first term and the second term, the second step follows from the definition of $\det_{\min},\sigma_{\max},\sigma_{\min}$, and the last step follows from simple algebra.
\end{proof}

\subsection{Lemmas for Lipschitz calculation: \texorpdfstring{$g(x)$}{}}\label{sec:general_two_gaussian:g_x_for_lip_score}

This lemma calculates Lipschitz constant of function $|g(x) - g(\wt{x})|$.
\begin{lemma}\label{lem:lip_g}
If the following conditions hold
\begin{itemize}
    \item Let $x,\wt{x} \in \R^d$.
    \item Let $t \in \R$, and $t \geq 0$.
    \item Let $N_1(x), N_2(x)$ be defined as Definition~\ref{def:N_1_N_2}.
    \item Let $g(x)$ be defined as Definition~\ref{def:f_x_g_x}.
    \item Let $\alpha \in \R$ and $\alpha \in (0,1)$.
    \item Let $\|x-\mu_i(t)\|_2 \leq R$, where $R \geq 1$, for each $i \in \{1,2\}$.
    \item Let $\|x-\mu_i(t)\|_2 \geq \beta$, where $\beta \in (0, 0.1)$, for each $i \in \{1,2\}$.
    \item Let $\alpha N_1(x) + (1-\alpha)N_2(x) \geq \gamma$, where $\gamma \in (0, 0.1)$.
    \item Let $\sigma_{\min} := \min \{ \sigma_{\min}( \Sigma_1(t) ) , \sigma_{\min} ( \Sigma_2(t) ) \}$. 
    \item Let $\sigma_{\max} := \max \{ \sigma_{\max}( \Sigma_1(t) ) , \sigma_{\max} ( \Sigma_2(t) ) \}$.
    \item Let $\det_{\min} := \min \{ \det( \Sigma_1(t) ) , \det ( \Sigma_2(t) ) \}$.
\end{itemize}

Then,
\begin{align*}
    |g(x) - g(\wt{x})| \leq 2 (1-\alpha) \cdot \gamma^{-2} \cdot (\frac{1}{(2\pi)^{d}\det_{\min}}+\frac{1}{(2\pi)^{d/2}\det_{\min}^{1/2}}) \cdot \exp(-\frac{\beta^2}{2\sigma_{\max}}) \cdot \frac{R}{\sigma_{\min}} \cdot \|x-\wt{x}\|_2
\end{align*}
\end{lemma}

\begin{proof}
We can show
\begin{align*}
    |g(x) - g(\wt{x})| = & ~ |\frac{(1-\alpha) N_2 (x)}{\alpha N_1(x) + (1-\alpha)N_2(x)} - \frac{(1-\alpha) N_2 (\wt{x})}{\alpha N_1(\wt{x}) + (1-\alpha)N_2(\wt{x})}|\\
    \leq & ~  |\frac{(1-\alpha) N_2 (x)}{\alpha N_1(x) + (1-\alpha)N_2(x)} - \frac{(1-\alpha) N_2 (x)}{\alpha N_1(\wt{x}) + (1-\alpha)N_2(\wt{x})}|\\
    & ~ + |\frac{(1-\alpha) N_2 (x)}{\alpha N_1(\wt{x}) + (1-\alpha)N_2(\wt{x})} - \frac{(1-\alpha) N_2 (\wt{x})}{\alpha N_1(\wt{x}) + (1-\alpha)N_2(\wt{x})}|\\ 
    = & ~ (1-\alpha) \cdot | N_2 (x)| \cdot |(\alpha N_1(x) + (1-\alpha)N_2(x))^{-1}- (\alpha N_1(\wt{x}) + (1-\alpha)N_2(\wt{x}))^{-1}|\\
    & ~ + (1-\alpha) \cdot |N_2(x)-N_2(\wt{x})|\cdot|(\alpha N_1(\wt{x}) + (1-\alpha)N_2(\wt{x}))^{-1}|
\end{align*}
where the first step follows from Definition~\ref{def:f_x_g_x}, the second step follows from Fact~\ref{fac:norm_bounds}, and the last step follows from simple algebra.

For the first term in the above, we have
\begin{align}\label{eq:lip_g_1st_term}
    & ~ (1-\alpha) \cdot | N_2 (x)| \cdot |(\alpha N_1(x) + (1-\alpha)N_2(x))^{-1}- (\alpha N_1(\wt{x}) + (1-\alpha)N_2(\wt{x}))^{-1}| \notag \\
    \leq & ~ (1-\alpha) \cdot \frac{1}{(2\pi)^{d/2}\det(\Sigma_2(t))^{1/2}} \cdot |(\alpha N_1(x) + (1-\alpha)N_2(x))^{-1}- (\alpha N_1(\wt{x}) + (1-\alpha)N_2(\wt{x}))^{-1}| \notag\\
    \leq & ~ (1-\alpha) \cdot \frac{1}{(2\pi)^{d/2}\det(\Sigma_2(t))^{1/2}} \cdot \gamma^{-2} \cdot \frac{1}{(2\pi)^{d/2}\det_{\min}^{1/2}} \cdot \exp(-\frac{\beta^2}{2\sigma_{\max}}) \cdot \frac{R}{\sigma_{\min}} \cdot \|x-\wt{x}\|_2 \notag \\
    \leq & ~ (1-\alpha) \cdot \gamma^{-2} \cdot \frac{1}{(2\pi)^{d}\det_{\min}} \cdot \exp(-\frac{\beta^2}{2\sigma_{\max}}) \cdot \frac{R}{\sigma_{\min}} \cdot \|x-\wt{x}\|_2
\end{align}
where the first step follows from $N_2(x) \leq \frac{1}{(2\pi)^{d/2}\det(\Sigma_2(t))^{1/2}}$, the second step follows from Lemma~\ref{lem:lip_alpha_N_to_neg_1}.

For the second term in the above, we have
\begin{align}\label{eq:lip_g_2nd_term}
    & ~ (1-\alpha) \cdot |N_2(x)-N_2(\wt{x})|\cdot|(\alpha N_1(\wt{x}) + (1-\alpha)N_2(\wt{x}))^{-1}| \notag\\
    \leq & ~  (1-\alpha)\cdot \gamma^{-1} \cdot|N_2(x)-N_2(\wt{x})| \notag\\
    \leq & ~ (1-\alpha)\cdot \gamma^{-1}\cdot \frac{1}{(2\pi)^{d/2}\det(\Sigma_2(t))^{1/2}} \cdot \exp(-\frac{\beta^2}{2\sigma_{\max}(\Sigma_2(t))}) \cdot \frac{R}{\sigma_{\min}(\Sigma_2(t))} \cdot \|x-\wt{x}\|_2
\end{align}
where the first step follows from $\alpha N_1(x) + (1-\alpha)N_2(x) \geq \gamma$, the second step follows from Lemma~\ref{lem:lip_alpha_N}.

Combining Eq.~\eqref{eq:lip_g_1st_term} and Eq.~\eqref{eq:lip_g_2nd_term} together, we have
\begin{align*}
    |g(x) - g(\wt{x})| \leq & ~
    (1-\alpha) \cdot \gamma^{-2} \cdot \frac{1}{(2\pi)^{d}\det_{\min}} \cdot \exp(-\frac{\beta^2}{2\sigma_{\max}}) \cdot \frac{R}{\sigma_{\min}} \cdot \|x-\wt{x}\|_2\\
    & ~ + (1-\alpha)\cdot \gamma^{-1}\cdot \frac{1}{(2\pi)^{d/2}\det(\Sigma_2(t))^{1/2}} \cdot \exp(-\frac{\beta^2}{2\sigma_{\max}(\Sigma_2(t))}) \cdot \frac{R}{\sigma_{\min}(\Sigma_2(t))} \cdot \|x-\wt{x}\|_2\\
    \leq & ~
    (1-\alpha) \cdot \gamma^{-2} \cdot \frac{1}{(2\pi)^{d}\det_{\min}} \cdot \exp(-\frac{\beta^2}{2\sigma_{\max}}) \cdot \frac{R}{\sigma_{\min}} \cdot \|x-\wt{x}\|_2\\
    & ~ + (1-\alpha)\cdot \gamma^{-1}\cdot \frac{1}{(2\pi)^{d/2}\det_{\min}^{1/2}} \cdot \exp(-\frac{\beta^2}{2\sigma_{\max}}) \cdot \frac{R}{\sigma_{\min}} \cdot \|x-\wt{x}\|_2\\
    \leq & ~ 2 (1-\alpha) \cdot \gamma^{-2} \cdot (\frac{1}{(2\pi)^{d}\det_{\min}}+\frac{1}{(2\pi)^{d/2}\det_{\min}^{1/2}}) \cdot \exp(-\frac{\beta^2}{2\sigma_{\max}}) \cdot \frac{R}{\sigma_{\min}} \cdot \|x-\wt{x}\|_2
\end{align*}
where the first step follows from the bound of the first term and the second term, the second step follows from the definition of $\det_{\min},\sigma_{\max},\sigma_{\min}$, and the last step follows from $\gamma < 0.1$.
\end{proof}

This lemma calculates Lipschitz constant of function $\|g(x)(- \Sigma_2(t)^{-1} (x - \mu_2(t)) )- g(\wt{x})(- \Sigma_2(t)^{-1} (\wt{x} - \mu_2(t)) )\|_2$.
\begin{lemma}\label{lem:lip_g_x_minus_mu}
If the following conditions hold
\begin{itemize}
    \item Let $x,\wt{x} \in \R^d$.
    \item Let $t \in \R$, and $t \geq 0$.
    \item Let $N_1(x), N_2(x)$ be defined as Definition~\ref{def:N_1_N_2}.
    \item Let $g(x)$ be defined as Definition~\ref{def:f_x_g_x}.
    \item Let $\alpha \in \R$ and $\alpha \in (0,1)$.
    \item Let $\|x-\mu_i(t)\|_2 \leq R$, where $R \geq 1$, for each $i \in \{1,2\}$.
    \item Let $\|x-\mu_i(t)\|_2 \geq \beta$, where $\beta \in (0, 0.1)$, for each $i \in \{1,2\}$.
    \item Let $\alpha N_1(x) + (1-\alpha)N_2(x) \geq \gamma$, where $\gamma \in (0, 0.1)$.
    \item Let $\sigma_{\min} := \min \{ \sigma_{\min}( \Sigma_1(t) ) , \sigma_{\min} ( \Sigma_2(t) ) \}$. 
    \item Let $\sigma_{\max} := \max \{ \sigma_{\max}( \Sigma_1(t) ) , \sigma_{\max} ( \Sigma_2(t) ) \}$.
    \item Let $\det_{\min} := \min \{ \det( \Sigma_1(t) ) , \det ( \Sigma_2(t) ) \}$.
\end{itemize} 

Then, we have
\begin{align*}
    & ~ \|g(x)(- \Sigma_2(t)^{-1} (x - \mu_2(t)) )- g(\wt{x})(- \Sigma_2(t)^{-1} (\wt{x} - \mu_2(t)) )\|_2\\
    \leq & ~  (\frac{1}{\sigma_{\min}}+ 2 (1-\alpha) \cdot \gamma^{-2} \cdot (\frac{1}{(2\pi)^{d}\det_{\min}}+\frac{1}{(2\pi)^{d/2}\det_{\min}^{1/2}})\cdot \exp(-\frac{\beta^2}{2\sigma_{\max}}) \cdot (\frac{R}{\sigma_{\min}})^2) \cdot \|x-\wt{x}\|_2
\end{align*}
\end{lemma}

\begin{proof}
We can show
\begin{align*}
    \mathrm{LHS} \leq & ~ \|g(x)(- \Sigma_2(t)^{-1} (x - \mu_2(t)) )- g(x)(- \Sigma_2(t)^{-1} (\wt{x} - \mu_2(t)) )\|_2\\  
    & ~ + \|g(x)(- \Sigma_2(t)^{-1} (\wt{x} - \mu_2(t)) )- f(\wt{x})(- \Sigma_2(t)^{-1} (\wt{x} - \mu_2(t)) )\|_2  \\
    \leq & ~ |g(x)| \cdot \|(- \Sigma_2(t)^{-1} (x - \mu_2(t)) )- (- \Sigma_2(t)^{-1} (\wt{x} - \mu_2(t)) )\|_2\\
    & ~ + |g(x)-g(\wt{x})| \cdot \|- \Sigma_2(t)^{-1} (\wt{x} - \mu_2(t))\|_2
\end{align*}
where the first step follows from Fact~\ref{fac:norm_bounds}, the second step follows from Fact~\ref{fac:norm_bounds}.

For the first term in the above, we have
\begin{align}\label{eq:lip_g_x_minus_mu_1st_term}
    & ~ |g(x)| \cdot \|(- \Sigma_2(t)^{-1} (x - \mu_2(t)) )- (- \Sigma_2(t)^{-1} (\wt{x} - \mu_2(t)) )\|_2 \notag\\
    \leq & ~ \|(- \Sigma_2(t)^{-1} (x - \mu_2(t)) )- (- \Sigma_2(t)^{-1} (\wt{x} - \mu_2(t)) )\|_2 \notag\\
    \leq & ~ \frac{1}{\sigma_{\min}(\Sigma_2(t))}  \cdot \|x-\wt{x}\|_2
\end{align}
where the first step follows from $g(x)\leq1$, the second step follows from Lemma~\ref{lem:lip_x_minus_mu}.

For the second term in the above, we have
\begin{align}\label{eq:lip_g_x_minus_mu_2nd_term}
    & ~ |g(x)-g(\wt{x})| \cdot \|- \Sigma_2(t)^{-1} (\wt{x} - \mu_2(t))\|_2 \notag\\
    \leq & ~ \frac{R}{\sigma_{\min}(\Sigma_2(t))} \cdot  |g(x)-g(\wt{x})| \notag\\
    \leq & ~ \frac{R}{\sigma_{\min}(\Sigma_2(t))} \cdot  2 (1-\alpha) \cdot \gamma^{-2} \cdot (\frac{1}{(2\pi)^{d}\det_{\min}}+\frac{1}{(2\pi)^{d/2}\det_{\min}^{1/2}}) \cdot \exp(-\frac{\beta^2}{2\sigma_{\max}}) \cdot \frac{R}{\sigma_{\min}} \cdot \|x-\wt{x}\|_2
\end{align}
where the first step follows from Lemma~\ref{lem:upper_bound_x_minus_mu}, the second step follows from Lemma~\ref{lem:lip_g}.

Combining Eq.~\eqref{eq:lip_g_x_minus_mu_1st_term} and Eq.~\eqref{eq:lip_g_x_minus_mu_2nd_term} together, we have
\begin{align*}
    & ~ \|g(x)(- \Sigma_2(t)^{-1} (x - \mu_2(t)) )- g(\wt{x})(- \Sigma_2(t)^{-1} (\wt{x} - \mu_2(t)) )\|_2\\
    \leq & ~ \frac{1}{\sigma_{\min}(\Sigma_2(t))}  \cdot \|x-\wt{x}\|_2 \\
    & ~ + \frac{R}{\sigma_{\min}(\Sigma_2(t))} \cdot  2 (1-\alpha) \cdot \gamma^{-2} \cdot (\frac{1}{(2\pi)^{d}\det_{\min}}+\frac{1}{(2\pi)^{d/2}\det_{\min}^{1/2}}) \cdot \exp(-\frac{\beta^2}{2\sigma_{\max}}) \cdot \frac{R}{\sigma_{\min}} \cdot \|x-\wt{x}\|_2\\
    \leq & ~ \frac{1}{\sigma_{\min}}  \cdot \|x-\wt{x}\|_2 \\
    & ~ + \frac{R}{\sigma_{\min}} \cdot  2 (1-\alpha) \cdot \gamma^{-2} \cdot (\frac{1}{(2\pi)^{d}\det_{\min}}+\frac{1}{(2\pi)^{d/2}\det_{\min}^{1/2}})\cdot \exp(-\frac{\beta^2}{2\sigma_{\max}}) \cdot \frac{R}{\sigma_{\min}} \cdot \|x-\wt{x}\|_2\\
    = & ~ (\frac{1}{\sigma_{\min}}+ 2 (1-\alpha) \cdot \gamma^{-2} \cdot (\frac{1}{(2\pi)^{d}\det_{\min}}+\frac{1}{(2\pi)^{d/2}\det_{\min}^{1/2}}) \cdot \exp(-\frac{\beta^2}{2\sigma_{\max}}) \cdot (\frac{R}{\sigma_{\min}})^2) \cdot \|x-\wt{x}\|_2
\end{align*}
where the first step follows from the bound of the first term and the second term, the second step follows from the definition of $\det_{\min},\sigma_{\max},\sigma_{\min}$, and the last step follows from simple algebra.
\end{proof}

\subsection{Lipschitz constant of the score function}\label{sec:general_two_gaussian:lip_score}

This lemma calculates the Lipschitz constant of the score funciton.
\begin{lemma}[Lipschitz]
If the following conditions hold
\begin{itemize}
    \item Let $x, \wt{x} \in \R^d$.
    \item Let $t \in \R$, and $t \geq 0$.
    \item Let $N_1(x), N_2(x)$ be defined as Definition~\ref{def:N_1_N_2}.
    \item Let $\alpha \in \R$ and $\alpha \in (0,1)$.
    \item Let $p_t(x)$ be defined as Definition~\ref{def:p_t_continuous_general_2_gaussian}.
    \item Let $f(x), g(x)$ be defined as Definition~\ref{def:f_x_g_x}.
    \item Let $\|x-\mu_i(t)\|_2 \leq R$, where $R \geq 1$, for each $i \in \{1,2\}$.
    \item Let $\|x-\mu_i(t)\|_2 \geq \beta$, where $\beta \in (0, 0.1)$, for each $i \in \{1,2\}$.
    \item Let $\alpha N_1(x) + (1-\alpha)N_2(x) \geq \gamma$, where $\gamma \in (0, 0.1)$.
    \item Let $\sigma_{\min} := \min \{ \sigma_{\min}( \Sigma_1(t) ) , \sigma_{\min} ( \Sigma_2(t) ) \}$. 
    \item Let $\sigma_{\max} := \max \{ \sigma_{\max}( \Sigma_1(t) ) , \sigma_{\max} ( \Sigma_2(t) ) \}$.
    \item Let $\det_{\min} := \min \{ \det( \Sigma_1(t) ) , \det ( \Sigma_2(t) ) \}$.
\end{itemize} 
Then, 
\begin{align*}
    \|\frac{\d \log p_t(x)}{\d x} - \frac{\d \log p_t(\wt{x})}{\d \wt{x}}\|_2 \leq & ~ (\frac{2}{\sigma_{\min}}+  \frac{2R^2}{\gamma^2\sigma_{\min}^2} \cdot (\frac{1}{(2\pi)^{d}\det_{\min}}+\frac{1}{(2\pi)^{d/2}\det_{\min}^{1/2}}) \cdot \exp(-\frac{\beta^2}{2\sigma_{\max}})) \cdot \|x-\wt{x}\|_2
\end{align*}
\end{lemma}

\begin{proof}
We can show
\begin{align*}
    \mathrm{LHS} = & ~ \| f(x)(-\Sigma_1(t)^{-1} (x - \mu_1(t))) + g(x) (-\Sigma_2(t)^{-1} (x - \mu_2(t))) \\
    & - (f(\wt{x})(-\Sigma_1(t)^{-1} (\wt{x} - \mu_1(t))) + g(\wt{x})(-\Sigma_2(t)^{-1} (\wt{x} - \mu_2(t)))) \|_2\\
    \leq & ~ \| f(x)(-\Sigma_1(t)^{-1} (x - \mu_1(t))) - f(\wt{x})(-\Sigma_1(t)^{-1} (\wt{x} - \mu_1(t)))\|_2 \\
    & ~ + \| g(x) (-\Sigma_2(t)^{-1} (x - \mu_2(t))) - g(\wt{x})(-\Sigma_2(t)^{-1} (\wt{x} - \mu_2(t)))\|_2\\
    \leq & ~ (\frac{1}{\sigma_{\min}}+ 2 \alpha \cdot \gamma^{-2} \cdot (\frac{1}{(2\pi)^{d}\det_{\min}}+\frac{1}{(2\pi)^{d/2}\det_{\min}^{1/2}}) \cdot \exp(-\frac{\beta^2}{2\sigma_{\max}}) \cdot (\frac{R}{\sigma_{\min}})^2) \cdot \|x-\wt{x}\|_2\\
    & ~ + \| g(x) (-\Sigma_2(t)^{-1} (x - \mu_2(t))) - g(\wt{x})(-\Sigma_2(t)^{-1} (\wt{x} - \mu_2(t)))\|_2\\
    \leq & ~ (\frac{1}{\sigma_{\min}}+ 2 \alpha \cdot \gamma^{-2} \cdot (\frac{1}{(2\pi)^{d}\det_{\min}}+\frac{1}{(2\pi)^{d/2}\det_{\min}^{1/2}}) \cdot \exp(-\frac{\beta^2}{2\sigma_{\max}}) \cdot (\frac{R}{\sigma_{\min}})^2) \cdot \|x-\wt{x}\|_2\\
    & ~ +  (\frac{1}{\sigma_{\min}}+ 2 (1-\alpha) \cdot \gamma^{-2} \cdot (\frac{1}{(2\pi)^{d}\det_{\min}}+\frac{1}{(2\pi)^{d/2}\det_{\min}^{1/2}}) \cdot \exp(-\frac{\beta^2}{2\sigma_{\max}}) \cdot (\frac{R}{\sigma_{\min}})^2) \cdot \|x-\wt{x}\|_2\\
    = & ~ (\frac{2}{\sigma_{\min}}+  \frac{2R^2}{\gamma^2\sigma_{\min}^2} \cdot (\frac{1}{(2\pi)^{d}\det_{\min}}+\frac{1}{(2\pi)^{d/2}\det_{\min}^{1/2}}) \cdot \exp(-\frac{\beta^2}{2\sigma_{\max}})) \cdot \|x-\wt{x}\|_2
\end{align*}
where the first step follows from Lemma~\ref{lem:gradient_log_p_t_continuous_general_2_gaussian}, the second step follows from Fact~\ref{fac:norm_bounds}, the third step follows from Lemma~\ref{lem:lip_f_x_minus_mu}, the fourth step follows from Lemma~\ref{lem:lip_g_x_minus_mu}, and the last step follows from simple algebra.
\end{proof}

\section{A General Version for \texorpdfstring{$k$}{} Gaussian}\label{sec:k_gaussain}

In this section we consider a more general case of $k$ mixture of Gaussians.

\begin{itemize}
    \item Section~\ref{sec:k_gaussain:definitions} provides the definition for $k$ mixture of Gaussians.
    \item Section~\ref{sec:k_gaussain:score} provides the expression of the score function.
    \item Section~\ref{sec:k_gaussain:upper_bound} provides the upper bound of the score function.
    \item Section~\ref{sec:k_gaussain:lemmas_for_lip} provides lemmas that are used in further calculation of Lipschitz constant.
    \item Section~\ref{sec:k_gaussain:lip} provides the Lipschitz constant for $k$ mixture of Gaussians.
\end{itemize}

\subsection{Definitions}\label{sec:k_gaussain:definitions}

Let $i \in [k]$.
Let $\alpha_i(t) \in (0,1)$, $\sum_{i=1}^k \alpha_i(t) = 1$, and is a function of time $t$. Consider $p_t $ such that
\begin{align*}
    p_t(x) = \Pr_{x' \sim \sum_{i=1}^k \alpha_i(t)\mathcal{N}( \mu_i(t), \Sigma_i(t)) }[ x' = x]
\end{align*}
where $\mu_i(t) \in \R^d$, $\Sigma_i(t) \in \R^{d \times d}$ and they are derivative to $t$ and $\Sigma_i(t)$ is a symmetric p.s.d. matrix whose the smallest singular value is always larger than a fixed value $\sigma_{\min} > 0$.

Below we define the $\mathsf{pdf}$ for a single multivariate Gaussian.
\begin{definition}\label{def:N_i}
If the following conditions hold
\begin{itemize}
    \item Let $x \in \R^d$.
    \item Let $i \in [k]$.
    \item Let $t \in \R$, and $t \geq 0$.
\end{itemize}

We define 
\begin{align*}
    N_i(x) := \frac{1}{(2\pi)^{d/2}\det(\Sigma_i(t))^{1/2}}  \exp(-\frac{1}{2} (x - \mu_i(t))^\top \Sigma_i(t)^{-1} (x - \mu_1(t)))
\end{align*}

This is the $\mathsf{pdf}$ of a single Gaussian so it's clearly to see that $0 \leq N_i \leq \frac{1}{(2\pi)^{d/2}\det(\Sigma_i(t))^{1/2}}$ since $N_i(x)$ takes maximum when $x=\mu_i$.
\end{definition}

Below we define the $\mathsf{pdf}$ for $k$ mixtures of Gaussians.
\begin{definition}\label{def:p_t_k_gaussian}
If the following conditions hold
\begin{itemize}
    \item Let $x \in \R^d$.
    \item Let $i \in [k]$.
    \item Let $t \in \R$, and $t \geq 0$.
    \item Let $\alpha_i(t) \in \R$, $\sum_{i=1}^k \alpha_i(t) = 1$,  and $\alpha_i(t) \in (0,1)$.
    \item Let $N_i(x)$ be defined as Definition~\ref{def:N_i}.
\end{itemize}

We define
\begin{align*}
    p_t(x) := \sum_{i=1}^k \frac{\alpha_i(t)}{(2\pi)^{d/2}\det(\Sigma_i(t))^{1/2}}  \exp(-\frac{1}{2} (x - \mu_i(t))^\top \Sigma_i(t)^{-1} (x - \mu_i(t)))
\end{align*}

This can be further rewritten as follows:
\begin{align*}
    p_t(x) = \sum_{i=1}^k \alpha_i(t) N_i(x)
\end{align*}

Further, we have 
\begin{align*}
    \log p_t(x) = \log (\sum_{i=1}^k \alpha_i(t) N_i(x))
\end{align*}
\end{definition}

This lemma calculates the gradient of $\mathsf{pdf}$ for $k$ mixture of Gaussians.
\begin{lemma}\label{lem:gradient_p_t_k_gaussian}
If the following conditions hold
\begin{itemize}
    \item Let $x \in \R^d$.
    \item Let $i \in [k]$.
    \item Let $t \in \R$, and $t \geq 0$.
    \item Let $\alpha_i(t) \in \R$, $\sum_{i=1}^k \alpha_i(t) = 1$,  and $\alpha_i(t) \in (0,1)$.
    \item Let $N_i(x)$ be defined as Definition~\ref{def:N_i}.
    \item Let $p_t(x)$ be defined as Definition~\ref{def:p_t_k_gaussian}
\end{itemize}

We have
\begin{align*}
    \frac{\d p_t(x)}{\d x} = \sum_{i=1}^k \alpha_i(t) N_i(x) (-\Sigma_i(t)^{-1} (x - \mu_i(t)))
\end{align*}
\end{lemma}

\begin{proof}
We can show
\begin{align*}
    \frac{\d p_t(x)}{\d x} = & ~ \frac{\d }{\d x} \sum_{i=1}^k \alpha_i(t) N_i(x)\\
    = & ~ \sum_{i=1}^k \alpha_i(t) \frac{\d N_i(x)}{\d x}\\
    = & ~ \sum_{i=1}^k \alpha_i(t) N_i(x) (-\Sigma_i(t)^{-1} (x - \mu_i(t)))
\end{align*}
where the first step follows from Definition~\ref{def:p_t_k_gaussian}, the second step follows from Fact~\ref{fac:calculus}, and the last step follows from Lemma~\ref{lem:gradient_N_1_N_2}.

\end{proof}

Below we define $f_i$ that simplifies further calculation.
\begin{definition}\label{def:f_i}
If the following conditions hold
\begin{itemize}
    \item Let $x \in \R^d$.
    \item Let $i \in [k]$.
    \item Let $t \in \R$, and $t \geq 0$.
    \item Let $\alpha_i(t) \in \R$, $\sum_{i=1}^k \alpha_i(t) = 1$,  and $\alpha_i(t) \in (0,1)$.
    \item Let $N_i(x)$ be defined as Definition~\ref{def:N_i}.
\end{itemize}

For further simplicity, we define
\begin{align*}
    f_i(x) := \frac{\alpha_i(t) N_i(x)}{\sum_{i=1}^k \alpha_i(t) N_i(x)} 
\end{align*}

It's clearly to see that $0\leq f_i(x) \leq 1$ and $\sum_{i=1}^k f_i(x) = 1$
\end{definition}

\subsection{Calculation of the score function}\label{sec:k_gaussain:score}

This lemma calculates the score function for $k$ mixture of Gaussians.
\begin{lemma}\label{lem:gradient_log_p_t_k_gaussian}
If the following conditions hold
\begin{itemize}
    \item Let $x \in \R^d$.
    \item Let $i \in [k]$.
    \item Let $t \in \R$, and $t \geq 0$.
    \item Let $\alpha_i(t) \in \R$, $\sum_{i=1}^k \alpha_i(t) = 1$,  and $\alpha_i(t) \in (0,1)$.
    \item Let $N_i(x)$ be defined as Definition~\ref{def:N_i}.
    \item Let $p_t(x)$ be defined as Definition~\ref{def:p_t_k_gaussian}.
    \item Let $f_i(x)$ be defined as Definition~\ref{def:f_i}.
\end{itemize}

We have
\begin{align*}
    \frac{\d \log p_t(x)}{\d x} = \sum_{i=1}^k f_i(x) (-\Sigma_i(t)^{-1} (x - \mu_i(t)))
\end{align*}
\end{lemma}

\begin{proof}
We can show
\begin{align*}
    \frac{\d \log p_t(x)}{\d x} = & ~ \frac{\d \log p_t(x)}{\d p_t(x)} \frac{\d p_t(x)}{\d x}\\
    = & ~ \frac{1}{p_t(x)} \frac{\d p_t(x)}{\d x} \\
    = & ~ \frac{1}{p_t(x)} \sum_{i=1}^k \alpha_i(t) N_i(x) (-\Sigma_i(t)^{-1} (x - \mu_i(t)))\\
    = & ~ \frac{\sum_{i=1}^k \alpha_i(t) N_i(x) (-\Sigma_i(t)^{-1} (x - \mu_i(t)))}{\sum_{i=1}^k \alpha_i(t) N_i(x)}\\
    = & ~ \sum_{i=1}^k f_i(x) (-\Sigma_i(t)^{-1} (x - \mu_i(t)))
\end{align*}
where the first step follows from Fact~\ref{fac:calculus}, the second step follows from Fact~\ref{fac:calculus}, the third step follows from Lemma~\ref{lem:gradient_p_t_k_gaussian}, the fourth step follows from Definition~\ref{def:p_t_k_gaussian}, and the last step follows from Definition~\ref{def:f_i}.
\end{proof}

\subsection{Upper bound of the score function}\label{sec:k_gaussain:upper_bound}

This lemma calculates upper bound of the score function for $k$ mixture of Gaussians.
\begin{lemma}
If the following conditions hold
\begin{itemize}
    \item Let $x \in \R^d$.
    \item Let $i \in [k]$.
    \item Let $t \in \R$, and $t \geq 0$.
    \item Let $\alpha_i(t) \in \R$, $\sum_{i=1}^k \alpha_i(t) = 1$,  and $\alpha_i(t) \in (0,1)$.
    \item Let $p_t(x)$ be defined as Definition~\ref{def:p_t_k_gaussian}.
    \item Let $f_i(x)$ be defined as Definition~\ref{def:f_i}.
    \item Let $\sigma_{\min} := \min \{ \sigma_{\min}( \Sigma_1(t) ), \sigma_{\min} ( \Sigma_2(t) ), \dots, \sigma_{\min} ( \Sigma_k(t) ) \}$. 
    \item Let $\mu_{\max}: = \max\{1, \| \mu_1(t) \|_2, \| \mu_2(t) \|_2, \dots, \| \mu_k(t) \|_2\}$. 
\end{itemize}

Then, we have
\begin{align*}
    \| \frac{\d \log p_t(x)}{\d x} \|_2 \leq \sigma_{\min}^{-1} \cdot \mu_{\max} \cdot  (1 + \| x \|_2 )
\end{align*}
\end{lemma}

\begin{proof}
We can show
\begin{align*}
    \| \frac{\d \log p_t(x)}{\d x} \|_2 = & ~ \|\sum_{i=1}^k f_i(x) (-\Sigma_i(t)^{-1} (x - \mu_i(t)))\|_2 \\
    \leq & ~ \sum_{i=1}^k f_i(x) \| -\Sigma_i(t)^{-1} (x - \mu_i(t))\|_2\\
    \leq & ~ \max_{i \in [k]} \| -\Sigma_i(t)^{-1} (x - \mu_i(t))\|_2\\
    \leq & ~ \max_{i \in [k]} \frac{1}{\sigma_{\min}(\Sigma_i(t))}  \cdot (\|x\|_2 + \|\mu_i(t)\|_2)\\
    \leq & ~ \sigma_{\min}^{-1} ( \mu_{\max} + \| x \|_2) \\
    \leq & ~ \sigma_{\min}^{-1} \cdot \mu_{\max} \cdot  (1 + \| x \|_2 )
\end{align*}
where the first step follows from Lemma~\ref{lem:gradient_log_p_t_k_gaussian}, the second step follows from triangle inequality, the third step follows from $\sum_{i=1}^k f_i(x) = 1$ and $f_i(x) \geq 0$, the fourth step follows from Lemma~\ref{lem:bound_x_minus_mu}, the fifth step follows from definition of $\mu_{\max}$ and $\sigma_{\min}$, and the last step follows from $\mu_{\max} \geq 1$. 
\end{proof}

\subsection{Lemmas for Lipshitz calculation}\label{sec:k_gaussain:lemmas_for_lip}
This section provides lemmas for calculation of Lipschitz constant of the score function for $k$ mixture of Gaussians.

This lemma calculates Lipschitz constant of function $|\sum_{i=1}^k \alpha_i(t)N_i(x)-\sum_{i=1}^k \alpha_i(t) N_i(\wt{x})| $.
\begin{lemma}\label{lem:lip_alpha_N_k_gaussian}
If the following conditions hold
\begin{itemize}
    \item Let $x, \wt{x} \in \R^d$.
    \item Let $i \in [k]$.
    \item Let $t \in \R$, and $t \geq 0$.
    \item Let $\alpha_i(t) \in \R$, $\sum_{i=1}^k \alpha_i(t) = 1$,  and $\alpha_i(t) \in (0,1)$.
    \item Let $N_i(x)$ be defined as Definition~\ref{def:N_i}.
    \item Let $\|x-\mu_i(t)\|_2 \leq R$, where $R \geq 1$, for each $i \in [k]$.
    \item Let $\|x-\mu_i(t)\|_2 \geq \beta$, where $\beta \in (0, 0.1)$, for each $i \in [k]$.
    \item Let $\sigma_{\min} := \min \{ \sigma_{\min}( \Sigma_1(t) ), \sigma_{\min} ( \Sigma_2(t) ), \dots, \sigma_{\min} ( \Sigma_k(t) ) \}$.  
    \item Let $\sigma_{\max} := \max \{ \sigma_{\max}( \Sigma_1(t) ) , \sigma_{\max} ( \Sigma_2(t) ), \dots, \sigma_{\max} ( \Sigma_k(t) )\}$.
    \item Let $\det_{\min} := \min \{ \det( \Sigma_1(t) ) , \det ( \Sigma_2(t) ), \dots,  \det ( \Sigma_k(t) )\}$.
\end{itemize}

Then, we have
\begin{align*}
    |\sum_{i=1}^k \alpha_i(t)N_i(x)-\sum_{i=1}^k \alpha_i(t) N_i(\wt{x})| \leq \frac{1}{(2\pi)^{d/2}\det_{\min}^{1/2}} \cdot  \exp(-\frac{\beta^2}{2\sigma_{\max}}) \cdot \frac{R}{\sigma_{\min}} \cdot \|x-\wt{x}\|_2
\end{align*}
\end{lemma}

\begin{proof}
We can show
\begin{align*}
    \mathrm{LHS} = & ~ |\sum_{i=1}^k \alpha_i(t) (N_i(x)- N_i(\wt{x}))|\\
    \leq & ~ \sum_{i=1}^k \alpha_i(t)|N_i(x)- N_i(\wt{x})|\\
    \leq & ~ \sum_{i=1}^k \alpha_i(t) \frac{1}{(2\pi)^{d/2}\det(\Sigma_i(t))^{1/2}} \cdot  \exp(-\frac{\beta^2}{2\sigma_{\max}(\Sigma_i(t))}) \cdot \frac{R}{\sigma_{\min}(\Sigma_i(t))} \cdot \|x-\wt{x}\|_2\\
    \leq & ~ \max_{i \in [k]} \frac{1}{(2\pi)^{d/2}\det(\Sigma_i(t))^{1/2}} \cdot  \exp(-\frac{\beta^2}{2\sigma_{\max}(\Sigma_i(t))}) \cdot \frac{R}{\sigma_{\min}(\Sigma_i(t))} \cdot \|x-\wt{x}\|_2\\
    \leq & ~ \frac{1}{(2\pi)^{d/2}\det_{\min}^{1/2}} \cdot  \exp(-\frac{\beta^2}{2\sigma_{\max}}) \cdot \frac{R}{\sigma_{\min}} \cdot \|x-\wt{x}\|_2
\end{align*}
where the first step follows from simple algebra, the second step follows from Fact~\ref{fac:norm_bounds}, the third step follows from Lemma~\ref{lem:lip_N}, the fourth step follows from $\sum_{i=1}^k \alpha_i(t) = 1$, and $\alpha_i(t) \in (0,1)$, and the last step follows from the definition of $\det_{\min},\sigma_{\max},\sigma_{\min}$.
\end{proof}

This lemma calculates Lipschitz constant of function $|(\sum_{i=1}^k \alpha_i(t)N_i(x))^{-1}-(\sum_{i=1}^k \alpha_i(t) N_i(\wt{x}))^{-1}|$.
\begin{lemma}\label{lem:lip_alpha_N_to_neg_1_k_gaussian}
If the following conditions hold
\begin{itemize}
    \item Let $x, \wt{x} \in \R^d$.
    \item Let $i \in [k]$.
    \item Let $t \in \R$, and $t \geq 0$.
    \item Let $\alpha_i(t) \in \R$, $\sum_{i=1}^k \alpha_i(t) = 1$,  and $\alpha_i(t) \in (0,1)$.
    \item Let $N_i(x)$ be defined as Definition~\ref{def:N_i}.
    \item Let $\|x-\mu_i(t)\|_2 \leq R$, where $R \geq 1$, for each $i \in [k]$.
    \item Let $\|x-\mu_i(t)\|_2 \geq \beta$, where $\beta \in (0, 0.1)$, for each $i \in [k]$.
    \item Let $\sum_{i=1}^k \alpha_i(t) N_i(x) \geq \gamma$, where $\gamma \in (0, 0.1)$.
    \item Let $\sigma_{\min} := \min \{ \sigma_{\min}( \Sigma_1(t) ), \sigma_{\min} ( \Sigma_2(t) ), \dots, \sigma_{\min} ( \Sigma_k(t) ) \}$.  
    \item Let $\sigma_{\max} := \max \{ \sigma_{\max}( \Sigma_1(t) ) , \sigma_{\max} ( \Sigma_2(t) ), \dots, \sigma_{\max} ( \Sigma_k(t) )\}$.
    \item Let $\det_{\min} := \min \{ \det( \Sigma_1(t) ) , \det ( \Sigma_2(t) ), \dots,  \det ( \Sigma_k(t) )\}$.
\end{itemize}

Then, we have
\begin{align*}
    |(\sum_{i=1}^k \alpha_i(t)N_i(x))^{-1}-(\sum_{i=1}^k \alpha_i(t) N_i(\wt{x}))^{-1}| \leq \gamma^{-2} \frac{1}{(2\pi)^{d/2}\det_{\min}^{1/2}} \cdot  \exp(-\frac{\beta^2}{2\sigma_{\max}}) \cdot \frac{R}{\sigma_{\min}} \cdot \|x-\wt{x}\|_2
\end{align*}
\end{lemma}

\begin{proof}
We can show
\begin{align*}
    \mathrm{LHS} = & ~ (\sum_{i=1}^k \alpha_i(t)N_i(x))^{-1} \cdot (\sum_{i=1}^k \alpha_i(t)N_i(\wt{x}))^{-1} \cdot |\sum_{i=1}^k \alpha_i(t)N_i(x)-\sum_{i=1}^k \alpha_i(t) N_i(\wt{x})|\\
    \leq & ~ \gamma^{-2} \cdot |\sum_{i=1}^k \alpha_i(t)N_i(x)-\sum_{i=1}^k \alpha_i(t) N_i(\wt{x})|\\
    \leq & ~ \gamma^{-2} \cdot \frac{1}{(2\pi)^{d/2}\det_{\min}^{1/2}} \cdot  \exp(-\frac{\beta^2}{2\sigma_{\max}}) \cdot \frac{R}{\sigma_{\min}} \cdot \|x-\wt{x}\|_2
\end{align*}
where the first step follows from simple algebra, the second step follows from $\sum_{i=1}^k \alpha_i(t) N_i(x) \geq \gamma$, and the last step follows from Lemma~\ref{lem:lip_alpha_N_k_gaussian}.
\end{proof}

This lemma calculates Lipschitz constant of function $|f_i(x)- f_i(\wt{x})|$.
\begin{lemma}\label{lem:lip_f_i}
If the following conditions hold
\begin{itemize}
    \item Let $x, \wt{x} \in \R^d$.
    \item Let $i \in [k]$.
    \item Let $t \in \R$, and $t \geq 0$.
    \item Let $\alpha_i(t) \in \R$, $\sum_{i=1}^k \alpha_i(t) = 1$,  and $\alpha_i(t) \in (0,1)$.
    \item Let $N_i(x)$ be defined as Definition~\ref{def:N_i}.
    \item Let $f_i(x)$ be defined as Definition~\ref{def:f_i}.
    \item Let $\|x-\mu_i(t)\|_2 \leq R$, where $R \geq 1$, for each $i \in [k]$.
    \item Let $\|x-\mu_i(t)\|_2 \geq \beta$, where $\beta \in (0, 0.1)$, for each $i \in [k]$.
    \item Let $\sum_{i=1}^k \alpha_i(t) N_i(x) \geq \gamma$, where $\gamma \in (0, 0.1)$.
    \item Let $\sigma_{\min} := \min \{ \sigma_{\min}( \Sigma_1(t) ), \sigma_{\min} ( \Sigma_2(t) ), \dots, \sigma_{\min} ( \Sigma_k(t) ) \}$.  
    \item Let $\sigma_{\max} := \max \{ \sigma_{\max}( \Sigma_1(t) ) , \sigma_{\max} ( \Sigma_2(t) ), \dots, \sigma_{\max} ( \Sigma_k(t) )\}$.
    \item Let $\det_{\min} := \min \{ \det( \Sigma_1(t) ) , \det ( \Sigma_2(t) ), \dots,  \det ( \Sigma_k(t) )\}$.
\end{itemize}

Then, for each $i \in [k]$, we have
\begin{align*}
    |f_i(x)- f_i(\wt{x})| \leq 2 \alpha_i(t) \cdot \gamma^{-2} \cdot \frac{1}{(2\pi)^{d/2}\det_{\min}^{1/2}} \cdot \exp(-\frac{\beta^2}{2\sigma_{\max}}) \cdot \frac{R}{\sigma_{\min}} \cdot \|x-\wt{x}\|_2
\end{align*}
\end{lemma}

\begin{proof}
We can show
\begin{align*}
    |f_i(x)- f_i(\wt{x})| = & ~ |\alpha_i(t) N_i(x) \cdot (\sum_{i=1}^k \alpha_i(t) N_i(x))^{-1} - \alpha_i(t) N_i(\wt{x}) \cdot (\sum_{i=1}^k \alpha_i(t) N_i(\wt{x}))^{-1}|\\
    \leq & ~ |\alpha_i(t) N_i(x) \cdot (\sum_{i=1}^k \alpha_i(t) N_i(x))^{-1} - \alpha_i(t) N_i(x) \cdot (\sum_{i=1}^k \alpha_i(t) N_i(\wt{x}))^{-1}|\\
    & ~ + |\alpha_i(t) N_i(x) \cdot (\sum_{i=1}^k \alpha_i(t) N_i(\wt{x}))^{-1} - \alpha_i(t) N_i(\wt{x}) \cdot (\sum_{i=1}^k \alpha_i(t) N_i(\wt{x}))^{-1}|\\
    \leq & ~ \alpha_i(t) N_i(x) \cdot |(\sum_{i=1}^k \alpha_i(t)N_i(x))^{-1}-(\sum_{i=1}^k \alpha_i(t) N_i(\wt{x}))^{-1}| \\
    & ~ + \alpha_i(t)(\sum_{i=1}^k \alpha_i(t) N_i(\wt{x}))^{-1} |N_i(x)-N_i(\wt{x})|
\end{align*}
where the first step follows from Definition~\ref{def:f_i}, the second step follows from Fact~\ref{fac:norm_bounds}, and the last step follows from simple algebra.

For the first term in the above, we have
\begin{align}\label{eq:lip_f_i_1st_term}
    & ~ \alpha_i(t) N_i(x) \cdot |(\sum_{i=1}^k \alpha_i(t)N_i(x))^{-1}-(\sum_{i=1}^k \alpha_i(t) N_i(\wt{x}))^{-1}| \notag \\
    \leq & ~ \alpha_i(t) \cdot \frac{1}{(2\pi)^{d/2}\det(\Sigma_i(t))^{1/2}} \cdot |(\sum_{i=1}^k \alpha_i(t)N_i(x))^{-1}-(\sum_{i=1}^k \alpha_i(t) N_i(\wt{x}))^{-1}| \notag\\
    \leq & ~ \alpha_i(t) \cdot \frac{1}{(2\pi)^{d/2}\det(\Sigma_i(t))^{1/2}}  \cdot \gamma^{-2} \cdot \frac{1}{(2\pi)^{d/2}\det_{\min}^{1/2}} \cdot  \exp(-\frac{\beta^2}{2\sigma_{\max}}) \cdot \frac{R}{\sigma_{\min}} \cdot \|x-\wt{x}\|_2 \notag\\
    \leq & ~ \alpha_i(t) \cdot \gamma^{-2} \cdot \frac{1}{(2\pi)^{d}\det_{\min}} \cdot  \exp(-\frac{\beta^2}{2\sigma_{\max}}) \cdot \frac{R}{\sigma_{\min}} \cdot \|x-\wt{x}\|_2 
\end{align}
where the first step follows from $N_i(x) \leq \frac{1}{(2\pi)^{d/2}\det(\Sigma_i(t))^{1/2}}$, the second step follows from Lemma~\ref{lem:lip_alpha_N_to_neg_1_k_gaussian}, and the last step follows from definition of $\det_{\min}$. 

For the second term in the above, we have
\begin{align}\label{eq:lip_f_i_2nd_term}
    & ~ \alpha_i(t)(\sum_{i=1}^k \alpha_i(t) N_i(\wt{x}))^{-1} |N_i(x)-N_i(\wt{x})| \notag\\
    \leq & ~  \alpha_i(t) \cdot \gamma^{-1} \cdot |N_i(x)-N_i(\wt{x})| \notag\\
    \leq & ~ \alpha_i(t) \cdot \gamma^{-1}\cdot \frac{1}{(2\pi)^{d/2}\det(\Sigma_i(t))^{1/2}} \cdot \exp(-\frac{\beta^2}{2\sigma_{\max}(\Sigma_i(t))}) \cdot \frac{R}{\sigma_{\min}(\Sigma_i(t))} \cdot \|x-\wt{x}\|_2
\end{align}
where the first step follows from $\sum_{i=1}^k \alpha_i(t) N_i(x) \geq \gamma$, the second step follows from Lemma~\ref{lem:lip_alpha_N_k_gaussian}.

Combining Eq.~\eqref{eq:lip_f_i_1st_term} and Eq.~\eqref{eq:lip_f_i_2nd_term} together, we have
\begin{align*}
    |f_i(x) - f_i(\wt{x})| \leq & ~
    \alpha_i(t) \cdot \gamma^{-2} \cdot \frac{1}{(2\pi)^{d}\det_{\min}}  \cdot  \exp(-\frac{\beta^2}{2\sigma_{\max}}) \cdot \frac{R}{\sigma_{\min}} \cdot \|x-\wt{x}\|_2\\
    & ~ + \alpha_i(t) \cdot \gamma^{-1}\cdot \frac{1}{(2\pi)^{d/2}\det(\Sigma_i(t))^{1/2}} \cdot \exp(-\frac{\beta^2}{2\sigma_{\max}(\Sigma_i(t))}) \cdot \frac{R}{\sigma_{\min}(\Sigma_i(t))} \cdot \|x-\wt{x}\|_2\\
    \leq & ~
    \alpha_i(t) \cdot \gamma^{-2} \cdot \frac{1}{(2\pi)^{d}\det_{\min}} \cdot \exp(-\frac{\beta^2}{2\sigma_{\max}}) \cdot \frac{R}{\sigma_{\min}} \cdot \|x-\wt{x}\|_2\\
    & ~ + \alpha_i(t) \cdot \gamma^{-1}\cdot \frac{1}{(2\pi)^{d/2}\det_{\min}^{1/2}} \cdot \exp(-\frac{\beta^2}{2\sigma_{\max}}) \cdot \frac{R}{\sigma_{\min}} \cdot \|x-\wt{x}\|_2\\
    \leq & ~ 2 \alpha_i(t) \cdot \gamma^{-2} \cdot (\frac{1}{(2\pi)^{d}\det_{\min}} + \frac{1}{(2\pi)^{d/2}\det_{\min}^{1/2}}) \cdot \exp(-\frac{\beta^2}{2\sigma_{\max}}) \cdot \frac{R}{\sigma_{\min}} \cdot \|x-\wt{x}\|_2
\end{align*}
where the first step follows from the bound of the first term and the second term, the second step follows from the definition of $\det_{\min},\sigma_{\max},\sigma_{\min}$, and the last step follows from $\gamma < 0.1$.
\end{proof}

This lemma calculates Lipschitz constant of function $\|f_i(x) (-\Sigma_i(t)^{-1} (x - \mu_i(t))) - f_i(\wt{x}) (-\Sigma_i(t)^{-1} (\wt{x} - \mu_i(t)))\|_2$.
\begin{lemma}\label{lem:lip_f_i_x_minus_mu}
If the following conditions hold
\begin{itemize}
    \item Let $x, \wt{x} \in \R^d$.
    \item Let $i \in [k]$.
    \item Let $t \in \R$, and $t \geq 0$.
    \item Let $\alpha_i(t) \in \R$, $\sum_{i=1}^k \alpha_i(t) = 1$,  and $\alpha_i(t) \in (0,1)$.
    \item Let $N_i(x)$ be defined as Definition~\ref{def:N_i}.
    \item Let $f_i(x)$ be defined as Definition~\ref{def:f_i}.
    \item Let $\|x-\mu_i(t)\|_2 \leq R$, where $R \geq 1$, for each $i \in [k]$.
    \item Let $\|x-\mu_i(t)\|_2 \geq \beta$, where $\beta \in (0, 0.1)$, for each $i \in [k]$.
    \item Let $\sum_{i=1}^k \alpha_i(t) N_i(x) \geq \gamma$, where $\gamma \in (0, 0.1)$.
    \item Let $\sigma_{\min} := \min \{ \sigma_{\min}( \Sigma_1(t) ), \sigma_{\min} ( \Sigma_2(t) ), \dots, \sigma_{\min} ( \Sigma_k(t) ) \}$.  
    \item Let $\sigma_{\max} := \max \{ \sigma_{\max}( \Sigma_1(t) ) , \sigma_{\max} ( \Sigma_2(t) ), \dots, \sigma_{\max} ( \Sigma_k(t) )\}$.
    \item Let $\det_{\min} := \min \{ \det( \Sigma_1(t) ) , \det ( \Sigma_2(t) ), \dots,  \det ( \Sigma_k(t) )\}$.
\end{itemize}

Then, for each $i \in [k]$, we have
\begin{align*}
    & ~ \|f_i(x) (-\Sigma_i(t)^{-1} (x - \mu_i(t))) - f_i(\wt{x}) (-\Sigma_i(t)^{-1} (\wt{x} - \mu_i(t)))\|_2\\ 
    \leq & ~  (\frac{|f_i(x)|}{\sigma_{\min}}+ 2 \alpha_i(t) \cdot \gamma^{-2} \cdot (\frac{1}{(2\pi)^{d}\det_{\min}} + \frac{1}{(2\pi)^{d/2}\det_{\min}^{1/2}})  \cdot \exp(-\frac{\beta^2}{2\sigma_{\max}}) \cdot (\frac{R}{\sigma_{\min}})^2) \cdot \|x-\wt{x}\|_2 
\end{align*}
\end{lemma}

\begin{proof}
We can show
\begin{align*}
    \mathrm{LHS} \leq & ~ \|f_i(x)(- \Sigma_i(t)^{-1} (x - \mu_i(t)) )- f_i(x)(- \Sigma_i(t)^{-1} (\wt{x} - \mu_i(t)) )\|_2\\  
    & ~ + \|f_i(x)(- \Sigma_i(t)^{-1} (\wt{x} - \mu_i(t)) )- f_i(\wt{x})(- \Sigma_i(t)^{-1} (\wt{x} - \mu_i(t)) )\|_2  \\
    \leq & ~ |f_i(x)| \cdot \|(- \Sigma_i(t)^{-1} (x - \mu_i(t)) )- (- \Sigma_i(t)^{-1} (\wt{x} - \mu_i(t)) )\|_2\\
    & ~ + |f_i(x)-f_i(\wt{x})| \cdot \|- \Sigma_i(t)^{-1} (\wt{x} - \mu_i(t))\|_2
\end{align*}
where the first step follows from Fact~\ref{fac:norm_bounds}, the second step follows from Fact~\ref{fac:norm_bounds}.

For the first term in the above, we have
\begin{align}\label{eq:lip_f_i_x_minus_mu_1st_term}
    & ~ |f_i(x)| \cdot \|(- \Sigma_i(t)^{-1} (x - \mu_i(t)) )- (- \Sigma_i(t)^{-1} (\wt{x} - \mu_i(t)) )\|_2 \notag\\
    \leq & ~ \frac{|f_i(x)|}{\sigma_{\min}(\Sigma_i(t))}  \cdot \|x-\wt{x}\|_2
\end{align}
where the first step follows from Lemma~\ref{lem:lip_x_minus_mu}.

For the second term in the above, we have
\begin{align}\label{eq:lip_f_i_x_minus_mu_2nd_term}
    & ~ |f_i(x)-f_i(\wt{x})| \cdot \|- \Sigma_i(t)^{-1} (\wt{x} - \mu_i(t))\|_2 \notag\\
    \leq & ~ \frac{R}{\sigma_{\min}(\Sigma_i(t))} \cdot  |f_i(x)-f_i(\wt{x})| \notag\\
    \leq & ~ \frac{R}{\sigma_{\min}(\Sigma_i(t))} \cdot  2 \alpha_i(t) \cdot \gamma^{-2} \cdot (\frac{1}{(2\pi)^{d}\det_{\min}} + \frac{1}{(2\pi)^{d/2}\det_{\min}^{1/2}})  \cdot \exp(-\frac{\beta^2}{2\sigma_{\max}}) \cdot \frac{R}{\sigma_{\min}} \cdot \|x-\wt{x}\|_2
\end{align}
where the first step follows from Lemma~\ref{lem:upper_bound_x_minus_mu}, the second step follows from Lemma~\ref{lem:lip_f_i}.

Combining Eq.~\eqref{eq:lip_f_i_x_minus_mu_1st_term} and Eq.~\eqref{eq:lip_f_i_x_minus_mu_2nd_term} together, we have
\begin{align*}
    & ~ \|f_i(x)(- \Sigma_i(t)^{-1} (x - \mu_i(t)) )- f_i(\wt{x})(- \Sigma_i(t)^{-1} (\wt{x} - \mu_i(t)))\|_2\\
    \leq & ~ \frac{|f_i(x)|}{\sigma_{\min}(\Sigma_i(t))}  \cdot \|x-\wt{x}\|_2 \\
    & ~ + \frac{R}{\sigma_{\min}(\Sigma_i(t))} \cdot  2 \alpha_i(t) \cdot \gamma^{-2} \cdot (\frac{1}{(2\pi)^{d}\det_{\min}} + \frac{1}{(2\pi)^{d/2}\det_{\min}^{1/2}})  \cdot \exp(-\frac{\beta^2}{2\sigma_{\max}}) \cdot \frac{R}{\sigma_{\min}} \cdot \|x-\wt{x}\|_2\\
    \leq & ~ \frac{|f_i(x)|}{\sigma_{\min}}  \cdot \|x-\wt{x}\|_2 \\
    & ~ + \frac{R}{\sigma_{\min}} \cdot  2 \alpha_i(t) \cdot \gamma^{-2} \cdot (\frac{1}{(2\pi)^{d}\det_{\min}} + \frac{1}{(2\pi)^{d/2}\det_{\min}^{1/2}})  \cdot \exp(-\frac{\beta^2}{2\sigma_{\max}}) \cdot \frac{R}{\sigma_{\min}} \cdot \|x-\wt{x}\|_2\\
    = & ~ (\frac{|f_i(x)|}{\sigma_{\min}}+ 2 \alpha_i(t) \cdot \gamma^{-2} \cdot (\frac{1}{(2\pi)^{d}\det_{\min}} + \frac{1}{(2\pi)^{d/2}\det_{\min}^{1/2}}) \cdot \exp(-\frac{\beta^2}{2\sigma_{\max}}) \cdot (\frac{R}{\sigma_{\min}})^2) \cdot \|x-\wt{x}\|_2
\end{align*}
where the first step follows from the bound of the first term and the second term, the second step follows from the definition of $\det_{\min},\sigma_{\max},\sigma_{\min}$, and the last step follows from simple algebra.
\end{proof}

\subsection{Lipschitz constant of the score function}\label{sec:k_gaussain:lip}
This lemma calculates Lipschitz constant of the score function for $k$ mixture of Gaussians.
\begin{lemma}\label{lem:lip_score}
If the following conditions hold
\begin{itemize}
    \item Let $x, \wt{x} \in \R^d$.
    \item Let $i \in [k]$.
    \item Let $t \in \R$, and $t \geq 0$.
    \item Let $\alpha_i(t) \in \R$, $\sum_{i=1}^k \alpha_i(t) = 1$,  and $\alpha_i(t) \in (0,1)$.
    \item Let $N_i(x)$ be defined as Definition~\ref{def:N_i}.
    \item Let $p_t(x)$ be defined as Definition~\ref{def:p_t_k_gaussian}.
    \item Let $f_i(x)$ be defined as Definition~\ref{def:f_i}.
    \item Let $\|x-\mu_i(t)\|_2 \leq R$, where $R \geq 1$, for each $i \in [k]$.
    \item Let $\|x-\mu_i(t)\|_2 \geq \beta$, where $\beta \in (0, 0.1)$, for each $i \in [k]$.
    \item Let $\sum_{i=1}^k \alpha_i(t) N_i(x) \geq \gamma$, where $\gamma \in (0, 0.1)$.
    \item Let $\sigma_{\min} := \min \{ \sigma_{\min}( \Sigma_1(t) ), \sigma_{\min} ( \Sigma_2(t) ), \dots, \sigma_{\min} ( \Sigma_k(t) ) \}$.  
    \item Let $\sigma_{\max} := \max \{ \sigma_{\max}( \Sigma_1(t) ) , \sigma_{\max} ( \Sigma_2(t) ), \dots, \sigma_{\max} ( \Sigma_k(t) )\}$.
    \item Let $\det_{\min} := \min \{ \det( \Sigma_1(t) ) , \det ( \Sigma_2(t) ), \dots,  \det ( \Sigma_k(t) )\}$.
\end{itemize}

Then, we have
\begin{align*}
    & ~ \|\frac{\d \log p_t(x)}{\d x} - \frac{\d \log p_t(\wt{x})}{\d \wt{x}}\|_2 \\
    \leq & ~ (\frac{1}{\sigma_{\min}}+  \frac{2R^2}{\gamma^2\sigma_{\min}^2} \cdot (\frac{1}{(2\pi)^{d}\det_{\min}} + \frac{1}{(2\pi)^{d/2}\det_{\min}^{1/2}})  \cdot \exp(-\frac{\beta^2}{2\sigma_{\max}})) \cdot \|x-\wt{x}\|_2
\end{align*}
\end{lemma}

\begin{proof}
We can show
\begin{align*}
    & ~ \mathrm{LHS} \\
    = & ~ \| \sum_{i=1}^k f_i(x) (-\Sigma_i(t)^{-1} (x - \mu_i(t))) - \sum_{i=1}^k f_i(\wt{x}) (-\Sigma_i(t)^{-1} (\wt{x} - \mu_i(t)))\|_2\\
    \leq & ~ \sum_{i=1}^k \|f_i(x) (-\Sigma_i(t)^{-1} (x - \mu_i(t))) - f_i(\wt{x}) (-\Sigma_i(t)^{-1} (\wt{x} - \mu_i(t)))\|_2\\
    \leq & ~ \sum_{i=1}^k (\frac{|f_i(x)|}{\sigma_{\min}}+ 2 \alpha_i(t) \cdot \gamma^{-2} \cdot (\frac{1}{(2\pi)^{d}\det_{\min}} + \frac{1}{(2\pi)^{d/2}\det_{\min}^{1/2}})  \cdot \exp(-\frac{\beta^2}{2\sigma_{\max}}) \cdot (\frac{R}{\sigma_{\min}})^2) \cdot \|x-\wt{x}\|_2 \\
    = & ~ (\frac{1}{\sigma_{\min}}+  \frac{2R^2}{\gamma^2\sigma_{\min}^2} \cdot (\frac{1}{(2\pi)^{d}\det_{\min}} + \frac{1}{(2\pi)^{d/2}\det_{\min}^{1/2}})  \cdot \exp(-\frac{\beta^2}{2\sigma_{\max}})) \cdot \|x-\wt{x}\|_2
\end{align*}
where the first step follows from Lemma~\ref{lem:gradient_log_p_t_k_gaussian}, the second step follows from triangle inequality, the third step follows from Lemma~\ref{lem:lip_f_i_x_minus_mu}, and the last step follows from $\sum_{i=1}^k |f_i(x)|=\sum_{i=1}^k f_i(x)=1$ and $\sum_{i=1}^k \alpha_i(t)=1$.
\end{proof}

\section{Putting It All Together}\label{sec:all_together}

Our overall goal is that we want to provide a more concrete calculation for theorems in Section~\ref{sec:tools_previous} by assuming the data distribution is a $k$ mixture of Gaussian.
Now we provide lemmas that are used in further calculation.

Now we provide the lemma for $k$-mixtue of Gaussians which states that if $p_0$ is mixture Gaussians, then all the $\mathsf{pdf}$s in the diffusion process are also mixtures of Gaussians.
\begin{lemma}[Formal version of Lemma~\ref{lem:pdf_of_k_gaussian_plus_single_gaussian:informal}]\label{lem:pdf_of_k_gaussian_plus_single_gaussian:formal}
Let $a, b \in \R$. Let ${\cal D}$ be a $k$-mixture of Gaussian distribution, and 
let $p$ be its $\mathsf{pdf}$, i.e.,
\begin{align*}
    p(x) := \sum_{i=1}^k \frac{\alpha_i}{(2\pi)^{d/2}\det(\Sigma_i)^{1/2}}  \exp(-\frac{1}{2} (x - \mu_i)^\top \Sigma_i^{-1} (x - \mu_i))
\end{align*}
Let $x \in \R^d$ sample from ${\cal D}$. Let $z\in \R^d$ and $z \sim {\cal N}(0,I)$, which is independent from $x$. Then we have a new random variable 
$y = a x + b z$ which is also a $k$-mixture of Gaussian distribution $\wt {\cal D}$, whose $\mathsf{pdf}$ is 
\begin{align*}
    \wt p(x) := \sum_{i=1}^k \frac{\alpha_i}{(2\pi)^{d/2}\det(\wt \Sigma_i)^{1/2}}  \exp(-\frac{1}{2} (x - \wt \mu_i)^\top \wt \Sigma_i^{-1} (x - \wt \mu_i)),
\end{align*}
where $ \wt \mu_i = a\mu_i, \wt \Sigma_i = a^2 \Sigma_i + b^2 I$.  
\end{lemma}

\begin{proof}
First, we know that the $\mathsf{pdf}$ of the sum of two independent random variables is the convolution of their $\mathsf{pdf}$.

From \cite{v04} we know that the convolution of 2 Gaussians is another Gaussian, i.e. $\mathcal{N}(\mu_1,\Sigma_1) * \mathcal{N}(\mu_2,\Sigma_2) = \mathcal{N}(\mu_1 + \mu_2,\Sigma_1+\Sigma_2)$, where $*$ is the convolution operator.

And we know the $\mathsf{pdf}$ of a linear transformation of a random variable $x\in \R^d$, let's say $Ax+b$ where $A \in \R^{d \times d}, b \in \R^d$, is $\frac{1}{|\det(A)|} p(A^{-1}(x-b))$.

If we consider the transformation $ax$ where $a \in \R, x \in \R^d$, this transformation can be written as $aI x$. Therefore the $\mathsf{pdf}$ of $ax$ is $\frac{1}{|\det(aI)|} p((aI)^{-1}x) = \frac{1}{|a^d|} p(x/a)$.

Now we prove the lemma.
To find the $\mathsf{pdf}$ of $ax$, where $x \sim p(x)$, we can show
\begin{align*}
    |\frac{1}{|a^d|}p(x/a)| = & ~ \frac{1}{|a^d|} \sum_{i=1}^k \frac{\alpha_i}{(2\pi)^{d/2}\det(\Sigma_i)^{1/2}}  \exp(-\frac{1}{2} (\frac{x}{a} -  \mu_i)^\top \Sigma_i^{-1} (\frac{x}{a} - \mu_i))\\
    = & ~ \sum_{i=1}^k \frac{\alpha_i}{(2\pi)^{d/2}\det(a^2 \Sigma_i)^{1/2}}  \exp(-\frac{1}{2} (x - a\mu_i)^\top a^{-2}\Sigma_i^{-1} (a - a\mu_i))\\
    = & ~ \sum_{i=1}^k \alpha_i \mathcal{N}(a\mu_i, a^2 \Sigma_i)
\end{align*}
where the first step follows from the definition of $p(x)$, the second step follows from $a^{2d} \det(\Sigma_i)= \det(a^2 \Sigma_i)$, and the last step follows from definition of Gaussian distribution.

For a single standard Gaussian random variable $z$, the $\mathsf{pdf}$ of $bz$ will simply be $\mathcal{N}(0, b^2I)$.

To find the $\mathsf{pdf}$ of $y = ax+bz$, we can show
\begin{align*}
    \wt{p}(x) = & ~ \frac{1}{|a^d|}p(x/a) * \mathcal{N}(0,b^2 I)\\
    = & ~ (\sum_{i=1}^k \alpha_i \mathcal{N}(a\mu_i, a^2 \Sigma_i)) * \mathcal{N}(0,b^2 I)\\
    = & ~ \sum_{i=1}^k (\alpha_i \mathcal{N}(a\mu_i, a^2 \Sigma_i) * \mathcal{N}(0,b^2 I))\\
    = & ~ \sum_{i=1}^k \alpha_i \mathcal{N}(a \mu_i, a^2\Sigma_i + b^2 I)
\end{align*}
where the first step follows from the $\mathsf{pdf}$ of the sum of 2 independent random variables is the convolution of their $\mathsf{pdf}$, the second step follows from the $\mathsf{pdf}$ of $\frac{1}{|a^d|}p(x/a)$, the third step follows from the distributive property of convolution, and the last step follows from $\mathcal{N}(a\mu_i,a^2\Sigma_i) * \mathcal{N}(0,b^2 I) =  \mathcal{N}(a \mu_i, a^2\Sigma_i + b^2 I)$.

Thus, the $\mathsf{pdf}$ of $y$ can be written as a mixture of $k$ Gaussians:
\begin{align*}
    \wt p(x) := \sum_{i=1}^k \frac{\alpha_i}{(2\pi)^{d/2}\det(\wt \Sigma_i)^{1/2}}  \exp(-\frac{1}{2} (x - \wt \mu_i)^\top \wt \Sigma_i^{-1} (x - \wt \mu_i)),
\end{align*}
where $ \wt \mu_i = a\mu_i, \wt \Sigma_i = a^2 \Sigma_i + b^2 I$.  
\end{proof}

Now we provide the lemma for the second momentum of $k$-mixtue of Gaussians.
\begin{lemma}[Formal version of Lemma~\ref{lem:second_moment:informal}]\label{lem:second_moment:formal}
If the following conditions hold:
\begin{itemize}
    \item $x_0 \sim p_0$, where $p_0$ is defined by Eq.~\eqref{eq:p_0}. 
\end{itemize}

Then, we have
\begin{align*}
    m_2^2 := \E_{x_0 \sim p_0}[\| x_0 \|_2^2] = \sum_{i=1}^k \alpha_i (\|\mu_i\|_2  + \tr[\Sigma_i])
\end{align*}

\end{lemma}

\begin{proof}
From \cite{pkk22}, we know the second momentum of data distribution $p_0(x)$ is given by:
\begin{align}\label{eq:second_moment}
   \E[x_0 x_0^\top] = \sum_{i=1}^k \alpha_i \cdot ( \| \mu_i \|_2^2 + \Sigma_i)
\end{align}

Then, we can show
\begin{align*}
    \E[\|x_0\|_2^2] = & ~ \E[x_0^\top x_0] \\
    = & ~ \E[\tr[x_0 x_0^\top]] \\
    = & ~ \tr[\E[x_0 x_0^\top]] \\
    = & ~ \tr[\sum_{i=1}^k \alpha_i(\mu_i \mu_i^\top + \Sigma_i)] \\
    = & ~ \sum_{i=1}^k \alpha_i ( \| \mu_i \|_2^2 + \tr[\Sigma_i])
\end{align*}
where the first step follows from definition of $\ell_2$-norm, the second step follows from $\tr[a a^\top]=a^\top a$ where $a$ is a vector, the third step follows from the linearity of the trace operator, the fourth step follows from Eq.~\eqref{eq:second_moment}, and the last step follows from $\tr[a a^\top]= \| a \|_2^2$.
\end{proof}

Now we give the Lipschitz constant explicitly.
\begin{lemma}[Formal version of Lemma~\ref{lem:lip_const_k_gaussian:informal}]\label{lem:lip_const_k_gaussian:formal}
If the following conditions hold
\begin{itemize}
    \item Let $\|x-a_t\mu_i\|_2 \leq R$, where $R \geq 1$, for each $i \in [k]$.
    \item Let $\|x-a_t\mu_i\|_2 \geq \beta$, where $\beta \in (0, 0.1)$, for each $i \in [k]$.
    \item Let $p_t(x)$ be defined as Eq.~\eqref{eq:p_t} and $p_t(x) \geq \gamma$, where $\gamma \in (0, 0.1)$.
    \item Let $\sigma_{\min} := \min_{i \in [k]} \{ \sigma_{\min}( a_t^2 \Sigma_i + b_t^2 I ) \}$.  
    \item Let $\sigma_{\max} := \max_{i \in [k]} \{ \sigma_{\max}( a_t^2 \Sigma_i + b_t^2 I)\}$.
    \item Let $\det_{\min} := \min_{i \in [k]} \{ \det( a_t^2 \Sigma_i + b_t^2 I)\}$.
\end{itemize}

The Lipschitz constant for the score function $\frac{\d \log(p_t(x))}{\d x}$ is given by:
\begin{align*}
    L = \frac{1}{\sigma_{\min}}+  \frac{2R^2}{\gamma^2\sigma_{\min}^2} \cdot (\frac{1}{(2\pi)^{d}\det_{\min}} + \frac{1}{(2\pi)^{d/2}\det_{\min}^{1/2}})  \cdot \exp(-\frac{\beta^2}{2\sigma_{\max}})
\end{align*}
\end{lemma}

\begin{proof}
Using Lemma~\ref{lem:lip_score} and \ref{lem:pdf_of_k_gaussian_plus_single_gaussian:formal}, we can get the result.
\end{proof}

\section{More Calculation for Application}\label{sec:app:application}

In this section, we will provide a more concrete calculation for Theorem~\ref{thm:theorem2_TV_in_ccl+22}, Theorem~\ref{thm:theorem1_kl_cll23}, Theorem~\ref{thm:theorem5_kl_cll23}, Theorem~\ref{thm:theorem2_DPOM_TV_ccl+24} and Theorem~\ref{thm:theorem3_DPUM_TV_ccl+24}.

\subsection{Concrete calculation of Theorem~\ref{thm:theorem2_TV_in_ccl+22}}\label{sec:all_together:concrete_calculation}

\begin{theorem}[$\mathsf{DDPM}$, total variation, formal version of Theorem~\ref{thm:TV_ccl22:informal}]\label{thm:TV_ccl22:formal}
If the following conditions hold:
\begin{itemize}
    \item Condition~\ref{con:init} and~\ref{con:all}.
    \item The step size $h_k := T/N$ satisfies $h_k = O(1/L)$ and $L \geq 1$ for $k \in [N]$.
    \item Let $\hat{q}$ denote the density of the output of the $\mathsf{EulerMaruyama}$ defined by Definition~\ref{def:euler}. 
\end{itemize}

We have 
\begin{align*}
    \mathrm{TV}(\hat{q}, p_0) \lesssim \underbrace{\sqrt{\mathrm{KL}(p_0 \| \mathcal{N}(0,I)) }\exp(-T)}_{ \mathrm{convergence~of~forward~process}} + \underbrace{(L\sqrt{dh} + L m_2 h)\sqrt{T}}_{\mathrm{discretization~ error}} + \underbrace{\epsilon_0\sqrt{T}}_{\mathrm{score~estimation~error}
}.
\end{align*}
where 
\begin{itemize}
    \item $L = \frac{1}{\sigma_{\min(p_t)}}+  \frac{2R^2}{\gamma^2\sigma_{\min(p_t)}^2} \cdot (\frac{1}{(2\pi)^{d}\det_{\min(p_t)}} + \frac{1}{(2\pi)^{d/2}\det_{\min(p_t)}^{1/2}})  \cdot \exp(-\frac{\beta^2}{2\sigma_{\max(p_t)}})$,
    \item $m_2 = (\sum_{i=1}^k \alpha_i (\|\mu_i\|_2^2  + \tr[\Sigma_i]))^{1/2}$,
    \item $\mathrm{KL}(p_0(x) \| \mathcal{N}(0,I)) \leq \frac{1}{2} (-\log( \mathrm{det}_{\min(p_0)} )  + d \sigma_{\max(p_0)} + \mu_{\max(p_0)} - d)$.
\end{itemize}
\end{theorem}

\begin{proof}
    
Now we want to find a more concrete $L$ in Assumption~\ref{ass:assumption_lipschitz}.
Notice that from Lemma~\ref{lem:pdf_of_k_gaussian_plus_single_gaussian:formal}, we know that at any time between $0 \leq t \leq T$, $p_t$ is also a $k$-mixture of gaussian, except that the mean and covariance change with time.

Using Lemma~\ref{lem:lip_const_k_gaussian:formal}, we can get $L$.

Now we want to find the second momentum in  Assumption~\ref{ass:assumption_moment}. Using Lemma~\ref{lem:second_moment:formal}, we know that $m_2 = ( \sum_{i=1}^k \alpha_i ( \| \mu_i \|_2^2  + \tr[\Sigma_i]) )^{1/2}$.

In Assumption~\ref{ass:assumption_score_estimation}, we also assume the same thing.

Now we want to have a more concrete setting for Theorem~\ref{thm:theorem2_TV_in_ccl+22} by calculating each term directly.
Notice that now we have all the quantities except for the KL divergence term.
Thus, we calculate $\mathrm{KL}(p_0 \| \mathcal{N}(0,I))$, which means the KL divergence of data distribution and standard Gaussian.

In our notation, we have
\begin{align*}
    \mathrm{KL}(p_0(x) \| \mathcal{N}(0,I)) = & ~ \mathrm{KL}(\sum_{i=1}^k \alpha_i \mathcal{N}(\mu_i,\Sigma_i) \| \mathcal{N}(0,I))\\
    = & ~ \int \sum_{i=1}^k \alpha_i \mathcal{N}(\mu_i,\Sigma_i) \log(\frac{\sum_{i=1}^k \alpha_i \mathcal{N}(\mu_i,\Sigma_i)}{\mathcal{N}(0,I)}) \d x.
\end{align*}

However, this integral has no close form, but we can find an upper bound for this KL divergence instead.

We know the KL divergence of 2 normal distribution is given by:
\begin{align*}
    & ~ \mathrm{KL}(\mathcal{N}(\mu_1,\Sigma_1)\|\mathcal{N}(\mu_2,\Sigma_2)) \\
    = & ~ -\frac{1}{2}\log(\frac{\det(\Sigma_1)}{\det(\Sigma_2)}) + \frac{1}{2} \tr[(\Sigma_2)^{-1}\Sigma_1] + \frac{1}{2} (\mu_1-\mu_2)^\top (\Sigma_2)^{-1} (\mu_1-\mu_2) - \frac{d}{2}
\end{align*}

We define $\sigma_{\max(p_0)} = \max_{i \in [k]} \{\sigma_{\max}(\Sigma_i)\}$, $\det_{\min(p_0)} = \min_{i \in [k]} \{\det(\Sigma_i)\}$, $\mu_{\max(p_0)} = \max_{i \in [k]} \{\|\mu_i\|_2^2\}$.
From \cite{ho07}, we can show
\begin{align*}
    \mathrm{KL}(\sum_{i=1}^k \alpha_i \mathcal{N}(\mu_i,\Sigma_i) \| \mathcal{N}(0,I))
    \leq & ~ \sum_{i=1}^k \alpha_i \mathrm{KL}(\mathcal{N}(\mu_i,\Sigma_i) \| \mathcal{N}(0,I))\\
    = & ~ \sum_{i=1}^k  \frac{\alpha_i}{2} (-\log(\det(\Sigma_i)) + \tr[\Sigma_i] + \|\mu_i\|_2^2 - d)\\
    \leq & ~ \max_{i \in [k]} \frac{1}{2} (-\log(\det(\Sigma_i)) + \tr[\Sigma_i] + \|\mu_i\|_2^2 - d)\\
    \leq & ~ \frac{1}{2} (-\log( \mathrm{det}_{\min(p_0)} )  + d \sigma_{\max(p_0)} + \mu_{\max(p_0)} - d)
\end{align*}
where the first step follows from the convexity of KL divergence, the second step follows from KL divergence of 2 normal distribution, the third step follows from $\sum_{i=1}^k \alpha_i = 1$ and $0 \leq \alpha_i \leq 1$, and the last step follows from the definition of $\det_{\min(p_0)}$, $\sigma_{\max(p_0)}$, $\mu_{\max(p_0)}$.

Then we have all the quantities in Theorem~\ref{thm:theorem2_TV_in_ccl+22}. After directly applying the theorem, we finish the proof.
\end{proof}

\subsection{Concrete calculation for Theorem~\ref{thm:theorem1_kl_cll23}}\label{sec:all_together:thm1_kl_cll23}

\begin{theorem}[$\mathsf{DDPM}$, KL divergence, formal version of Theorem~\ref{thm:kl_cll23:informal}]\label{thm:kl_cll23:formal}
If the following conditions hold:
\begin{itemize}
    \item Condition~\ref{con:all}.
    \item We use uniform discretization points.
\end{itemize}

We have 
\begin{itemize}
    \item Let $\hat{q}$ denote the density of the output of the $\mathsf{ExponentialIntegrator}$ defined by Definition~\ref{def:exp_int}, we have
    \begin{align*}
        \mathrm{KL}(p_0 \| \hat{q}) \lesssim (M_2 + d)e^{-T} + T\epsilon_0^2 + \frac{d T^2 L^2}{N}.
    \end{align*}
    
    In particular, choosing $T = \Theta(\log ( M_2 d/\epsilon_0 ))$ and $N = \Theta ( d T^2 L^2/\epsilon_0^2 )$, then we can show that
    \begin{align*}
       \mathrm{KL}(p_0 \| \hat{q}) = \tilde{O} (\epsilon_0^2)
   \end{align*} 
    
    \item Let $\hat{q}$ denote the density of the output of the $\mathsf{EulerMaruyama}$ defined by Definition~\ref{def:euler}, we have
    \begin{align*}
        \mathrm{KL}(p_0 \| \hat{q}) \lesssim (M_2 + d)e^{-T} + T\epsilon_0^2 + \frac{d T^2 L^2}{N} + \frac{T^3 M_2}{N^2}.
    \end{align*}
\end{itemize}
where \begin{itemize}
    \item $L = \frac{1}{\sigma_{\min(p_t)}}+  \frac{2R^2}{\gamma^2\sigma_{\min(p_t)}^2} \cdot (\frac{1}{(2\pi)^{d}\det_{\min(p_t)}} + \frac{1}{(2\pi)^{d/2}\det_{\min(p_t)}^{1/2}})  \cdot \exp(-\frac{\beta^2}{2\sigma_{\max(p_t)}})$,
    \item $M_2 = \sum_{i=1}^k \alpha_i (\|\mu_i\|_2^2  + \tr[\Sigma_i])$.
\end{itemize}
\end{theorem}

\begin{proof}
Using Lemma~\ref{lem:lip_const_k_gaussian:formal}, we can get $L$.
Using Lemma~\ref{lem:second_moment:formal}, we can get $m_2$.
Then we directly apply Theorem~\ref{thm:theorem1_kl_cll23}.
\end{proof}

\subsection{Concrete calculation for Theorem~\ref{thm:theorem5_kl_cll23}}\label{sec:all_together:thm5_kl_cll23}

\begin{theorem}[$\mathsf{DDPM}$, KL divergence for smooth data distribution, formal version of Theorem~\ref{thm:kl_smooth_data_cll23:informal}]\label{thm:kl_smooth_data_cll23:formal}
If the following conditions hold:
\begin{itemize}
    \item Condition~\ref{con:init} and~\ref{con:all}.
    \item We use the exponentially decreasing (then constant) step size $h_k = c \min \{\max \{t_k, \frac{1}{L} \}, 1\},  c = \frac{T + \log L}{N} \leq \frac{1}{Kd}$.
    \item Let $\hat{q}$ denote the density of the output of the $\mathsf{ExponentialIntegrator}$ defined by Definition~\ref{def:exp_int}.
\end{itemize}

We have 
\begin{align*}
    \mathrm{KL}(p_0 \| \hat{q}) \lesssim (M_2 + d)\exp(-T) + T \epsilon_0^2 + \frac{d^2 (T+\log L )^2}{N},
\end{align*}
where \begin{itemize}
    \item $L = \frac{1}{\sigma_{\min(p_0)}}+  \frac{2R^2}{\gamma^2(\sigma_{\min(p_0)})^2} \cdot (\frac{1}{(2\pi)^{d}\det_{\min(p_0)}} + \frac{1}{(2\pi)^{d/2}(\det_{\min(p_0)})^{1/2}})  \cdot \exp(-\frac{\beta^2}{2\sigma_{\max(p_0)}})$,
    \item $M_2 = \sum_{i=1}^k \alpha_i (\|\mu_i\|_2^2  + \tr[\Sigma_i])$.
\end{itemize}

Furthermore, if we choosing $T = \Theta(\log ( M_2 d/\epsilon_0) )$ and $N = \Theta ( d^2 (T + \log L)^2/\epsilon_0^2 )$, then we can show
\begin{align*}
 \mathrm{KL}(p_0 \| \hat{q}) \leq \wt{O}( \epsilon_0^2 )
\end{align*}

In addition, for Euler-Maruyama scheme defined in Definition~\ref{def:euler}, the same bounds hold with an additional $M_2 \sum_{k=1}^N h_k^3$ term.
\end{theorem}

\begin{proof}
Clearly, $p_0$ is second-order differentiable.
Using Lemma~\ref{lem:lip_const_k_gaussian:formal}, we can get $L$.
Using Lemma~\ref{lem:second_moment:formal}, we can get $m_2$.
Then we directly apply Theorem~\ref{thm:theorem5_kl_cll23}.
\end{proof}

\subsection{Concrete calculation for Theorem~\ref{thm:theorem2_DPOM_TV_ccl+24}}\label{sec:all_together:thm2_DPOM}

\begin{theorem}[$\mathsf{DPOM}$, formal version of Theorem~\ref{thm:DPOM_ccl24:informal}]\label{thm:DPOM_ccl24:formal}
If the following conditions hold:
\begin{itemize}
    \item Condition~\ref{con:all}.
    \item We use the $\mathsf{DPOM}$ algorithm defined in Definition~\ref{def:algo_DPOM}, and let $\hat{q}$ be the output density of it.
\end{itemize}

We have 
\begin{align*}
    \mathrm{TV}(\hat{q}, p_0) \lesssim (\sqrt{d} + m_2) \exp(-T) + L^2 T d^{1/2} h_{\mathrm{pred}} + L^{3/2} T d^{1/2} h_{\mathrm{corr}}^{1/2} + L^{1/2} T \epsilon_0 + \epsilon.
\end{align*}
where \begin{itemize}
    \item $L = \frac{1}{\sigma_{\min(p_t)}}+  \frac{2R^2}{\gamma^2\sigma_{\min(p_t)}^2} \cdot (\frac{1}{(2\pi)^{d}\det_{\min(p_t)}} + \frac{1}{(2\pi)^{d/2}\det_{\min(p_t)}^{1/2}})  \cdot \exp(-\frac{\beta^2}{2\sigma_{\max(p_t)}})$,
    \item $m_2 = (\sum_{i=1}^k \alpha_i (\|\mu_i\|_2^2  + \tr[\Sigma_i]))^{1/2}$. 
\end{itemize}

In particular, if we set $T = \Theta(\log(d  m_2/\epsilon))$, $h_{\mathrm{pred}} = \tilde{\Theta}(\frac{\epsilon}{L^{2}d^{1/2}})$, $h_{\mathrm{corr}} = \tilde{\Theta}(\frac{\epsilon}{L^{3} d})$, and if the score estimation error satisfies $\epsilon_0 \leq \tilde{O}(\frac{\epsilon}{\sqrt{L}})$, then we can obtain TV error $\epsilon$ with a total iteration complexity of $\tilde{\Theta}(L^{3}d/\epsilon^2)$ steps.
\end{theorem}

\begin{proof}
Using Lemma~\ref{lem:lip_const_k_gaussian:formal}, we can get $L$.
Using Lemma~\ref{lem:second_moment:formal}, we can get $m_2$.
Then we directly apply Theorem~\ref{thm:theorem2_DPOM_TV_ccl+24}.
\end{proof}

\subsection{Concrete calculation for Theorem~\ref{thm:theorem3_DPUM_TV_ccl+24}}\label{sec:all_together:thm3_DPUM}

\begin{theorem}[$\mathsf{DPUM}$, formal version of Theorem~\ref{thm:DPUM_ccl24:informal}]\label{thm:DPUM_ccl24:formal}
If the following conditions hold:
\begin{itemize}
    \item Condition~\ref{con:all}.
    \item We use the $\mathsf{DPUM}$ algorithm defined in Definition~\ref{def:algo_DPUM}, and let $\hat{q}$ be the output density of it.
\end{itemize}

We have 
\begin{align*}
    \mathrm{TV}(\hat{q}, p_0) \lesssim (\sqrt{d} + m_2) \exp(-T) + L^2 T d^{1/2} h_{\mathrm{pred}} + L^{3/2} T d^{1/2} h_{\mathrm{corr}}^{1/2} + L^{1/2} T \epsilon_0 + \epsilon.
\end{align*}
where 
\begin{itemize}
    \item $L = \frac{1}{\sigma_{\min(p_t)}}+  \frac{2R^2}{\gamma^2\sigma_{\min(p_t)}^2} \cdot (\frac{1}{(2\pi)^{d}\det_{\min(p_t)}} + \frac{1}{(2\pi)^{d/2}\det_{\min(p_t)}^{1/2}})  \cdot \exp(-\frac{\beta^2}{2\sigma_{\max(p_t)}})$,
    \item $m_2 = (\sum_{i=1}^k \alpha_i (\|\mu_i\|_2^2  + \tr[\Sigma_i]))^{1/2}$.
\end{itemize}

In particular, if we set $T = \Theta(\log(d  m_2/\epsilon))$, $h_{\mathrm{pred}} = \tilde{\Theta}(\frac{\epsilon}{L^{2}d^{1/2}})$, $h_{\mathrm{corr}} = \tilde{\Theta}(\frac{\epsilon}{L^{3/2} d^{1/2}})$, and if the score estimation error satisfies $\epsilon_0 \leq \tilde{O}(\frac{\epsilon}{\sqrt{L}})$, then we can obtain TV error $\epsilon$ with a total iteration complexity of $\tilde{\Theta}(L^{2}d^{1/2}/\epsilon)$ steps.
\end{theorem}

\begin{proof}
Using Lemma~\ref{lem:lip_const_k_gaussian:formal}, we can get $L$.
Using Lemma~\ref{lem:second_moment:formal}, we can get $m_2$.
Then we directly apply Theorem~\ref{thm:theorem3_DPUM_TV_ccl+24}.
\end{proof}

\ifdefined\isarxiv

\else
\fi



\end{document}